\newif\ifpreprintmode \preprintmodetrue
\documentclass[twoside]{article}
\usepackage[letterpaper,margin=0.75in,top=1in]{geometry}
\usepackage[round]{natbib}   %

\usepackage[expansion=false]{microtype}
\usepackage{hyperref}
\hypersetup{colorlinks=true,linkcolor=[rgb]{0.1,0.2,0.6},citecolor=[rgb]{0.1,0.2,0.6},urlcolor=[rgb]{0.1,0.2,0.6}}
\usepackage{url}

\newcommand{\reponame}{CDM}
\ifpreprintmode
  \newcommand{\repomain}{https://github.com/orineo1/CDM}
  \newcommand{\reposim}{https://github.com/orineo1/CDM/tree/master/experiments/synthetic}
  \newcommand{\repomnist}{https://github.com/orineo1/CDM/tree/master/experiments/mnist}
  \newcommand{\reposd}{https://github.com/orineo1/CDM/tree/master/experiments/stable_diffusion}
  \newcommand{\repoeps}{https://github.com/orineo1/CDM/blob/master/experiments/stable_diffusion/jacobian_error/eps_g_experiment.ipynb}
  \newcommand{\hfmnist}{https://huggingface.co/anon-submission-cdm/cdm-inverse-design/tree/main/MNIST}
  
\fi
\ifpreprintmode
  \newcommand{\codelink}{\href{\repomain}{GitHub}}
\fi

\newcommand{\mlgdf}{MLGD-F}                 %
\newcommand{\mlgd}{MLGD}                    %
\newcommand{\methodname}{\mlgdf}            %
\newcommand{\lgd}{LGD}
\newcommand{\dps}{DPS}
\newcommand{\tfg}{TFG}
\newcommand{\sdedit}{SDEdit}
\newcommand{\cdm}{CDM}
\newcommand{\cdms}{CDMS}
\newcommand{\cdmo}{CDMO}
\newcommand{\mmd}{MMD}
\newcommand{\swd}{SWD}
\newcommand{\ddim}{DDIM}
\newcommand{\ddpm}{DDPM}
\newcommand{\ict}{iCT}
\newcommand{\clip}{CLIP}
\newcommand{\sdxl}{SDXL}
\newcommand{\sdxlturbo}{SDXL-Turbo}
\newcommand{\sdxlbase}{SDXL-Base}
\newcommand{\sdxllightning}{SDXL-Lightning}
\newcommand{\controlnet}{ControlNet}
\newcommand{\controlnetscribble}{ControlNet-Scribble}

\newcommand{\PP}{\mathcal{P}}                 %
\newcommand{\Pcond}[1]{\PP(Y \mid X = #1)}    %
\newcommand{\target}{\mathcal{G}}             %
\newcommand{\Qb}{\mathcal{Q}_\beta}           %
\newcommand{\dist}{\mathcal{D}}               %
\newcommand{\distHat}{\widehat{\mathcal{D}}}  %
\newcommand{\loss}{\mathcal{L}}               %
\newcommand{\lossHat}{\widehat{\mathcal{L}}}  %
\newcommand{\fphi}{f_\phi}                    %
\newcommand{\stheta}{s_\theta}                %
\newcommand{\abar}{\bar{\alpha}}              %
\newcommand{\xzero}{\hat{x}_0}                %
\newcommand{\xhat}{\hat{x}}                   %
\newcommand{\xstar}{x^\star}                  %
\newcommand{\ncond}{n_{\mathrm{cond}}}        %
\newcommand{\algGradEst}{\textsc{GradEst}}
\newcommand{\dimx}{p}   %
\newcommand{\dimy}{q}   %
\newcommand{\nsel}{n_{\mathrm{sel}}}          %
\newcommand{\nmc}{n_{\mathrm{MC}}}            %
\newcommand{\Scond}{\mathcal{S}_{\mathrm{cond}}}
\newcommand{\Ssel}{\mathcal{S}_{\mathrm{sel}}}
\newcommand{\SG}{\mathcal{S}_{\target}}       %
\newcommand{\Ks}{K_s}                         %
\newcommand{\Kstar}{K_\star}                  %
\newcommand{\epss}{\varepsilon_s}             %
\newcommand{\epsg}{\varepsilon_g}             %
\newcommand{\zetat}{\zeta_t}                  %

\usepackage[colorinlistoftodos]{todonotes}
\usepackage{amsmath}
\usepackage{amsthm}
\usepackage{amssymb}
\newtheorem{problem}{Problem}
\newtheorem{proposition}{Proposition}
\newtheorem{corollary}{Corollary}
\theoremstyle{definition}
\newtheorem{assumption}{Assumption}
\newtheorem{remark}{Remark}
\theoremstyle{plain}
\usepackage{algorithm}
\usepackage[noend]{algpseudocode}
\usepackage{subcaption}
\usepackage{graphicx}
\usepackage{multirow}
\usepackage{threeparttable}
\usepackage{tabularx}
\usepackage[normalem]{ulem}

\newif\ifnotes \notestrue    %

\newcommand{\algGuide}{\textsc{Guide}}            %
\newcommand{\algInnerLoss}{\textsc{InnerLoss}}    %
\newcommand{\algSelect}{\textsc{Select}}          %
\newcommand{\algRefine}{\textsc{Refine}}          %
\newcommand{\upd}[1]{#1}   %
\usepackage{placeins}
\usepackage{booktabs}

\begin{document}

\twocolumn[
\begin{center}
{\LARGE\bfseries Inverse Design for Conditional Distribution Matching\par}
\vspace{1em}
{\large Ori Meidler \qquad Shaul Tolkowsky \qquad Or Zuk\par}
\vspace{0.3em}
Hebrew University of Jerusalem
\end{center}
\vspace{1.5em}
]

\begin{abstract}
Generative models are powerful tools for sampling from learned distributions,
including joint modeling of coupled $(X, Y)$ signals, and conditional sampling
from $\PP(Y \mid X)$. We introduce \textbf{Conditional Distribution Matching
(\cdm{})}, a new inverse-design problem class for distributional design goals.
\cdm{} accounts for the inherent randomness of $Y$ given $X$ by finding values
$x$ whose induced conditionals match a target distribution $\target(Y)$. Given a
trained diffusion model for $\PP(X)$ and a differentiable sampler of $Y$ given
$X$, we present Matching-Loss Guided Diffusion (\mlgd{}), which solves \cdm{} at
inference time by loss-guided diffusion, with the distribution-matching loss as
the guidance signal. \mlgd{} estimates the matching loss and its gradient from
samples of the conditional sampler. We bound the loss and gradient errors by the
sampler's output and Jacobian errors and the number of samples. Combined with fast few-step
samplers, it scales to large generative pipelines such as Stable Diffusion. We
validate \mlgd{} on (i) synthetic Gaussian-mixture benchmarks, (ii) rotated
MNIST digits, where the recovered digits reflect the rotational symmetry that
the target angle distribution requires, and (iii) sketch-to-portrait generation,
where the portraits generated from one optimized sketch follow a prescribed
bimodal gender or continuous age distribution.
\end{abstract}

\section{Introduction}
\label{sec:introduction}

Generative models learn an unknown distribution $\PP$ from a finite set of
draws and then generate new ones, e.g.\
images~\citep{DDPMsohldickstein,DDPMsong,DDIMSong}.
Often the data come in coupled pairs, $\PP = \PP(X,Y)$, possibly from different
modalities, e.g.\ a sketch and the images drawn from it, or a drug dose and
the responses it produces across patients. The conditional $\Pcond{x}$ changes
with $x$: one sketch may specify a clear shape while another is ambiguous and
admits many diverse images.
Generative models also enable \emph{conditional generation}, sampling $y \sim
\Pcond{x}$ for a given
$x$~\citep{classifierGuidance, classifierfree}, e.g.\ images drawn from a given
sketch~\citep{controlnet}. \emph{Inverse
design} turns this around: find $x$ such that $\Pcond{x}$ is
concentrated around a desired \emph{point}
$y^*$~\citep{inversedesing}, applied successfully in protein
engineering~\citep{inverseProtienEng} and drug
discovery~\citep{inverseDrug}. Here we generalize inverse design to the
distributional setting: given a
\emph{distribution} $\target$ over $Y$, find $x$ whose
\emph{induced conditional distribution} $\Pcond{x}$ matches $\target$, e.g.\ a
sketch whose drawings show a prescribed mix of attributes.

In many modern settings the conditional
$\PP(Y\mid X)$ is a stochastic generative process
that we cannot modify, such as a pretrained
text-to-image pipeline, and only $x$ can be chosen. For a
fixed $x$, $Y$ is then a population: the
same sketch yields a spread of images over the sampler's
randomness, and a user may want that spread to have a prescribed
shape,  for example a balanced mix of demographic attributes or a
uniform range of ages. Such targets arise wherever a design induces a
distribution of outcomes, as in stochastic simulators. Our main contribution is
to formulate these as one problem and to make it solvable with pretrained
generative models.

We call this problem \emph{Conditional Distribution Matching} (\cdm{}) and study
two variants:
a \emph{Sampling} version (\cdms{}), which draws $x$ from the prior $\PP(X)$
tilted toward better matching to $\target(Y)$
(Problem~\ref{prob:sampling})\footnote{The phrase
``conditional distribution matching'' also names, unrelatedly, aligning
conditionals across domains~\citep{TachetDesCombes2020GLS} and a distillation
objective for discrete diffusion~\citep{Gao2025DiscreteCDM}.} and an \emph{Optimization} version (\cdmo{},
Problem~\ref{prob:opt}), which seeks the $x$ that minimizes the distance to
$\target$.
We propose a tractable inference-time solution that requires no further training
or fine-tuning, given a diffusion model of $\PP(X)$ and a differentiable sampler
of $Y$ given $X$, trained on sampled pairs or given as an existing generative
pipeline. We achieve this by casting \cdm{} as
loss-guided diffusion~\citep{DiffusionPosteriorSampling, lossGuidedDiffusion}
with the distribution-matching loss as the
guidance signal, a framework we call \emph{Matching-Loss Guided Diffusion}
(\mlgd{}, Algorithm~\ref{alg:mlgd-outer}). We make that loss tractable with a
plug-in
estimator that draws samples from a pretrained differentiable conditional
sampler, compares them with a sample of $\target$ by a sample distance such as the
maximum mean discrepancy (\mmd{}) or the sliced Wasserstein distance (\swd{}), and
backpropagates through a random subset of the samples
(Algorithm~\ref{alg:inner-estimator}). Any such sampler can be plugged in; with
a few-step sampler, such as a consistency or distilled model, we call the method
\mlgdf{}. For low-dimensional $x$, a finite-difference variant of the gradient
captures how the mode weights of $\Pcond{x}$ change with $x$, which
backpropagation through the sampler misses; it solves the hardest cell of our
synthetic benchmark.
We prove that the loss and gradient errors of
Algorithm~\ref{alg:inner-estimator} are bounded by the error of the conditional
sampler plus a term that vanishes with the number of samples, and that a
few-step sampler cuts the memory of backpropagation in proportion to its number
of steps (Section~\ref{sec:method:theory}).

We evaluate the method on three problems (Section~\ref{sec:experiments}).
On synthetic Gaussian mixtures with $X$ and $Y$ of dimension $1$, $5$ or $10$,
the recovered conditionals are within total variation $0.02$--$0.19$ of the
target; the few-step variant is $10\times$ faster per run but is limited by its
sampler's accuracy once $X$ has $5$ or more dimensions (Section~\ref{sec:exp:sim}).
On rotated MNIST digits $X$ with their angle $Y$, the recovered digits match
uniform, bimodal and unimodal angle targets to an \swd{} of $0.03$ and show the
rotational symmetries each target requires (Section~\ref{sec:exp:mnist}).
With Stable Diffusion, optimizing a face sketch $X$ lowers the \mmd{} between the
Contrastive Language-Image Pre-training (\clip{}) embeddings~\citep{clip} of the
generated portraits and the target by $31$--$37\%$, and yields $49\%$ male
portraits for a $50/50$ gender target (Section~\ref{sec:exp:sd}).

\subsection{Related work}
\label{sec:background}

\paragraph{Loss-guided sampling.}
Diffusion Posterior Sampling (\dps{})~\citep{DiffusionPosteriorSampling} steers
a pretrained diffusion prior toward a differentiable objective at inference
time: at each denoising step it adds the gradient of a loss evaluated at the
current estimate of the clean sample. Loss-Guided Diffusion
(\lgd{})~\citep{lossGuidedDiffusion} extends this to general losses and averages
the gradient over Monte Carlo perturbations of that estimate. Further
training-free methods refine the idea, including
FreeDoM~\citep{yu2023freedomtrainingfreeenergyguidedconditional},
MPGD~\citep{he2023manifoldpreservingguideddiffusion} and
UGD~\citep{bansal2023universalguidancediffusionmodels}; Training-Free Guidance
(\tfg{})~\citep{ye2024tfgunifiedtrainingfreeguidance} unifies them in one design
space, and AdamDPS~\citep{belardi2026adaptivemomentssurprisinglyeffective}
applies Adam's adaptive moments to the gradient. We use \emph{loss-guided
diffusion} for the whole family and \lgd{} for the method of
\citet{lossGuidedDiffusion}. All of these methods guide toward a loss on a
single generated sample; in \cdm{} the loss compares a whole conditional
distribution with $\target$.

\paragraph{Distribution matching.}
A large family of generative methods trains a parametric generator so
that its output distribution matches a target $\target$, using
losses based on the \mmd{}~\citep{li2015generativemomentmatchingnetworks,
li2017mmdgandeeperunderstanding}
or Sinkhorn divergences~\citep{genevay2017learninggenerativemodelssinkhorn}.
Each new target means retraining the generator. In contrast, we keep the generative
pipeline fixed, both the diffusion prior over $X$ and the conditional model of
$Y$ given $X$, and optimize a single value $x$ at inference time so that its
induced conditional matches $\target$. Simulation-based
inference~\citep{cranmer2020frontier} fits simulator parameters to observed data;
\cdm{} instead seeks a design whose outcome distribution matches a prescribed
target, under a learned prior over designs.

\paragraph{Ensemble-level inverse design.}
\citet{zick2026gradient} pose the same task for molecular systems: they tune the interaction parameters of a coarse-grained peptide so that the configurations generated by a frozen conditional diffusion model reproduce a target free-energy surface. Their method appeared after the first version of this work and solves an instance of Problem~\ref{prob:opt} by gradient descent within a box, with a kernel density estimate of the Kullback--Leibler divergence as $\dist$. \mlgd{} adds a prior over designs and a guided search from many starting points.

\section{Conditional Distribution Matching by Loss-Guided Diffusion}
\label{sec:methods}

\subsection{Problem formulation}
\label{sec:problem}

Consider an unknown joint distribution $\PP(X, Y)$ over two random vectors
$X \in \mathcal{X} \subseteq \mathbb{R}^{\dimx}$ and $Y \in \mathcal{Y} \subseteq
\mathbb{R}^{\dimy}$, possibly of different modalities, that can be probed
only indirectly, through sampling. The user specifies a target distribution
$\target(Y)$. Let $\dist(P, Q)$ be a distance between
distributions, such as the \mmd{}.
 \begin{problem}[Conditional Distribution Matching Sampling]
\label{prob:sampling}
Given samples of $\PP$ and $\target(Y)$, and an inverse temperature $\beta \ge
0$, sample $x$ from the
tilted distribution
\begin{gather*}
  \Qb(x) \propto \PP(x)\,e^{-\beta\,\loss(x)}, \\
  \text{where}\quad \loss(x) = \dist\bigl(\Pcond{x},\, \target(Y)\bigr).
\end{gather*}
\end{problem}

\begin{problem}[Conditional Distribution Matching Optimization]
\label{prob:opt}
Given samples of $\PP$ and $\target(Y)$, find $\xstar$ that minimizes the
distance between the induced
conditional distribution and the target:
\[
  \xstar \;\equiv\; \arg\min_{x}\;\dist\bigl(\Pcond{x},\, \target(Y)\bigr).
\]
\end{problem}

Taking $\beta \to 0$ in Problem~\ref{prob:sampling} reduces to sampling from
the prior $\PP(X)$; as $\beta \to \infty$, $\Qb$ concentrates on the minimizers
of $\loss$, the solutions of Problem~\ref{prob:opt}.

\paragraph{Population and sample loss.}
Both problems are stated with the population loss $\loss(x)$, which defines the
objective but is not observable, since $\PP$ and $\target$ are known only
through samples.
For distributions $P$ and $Q$ known through samples $\mathcal{S}_P \sim P$ and
$\mathcal{S}_Q \sim Q$, we compute a sample version $\distHat(\mathcal{S}_P,
\mathcal{S}_Q)$, an estimator of $\dist(P,Q)$ built from the two samples
(Appendix~\ref{app:mmd-def}).
In this paper we access $\PP$ through pretrained generative models: a diffusion
model of $\PP(X)$ and a differentiable sampler of $\PP(Y \mid X)$
(Section~\ref{sec:prelim}). They are trained on sampled pairs $(x_i, y_i)$
(Sections~\ref{sec:exp:sim}, \ref{sec:exp:mnist}), or, when a generative
pipeline defines $\PP(Y \mid X)$, the sampler is the pipeline itself
(Section~\ref{sec:exp:sd}). Algorithm~\ref{alg:inner-estimator} then replaces
$\loss(x)$ by the sample loss
$\lossHat(x) = \distHat\bigl(\Scond(x),\, \SG\bigr)$, where $\Scond(x)$ is a
sample from the conditional model at $x$ and $\SG$ a sample of $\target$;
Section~\ref{sec:method:theory} bounds the distance between $\lossHat(x)$ and
$\loss(x)$ by the conditional model's error and the sample sizes.

\subsection{Preliminaries: diffusion priors and conditional samplers}
\label{sec:prelim}

Denoising Diffusion Probabilistic Models
(\ddpm{})~\citep{DDPMsohldickstein,DDPMsong} gradually add noise to data $x_0$ over $T$
steps and learn a reverse process that denoises $x_T \sim \mathcal{N}(0,I)$ back
to a data sample. The noisy state at step $t$ is $x_t = \sqrt{\abar_t}\,x_0 +
\sqrt{1-\abar_t}\,\epsilon$, with $\epsilon \sim \mathcal{N}(0, I)$ and noise
schedule $\{\abar_t\}_{t=1}^{T}$. We model $\PP(X)$ implicitly with a score
network $\stheta(x_t, t)$, trained by denoising score matching to approximate
the score $\nabla_{x_t} \log \PP_t(x_t)$ of the noised marginal~\citep{scoreSDE};
running the reverse process with $\stheta$ yields approximate samples from
$\PP(X)$. By Tweedie's formula, the estimate of the clean sample given $x_t$ is
\begin{equation}
\xzero(x_t) = \bigl(x_t + (1-\abar_t)\,\stheta(x_t, t)\bigr)/\sqrt{\abar_t}.
\label{eq:tweedie}
\end{equation}
The
Denoising Diffusion Implicit Models (\ddim{}) sampler~\citep{DDIMSong}
discretizes the deterministic probability-flow ODE, allowing step-by-step sampling.
\upd{The method uses $\stheta$ frozen and differentiates $\xzero$ with respect to
$x_t$ (Section~\ref{sec:methods:inner}), so $\stheta$ must be differentiable in
its input.}

The conditional $\PP(Y \mid X)$ enters only through a pretrained
\emph{conditional sampler} $\fphi(x, \eta)$: \upd{a map with frozen parameters
$\phi$ from a design $x$ and internal noise $\eta \sim \pi$ to an outcome,} such
that for fixed $x$, $\fphi(x, \eta)$ is distributed as $\Pcond{x}$, or
approximately so \upd{(Section~\ref{sec:experiments} discusses when each holds)}. \upd{For the backpropagated gradient
(Section~\ref{sec:methods:inner}) $\fphi$ must also be differentiable in $x$ for
fixed $\eta$; the finite-difference gradient of the same section needs only
forward evaluations of $\fphi$. Its parameters $\phi$ are never trained. These two frozen networks,
$\stheta$ and $\fphi$, are all the method requires.}
In the sample loss of Section~\ref{sec:problem}, $\Scond(x)$ is drawn from
$\fphi(x, \cdot)$. A conditional diffusion model is one choice of $\fphi$, but
standard diffusion samplers require dozens to thousands of
steps~\citep{DDPMsong,scoreSDE,DDIMSong}. Few-step samplers trade some accuracy
for speed: consistency
models~\citep{ConsistencyModelts1,ConsistencyModelts2}, here in their improved
Consistency Training (\ict{}) variant, generate in one or a few steps, and
distilled diffusion
models~\citep{salimans2022progressivedistillationfastsampling,sauer2023adversarialdiffusiondistillation}
in a few.

\subsection{Outer loop: guided reverse diffusion}
\label{sec:methods:outer}

We solve both problems by loss-guided diffusion: a reverse diffusion over $x$
driven by the pretrained score network $\stheta$ for $\PP(X)$, pushed at every
step toward a smaller distribution-matching loss $\loss(x)$
(Algorithm~\ref{alg:mlgd-outer}). \upd{We run it $R$ times independently, from
$x_T^{(r)} \sim \mathcal{N}(0,I)$, $r = 1, \ldots, R$. For \cdms{} the $R$
outputs are the samples. For \cdmo{} we refine each output with
$\algRefine{}$ (gradient steps on $\lossHat(x)$ at $x$ itself, no diffusion;
Algorithm~\ref{alg:refine}) and return
$\xhat = \arg\min_{r}\, \lossHat(x_0^{(r)})$, with $\lossHat$ re-estimated
from fresh samples for the few smallest.}
The strength of the guidance is proportional to the inverse temperature $\beta$
and so selects the tilted distribution $\Qb$; \cdmo{} uses a large $\beta$. The
trajectory targets $\Qb$ only approximately (the exact noisy-time guidance score
is intractable); Remark~\ref{rem:lgd-approx} in Appendix~\ref{app:theory}
states the approximation explicitly.

At each reverse-diffusion step, Algorithm~\ref{alg:mlgd-outer} forms Tweedie's
estimate $\xzero$ from the current state $x_t$ (step~3), calls the inner loss
estimator (Algorithm~\ref{alg:inner-estimator}, Section~\ref{sec:methods:inner})
in step~4 to obtain the sample loss $\lossHat$ and its gradient at $x_t$, and
passes both to the guidance update $\algGuide{}$ in step~5, which advances $x_t$
one step down the noise schedule and adds a guidance term that decreases the
loss. The inner estimator, one of our contributions, computes $\lossHat$ from a
sample $\Scond(x)$ of $\fphi(x, \cdot)$ and a sample $\SG$ of $\target$, and
differentiates it with respect to $x$. Apart from the inner estimator and the final refinement step $\algRefine{}$
(line~7, optimization only), Algorithm~\ref{alg:mlgd-outer}
is the generic loss-guided diffusion framework of
Section~\ref{sec:background}~\citep{DiffusionPosteriorSampling,
lossGuidedDiffusion, ye2024tfgunifiedtrainingfreeguidance}, and the name
\mlgd{} follows \lgd{}. We build on diffusion priors because they combine
strong sample quality and
diversity~\citep{xiao2022tacklinggenerativelearningtrilemma}.

In all our experiments the guidance update is a deterministic \ddim{}
step~\citep{DDIMSong} followed by a gradient step,
\begin{equation}
\begin{aligned}
x_{t-1} \;=\;& \underbrace{\sqrt{\abar_{t-1}}\,\xzero
 \;+\; \sqrt{1-\abar_{t-1}}\;
 \frac{x_t - \sqrt{\abar_t}\,\xzero}{\sqrt{1-\abar_t}}}_{\text{\ddim{} step}} \\
 &\;-\; \zetat\,\nabla_{x_t}\lossHat(\xzero),
\end{aligned}
\label{eq:guidance-step}
\end{equation}
where $\abar_t$ is the noise schedule of the prior (Section~\ref{sec:prelim})
and $\zetat > 0$ is the step size that sets the guidance strength, the only
hyperparameter of this rule ($\Theta_t = \zetat$ in
Algorithm~\ref{alg:mlgd-outer}); its scale is proportional to $\beta$. The step
size may depend on $t$ and on the current loss value $\lossHat(\xzero)$. The synthetic
experiments replace the gradient in Equation~\ref{eq:guidance-step} by its
Adam-normalized version, with moments accumulated over the reverse steps as in
AdamDPS~\citep{belardi2026adaptivemomentssurprisinglyeffective}, so that one
constant step works across dimensions (Appendix~\ref{sec:app:sim:step}); MNIST
uses a
noise-level schedule (Table~\ref{tab:mnist_lgd_hparams}), and Stable Diffusion
the \dps{} normalization~\citep{DiffusionPosteriorSampling}
$\zetat = \zeta' / \lossHat(\xzero)$ (Table~\ref{tab:sd_hparams}).
In the synthetic and MNIST experiments we follow \lgd{} and average the gradient
over $\nmc$ Gaussian perturbations of $\xzero$ (Algorithm~\ref{alg:mlgd-d-full},
Appendix~\ref{sec:app:fullalg}).
Any rule of the \tfg{} family can replace
Equation~\ref{eq:guidance-step} unchanged (Table~\ref{tab:guidance-rules},
Appendix~\ref{sec:app:fullalg}).

\begin{algorithm}[t]
\caption{\mlgd{} outer loop}
\label{alg:mlgd-outer}
\raggedright
\textbf{Input:} Pretrained score network $\stheta$ for $\PP(X)$ with noise
schedule $\{\abar_t\}_{t=1}^T$; guidance hyperparameters
$\{\Theta_t\}_{t=1}^T$; refinement hyperparameters (optimization only).\\
\textbf{Output:} Sample $x_0$ guided toward $\target$.\\[3pt]
\begin{algorithmic}[1]
\State $x_T \sim \mathcal{N}(0, I)$
\For{$t = T, T{-}1, \ldots, 1$}
  \State $\xzero \leftarrow \xzero(x_t)$ \Comment{Tweedie, Eq.~\ref{eq:tweedie}}
  \State $\lossHat(\xzero),\; \nabla_{x_t}\lossHat(\xzero) \leftarrow \algInnerLoss{}(x_t, \xzero)$
  \State $x_{t-1} \leftarrow \algGuide{}\bigl(x_t, \xzero, t,\, \lossHat(\xzero),\, \nabla_{x_t}\lossHat(\xzero),\, \Theta_t\bigr)$
\EndFor
\State $x_0 \leftarrow \algRefine{}(x_0)$ \Comment{optimization only}
\State \Return $x_0$
\end{algorithmic}
\end{algorithm}

\begin{algorithm}[t]
\caption{\algInnerLoss{}: inner loss estimator}
\label{alg:inner-estimator}
\raggedright
\textbf{Input:} Current state $x_t$ and its Tweedie estimate $\xzero$.\\
\textbf{Fixed throughout:} conditional sampler $\fphi(x, \eta)$ with internal
noise $\eta \sim \pi$; target sample $\SG \sim \target$; sample distance
$\distHat$ (Section~\ref{sec:problem}); sample budget $\ncond$; \upd{gradient
estimator mode $\mathit{mode} \in \{\textsc{Backprop}, \textsc{FiniteDiff}\}$}.\\
\textbf{Output:} Loss estimate $\lossHat(\xzero)$ and its gradient $\nabla_{x_t}
\lossHat$ with respect to $x_t$.\\[3pt]
\begin{algorithmic}[1]
\State Sample $\eta_1, \ldots, \eta_{\ncond} \stackrel{\text{iid}}{\sim} \pi$
\State $\Scond \leftarrow \bigl\{\fphi(\xzero, \eta_i)\bigr\}_{i=1}^{\ncond}$
\State $\lossHat(\xzero) \leftarrow \distHat\bigl(\Scond,\, \SG\bigr)$
\State \upd{$\nabla_{x_t} \lossHat \leftarrow \algGradEst{}\bigl(x_t,\, \xzero,\, \{\eta_i\},\, \SG,\, \mathit{mode}\bigr)$}
\State \Return $\lossHat(\xzero),\; \nabla_{x_t} \lossHat$
\end{algorithmic}
\end{algorithm}

\subsection{Inner estimator}
\label{sec:methods:inner}

Algorithm~\ref{alg:inner-estimator} is the estimator called in step~4. It draws
$\ncond$ samples from $\fphi(\xzero, \cdot)$, evaluates the sample distance
$\distHat$ (e.g.\ the \mmd{} or the \swd{}) against a target sample $\SG \sim
\target$, and computes the gradient with respect to $x_t$ by one of two
estimators (Algorithm~\ref{alg:gradest}). Both end with the pullback through
$\partial \xzero / \partial x_t = \bigl(I + (1-\abar_t)\,\partial_{x_t}
\stheta\bigr)/\sqrt{\abar_t}$, as in \dps{}~\citep{DiffusionPosteriorSampling}.
The sample budgets may change with the step; we keep them fixed.

\textsc{Backprop} differentiates $\lossHat$ through $\fphi$. A $K$-step $\fphi$
costs $T \cdot K \cdot \ncond$ network evaluations per run, and the backward pass
at each outer step has activation memory proportional to $K \cdot \ncond$; a
few-step sampler with $\Ks$ steps instead of the $\Kstar$ of a standard diffusion
sampler ($\Ks \ll \Kstar$) divides both by $\Kstar / \Ks$
(Appendix~\ref{app:ablation} measures this gap at Stable Diffusion scale).
Since the backward pass dominates, it backpropagates through only $\nsel \le
\ncond$ samples chosen by $\algSelect{}$ (All or Uniform,
Appendix~\ref{app:select}), while each sample's term $\partial_y
\distHat(\Scond, \SG)$ still uses all $\ncond$ samples; for the \mmd{}, a selected
sample is thus also repelled by the unselected ones, as in the full gradient.

\upd{\textsc{FiniteDiff} replaces backpropagation through $\fphi$ by central
differences of $\lossHat$ along $M$ orthonormal directions, with the same noise
$\eta_i$ at every point so that the sampling noise cancels in each difference.
Backpropagation at fixed $\eta$ is unbiased in expectation but, when $\Pcond{x}$
is multimodal, estimates poorly how its mode weights change with $x$:
$\fphi(x, \cdot)$ sends most noise values into one mode or another, and the
weights change only through the few $\eta_i$ near the boundary between these
regions, where $\fphi$ varies steeply in $x$. These few samples carry the whole
mode-weight term, with very large derivatives, so for moderate $\ncond$ its
estimate is heavy-tailed: usually near zero, occasionally very large. A central difference over a
step $h$ counts the samples that cross the boundary between $x \pm h u_k$,
trading a bias that vanishes with $h$ for a much smaller variance. It needs only forward evaluations of $\fphi$, so it also
applies when $\fphi$ is not differentiable in $x$. Our experiments use the $M =
\dimx$ coordinate directions ($2\dimx + 1$ loss evaluations per step);
Algorithm~\ref{alg:gradest} also allows $M \ll \dimx$ random orthonormal or tangent directions, which cost $2M + 1$
evaluations, but we do not evaluate them here.}

\subsection{Theoretical analysis}
\label{sec:method:theory}

\mlgd{} treats $\fphi$ as a black-box differentiable map: at each
reverse-diffusion step, the gradient $\nabla_x \lossHat_\phi$ of the sample loss
computed with $\fphi$ is obtained by backpropagating through $\fphi$
($\mathit{mode} = \textsc{Backprop}$; the analysis does not cover
\textsc{FiniteDiff}).
Appendix~\ref{app:theory} analyzes the squared-MMD estimator under three
idealized assumptions: a kernel-smoothness condition, an output error bound
$\mathbb{E}_\eta\|\fphi(x,\eta)-f_{\mathrm{true}}(x,\eta)\|^2 \le \epss^2$
between the sampler and the true conditional map $f_{\mathrm{true}}$ under
shared noise $\eta$, and a Jacobian error bound $\epsg :=
\bigl(\mathbb{E}_\eta\|\partial_x \fphi-\partial_x
f_{\mathrm{true}}\|_F^2\bigr)^{1/2}$ on their derivatives with respect to $x$.
Under these assumptions the loss error is $|\lossHat_\phi - \loss| = O(\epss +
\ncond^{-1/2})$ (Prop.~\ref{prop:loss}). The gradient has bias $O(\epsg +
\epss)$ and variance $O(\ncond^{-1})$ (Prop.~\ref{prop:grad}). Replacing a
$\Kstar$-step teacher with a $\Ks\!<\!\Kstar$ step student reduces reverse-mode
memory by $\Theta(\Kstar/\Ks)$ (Prop.~\ref{prop:memory}), which formalizes the
memory gap measured in Appendix~\ref{sec:ablation}. Its gradient discrepancy is bounded by the distillation errors
$\varepsilon_{\mathrm{dist}}, \varepsilon_{g,\mathrm{dist}}$, the
student-vs-teacher analogs of $\epss, \epsg$, neither of which follows from
existing convergence results (Appendix~\ref{app:theory}).

\section{Experiments}
\label{sec:experiments}

\begin{figure*}[t]
  \centering
  \includegraphics[width=\textwidth]{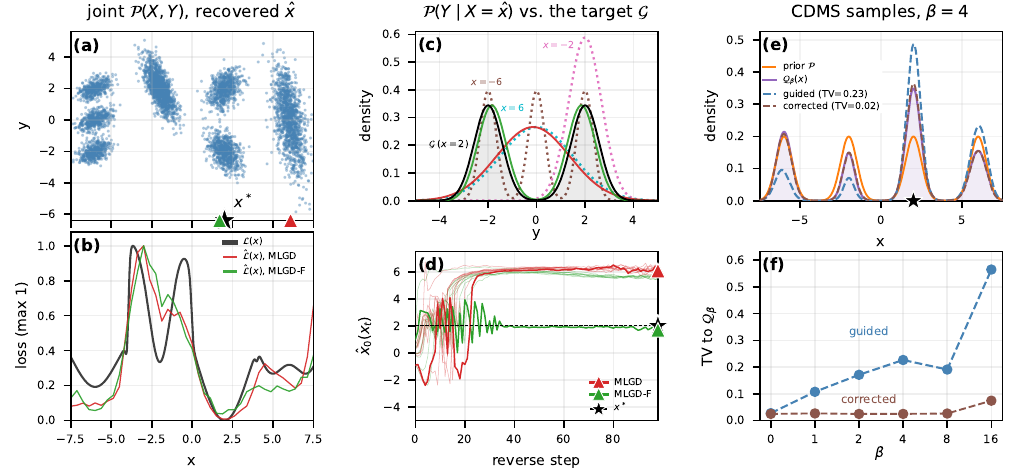}
\caption{\textbf{Toy example} ($\dimx = \dimy = 1$; details in
Appendix~\ref{sec:app:toy}). Four tilted clusters in $x$ each induce a different
conditional on $y$; only $\xstar = 2$ induces the target $\target$. Colors:
red \mlgd{}, green \mlgdf{}, black $\target$ and $\xstar$, in all panels.
\textbf{(a)} Joint samples; triangles mark the $\xhat$ returned by each method (best of ten
restarts), the star $\xstar$. \textbf{(b)} The population loss $\loss(x)$, zero only
at $\xstar$, and the sample loss $\lossHat(x)$ (\mmd{}) the optimizer sees with each
sampler, each scaled to maximum $1$.
\textbf{(c)} The conditional induced by each $\xhat$ against $\target$, with the
conditionals of the decoys (dotted). \textbf{(d)} The Tweedie estimate $\xzero(x_t)$ of every restart along
the reverse process (bold: the one returned). Each run commits to a basin
within about $30$ of the $100$ steps, and only about one restart in ten enters
the narrow $\xstar$ basin: here one \mlgdf{} restart does and every \mlgd{}
restart ends in the broad decoy, while in the run of Appendix~\ref{sec:app:toy}
the roles are reversed. \textbf{(e)} \cdms{} at $\beta = 4$: smoothed densities of
$10^5$ guided samples and of twisted-SMC corrected samples (dashed) against the
tilted target $\Qb(x) \propto \PP(x)\, e^{-\beta \loss(x)}$ (shaded) and the prior;
the corrected samples follow $\Qb$, the guided ones miss it in every basin.
\textbf{(f)} TV to $\Qb$ of both sample sets against $\beta$: the bias of the guided
sampler grows with $\beta$, and the correction removes most of it.}
\label{fig:sim_2d}
\end{figure*}

Three settings of increasing scale test whether \mlgd{} optimizes over
high-dimensional $X$ and handles complex, high-dimensional $Y$: Gaussian
mixtures with known optima and $\dimx, \dimy \in \{1, 5, 10\}$
(Section~\ref{sec:exp:sim}), an MNIST rotation task with $\dimx = 784$ pixels
and $\dimy = 2$ (Section~\ref{sec:exp:mnist}), and Stable Diffusion with a
scribble latent of $\dimx = 16{,}384$ and a \clip{} embedding of $\dimy = 768$
(Section~\ref{sec:exp:sd}).
\upd{In the first two, $\fphi$ is trained on pairs from $\PP(X, Y)$ and only
approximates $\PP(Y \mid X)$; in the third, $\PP(Y \mid X)$ is \emph{defined} as
the output distribution of a frozen generative pipeline, so $\fphi$ is that
pipeline and exact. We therefore score a returned $\xhat$ against the true
conditional where it is known (synthetic), with the pipeline itself (Stable
Diffusion), and, for MNIST, where a generated image has no ground-truth angle
distribution, with $\fphi$ together with an independent digit classifier.}
Code for the method and all experiments is available in \codelink{}.

\subsection{Synthetic Simulations}
\label{sec:exp:sim}

\paragraph{Data.}
All synthetic data are Gaussian mixtures $\PP(X,Y)$, so the conditional
$\Pcond{x}$ is a mixture in closed form for every $x$.
The target is $\target(Y) = \Pcond{\xstar}$ for a known $\xstar$, and any
recovered $\xhat$ is scored exactly by the total variation distance
$\mathrm{TV}(\Pcond{\xhat}, \target)$, which does not depend on $\dimy$ (the
relative $L^2$ distance is reported in Appendix~\ref{sec:app:sim:results}).
A toy example with $\dimx = \dimy = 1$ illustrates the problem, the recovered
solution and the sampling-to-optimization interpolation in $\beta$
(Figure~\ref{fig:sim_2d}; Appendix~\ref{sec:app:toy}).
We then vary both dimensions over the grid $\dimx, \dimy \in \{1, 5, 10\}$:
each scenario is a random four-component mixture in which two components share
the $x$-part of their mean, that shared vector is $\xstar$, and $\target$ is
bimodal (Appendix~\ref{sec:app:sim:params}).

\paragraph{Models and setup.}
The optimizer sees trained networks only: an unconditional \ddpm{} prior on
$\PP(X)$ ($T = 100$, deterministic \ddim{}), an \ict{} few-step sampler for
$\PP(Y \mid X)$ (\mlgdf{}), and, as the multi-step reference, a conditional \ddpm{} sampled
with $\Kstar = 100$ differentiated \ddim{} steps (\mlgd{}); all are multilayer
perceptrons (MLPs) trained on fresh draws (Appendix~\ref{sec:app:sim:arch}) and
checked against the analytic mixture before any search
(Appendix~\ref{sec:app:sim:quality}); their accuracy bounds the accuracy of any
search, since the search minimizes the loss of the trained sampler.
Every algorithm gets ten restarts (thirty for \mlgd{}$^\star$, below) from $x_T \sim \mathcal{N}(0, I)$ per run and
returns the one with the smallest final loss.
Runs use the \mmd{} with $\ncond = 250$, Adam-normalized guidance with
step $0.05$ (a constant step on the raw gradient has no value that works across
dimensions, Appendix~\ref{sec:app:sim:step}), and $100$ refinement steps on the
sample loss $\lossHat$ at the end ($1{,}000$ samples per step, no diffusion;
ablated in Table~\ref{tab:sim_full}).
Each algorithm is repeated five times; Table~\ref{tab:sim_main} reports mean
$\pm$ std over the runs.

\paragraph{Results.}
Figure~\ref{fig:sim_2d} shows the toy example: the loss landscape, the recovered
$\xhat$ and its conditional for \mlgd{} and \mlgdf{}, their guided trajectories, and the $\beta$-sweep with the analytic loss as
guidance (Appendix~\ref{sec:app:cdms_beta}), which illustrates the interpolation
of Problem~\ref{prob:sampling} and the bias of the guided sampler with respect
to $\Qb$, which twisted sequential Monte Carlo (SMC) reweighting removes (black
curves in Figure~\ref{fig:sim_2d}(e)).
In Table~\ref{tab:sim_main}, difficulty is set by $\dimx$ far more than by
$\dimy$: \mlgd{} returns a conditional within $\mathrm{TV} \le 0.08$ of the
target at $\dimx = 1$, $0.06$--$0.14$ at $\dimx = 5$ and $0.13$--$0.20$ at
$\dimx = 10$, except in the scalar-$Y$ cell $(\dimx, \dimy) = (10,1)$
($0.49$), where the outcome is decided by the basin a trajectory starts in, and
$100$ restarts in place of $10$ bring \mlgd{} to $0.15$
(Appendix~\ref{sec:app:sim:results}).
\mlgd{}$^\star$ (\textsc{FiniteDiff} guidance, refinement keeping the best
iterate, and $30$ restarts with re-estimated selection) is lowest or tied in six
of the nine cells and solves the $(10,1)$ cell ($0.07$). GB-FESO, Adam descent on
$x$ through the same sampler, is best at $(5,5)$ and $(10,5)$, where its KDE-KL
loss suits the target, but stays at $0.55$ in $(10,1)$, where descent from a
training design remains in its starting basin
(Appendix~\ref{sec:app:sim:gbfeso}).
When $\dimy < \dimx$ the minimizer is a set, so $\|\xhat - \xstar\|$ is not a
quality measure and is reported only in the appendix.
\mlgdf{} is $10\times$ faster per run ($50$\,s against $8$\,min) and matches
\mlgd{} at $\dimx = 1$, but at $\dimx \ge 5$ it is capped by the consistency
sampler, whose error at $\xstar$ exceeds the diffusion sampler's in every cell
but $(10,1)$, by up to $6\times$, and does not shrink with longer or wider training (Table~\ref{tab:model_quality},
Appendix~\ref{sec:app:sim:step}).
Table~\ref{tab:sim_full} separates the sources of error: without refinement
\mlgd{} is $2$--$3\times$ worse at $\dimx \ge 5$, and the oracle runs with the
exact conditional as sampler land where unrefined \mlgd{} lands, so at high
dimension the precision of the guided trajectory under a noisy $250$-sample loss
sets the limit.

\paragraph{Infeasible targets and point losses.}
Even a point target calls for the distributional view of
Problem~\ref{prob:opt}. Appendix~\ref{sec:app:nostar} uses a point-mass target $\target=\delta(y-4)$
that no $x$ can realize: \mlgdf{} reaches the best available design in $100\%$ of
restarts, whereas guiding with the pointwise loss $\mathbb{E}(y-y^*)^2$ of
standard inverse design converges to a design whose outcomes cluster near $y^*$
but never reach it ($0\%$ of restarts).

\begin{table}[h]
\centering
\caption{Synthetic simulations: exact TV distance between $\Pcond{\xhat}$ and
$\target$ (lower is better), mean $\pm$ std over five runs on an NVIDIA L40S.
\mlgd{} and \mlgdf{}: ten restarts and $100$ refinement steps per run, about
$8$\,min and $50$\,s. \mlgd{}$^\star$: the \textsc{FiniteDiff} gradient of
Algorithm~\ref{alg:inner-estimator}, refinement keeping the best iterate and
$30$ restarts with re-estimated selection, with a wider conditional sampler;
about $25$\,min. GB-FESO \citep{zick2026gradient}: Adam on $x$ through the same
sampler with its KDE-KL loss, at the budget of \mlgd{}$^\star$
(Appendix~\ref{sec:app:sim:gbfeso}). Lowest mean in bold. $\mathrm{TV} = 0.3$ corresponds to $\xhat$ about $0.1$ std of
$X$ from $\xstar$. Ablations and $L^2$ distances:
Appendix~\ref{sec:app:sim:results}.}

\label{tab:sim_main}
\setlength{\tabcolsep}{3pt}
\begin{tabular}{@{}l cccc@{}}
\toprule
$(\dimx, \dimy)$ & \mlgd{} & \mlgdf{} & \mlgd{}$^\star$ & GB-FESO \\
\midrule

(1, 1) & {\boldmath $.02{\pm}.01$} & $.08{\pm}.01$ & $.03{\pm}.01$ & $.05{\pm}.01$ \\
(1, 5) & {\boldmath $.02{\pm}.01$} & $.08{\pm}.02$ & {\boldmath $.02{\pm}.01$} & $.04{\pm}.01$ \\
(1, 10) & $.07{\pm}.07$ & $.38{\pm}.05$ & {\boldmath $.05{\pm}.03$} & $.18{\pm}.09$ \\
(5, 1) & $.08{\pm}.01$ & $.25{\pm}.09$ & {\boldmath $.02{\pm}.01$} & {\boldmath $.02{\pm}.01$} \\
(5, 5) & $.06{\pm}.02$ & $.50{\pm}.05$ & $.06{\pm}.02$ & {\boldmath $.04{\pm}.01$} \\
(5, 10) & $.14{\pm}.03$ & $.41{\pm}.04$ & {\boldmath $.07{\pm}.03$} & $.18{\pm}.05$ \\
(10, 1) & $.49{\pm}.08$ & $.56{\pm}.03$ & {\boldmath $.07{\pm}.03$} & $.55{\pm}.16$ \\
(10, 5) & $.13{\pm}.04$ & $.30{\pm}.01$ & $.07{\pm}.02$ & {\boldmath $.04{\pm}.01$} \\
(10, 10) & $.20{\pm}.05$ & $.92{\pm}.02$ & {\boldmath $.19{\pm}.18$} & $.30{\pm}.18$ \\
\bottomrule
\end{tabular}
\end{table}

\subsection{MNIST: Rotated Digits}
\label{sec:exp:mnist}
\paragraph{Data and models.}
From MNIST~\citep{lecun1998gradient} we built $82{,}483$ pairs $(x_i, y_i)$ of a
randomly rotated digit $x_i \in \mathbb{R}^{\dimx}$, $\dimx = 784$, and its angle $y_i =
(\cos\theta_i, \sin\theta_i) \in \mathbb{R}^{\dimy}$, $\dimy = 2$, $\theta_i \sim \mathrm{U}[0, 360)$; symmetric
digits (``0'', ``1'', ``8'', ``6/9'') contribute several valid orientations
(Appendix~\ref{app:mnist:data}).
The prior score network $\stheta$ is an unconditional \ddpm{} with a
UNet backbone on the rotated images. The fast conditional sampler
$\fphi$ is a conditional \ict{} model that samples the angle $y$
given an image $x$: a small convolutional network encodes $x$ into a
$128$-dimensional
embedding that is added to the noisy angle before an MLP backbone,
and the output is normalized to the unit circle. Architectures and
training are in Appendix~\ref{app:mnist:arch} and
Table~\ref{tab:mnist_hparams}.

\paragraph{Setup.}
The sample distance $\distHat$ is the \swd{} (Appendix~\ref{app:mnist:swd}),
computed between $\ncond$ angle samples from $\fphi(\xzero, \cdot)$ and a sample
from the target; \upd{the reported \swd{} is computed the same way at $\xhat$, so it
measures the match under $\fphi$, and the digit classifier checks that $\xhat$
is a digit rather than an input that merely misleads $\fphi$.}
We use three targets on the angle: uniform on $[0, 360)$, bimodal
(equal Gaussian modes at $0^\circ$ and $180^\circ$, $\sigma \approx
17^\circ$), and unimodal (a Gaussian at $0^\circ$, $\sigma \approx
23^\circ$). For each target we ran \mlgdf{} $15$ times from
independent initializations $x_T \sim \mathcal{N}(0, I)$
(hyperparameters in Table~\ref{tab:mnist_lgd_hparams}) and report
all $15$ runs and the top-$5$ by final \swd{}. Each recovered image
is labeled by a rotation-robust digit classifier
(Appendix~\ref{app:mnist:arch}) to read off the digit classes found.

\paragraph{Results.}
Since the unconditional model generates any digit at any orientation, the search
runs over digit identity and rotation jointly, and the runs are approximate
samples from the tilted distribution $\Qb$ (Problem~\ref{prob:sampling}), whose
empirical realization is the digit-class distribution of Table~\ref{tab:mnist}.
Although the objective is agnostic to digit identity, \mlgdf{} recovers coherent
digits whose orientation distribution matches $\target$ (Figure~\ref{fig:mnist};
all $15$ runs per target in
Figures~\ref{fig:app_bimodal_full}--\ref{fig:app_unimodal_full},
Appendix~\ref{app:mnist:fullresults}): symmetric
digits dominate under the uniform target (digit~0) and the bimodal target
(narrow 0, 1, 8, 6, 9), while oriented digits (2, 3, 7) emerge under the
unimodal target.

\begin{table}[h]
\centering
\caption{MNIST rotated digits (15 seeds per target): SWD over all runs and
the top 5, and recovered digit classes (\% of classified seeds; up
to 4 per target fell below the classifier threshold $0.8$); runs take
8--18\,s.
}
\label{tab:mnist}
\setlength{\tabcolsep}{2.5pt}
\begin{tabularx}{\columnwidth}{@{}l cc >{\raggedright\arraybackslash}X@{}}
\toprule
Target & SWD\textsubscript{all} & SWD\textsubscript{top5} & Recovered digits (\%) \\
\midrule
$\target_{\mathrm{unif.}}$ & $.049 {\pm} .04$ & $.029 {\pm} .01$ & 0\,(91), 6\,(9) \\
$\target_{\mathrm{bim.}}$  & $.081 {\pm} .07$ & $.034 {\pm} .01$ & 8\,(27), 0\,(18), 1\,(18), 6\,(18), 5\,(9), 9\,(9) \\
$\target_{\mathrm{unim.}}$ & $.050 {\pm} .02$ & $.034 {\pm} .01$ & 2\,(38), 3\,(31), 7\,(23), 8\,(8) \\
\bottomrule
\end{tabularx}
\end{table}

\begin{figure}[t]
    \centering
    \includegraphics[width=\columnwidth]{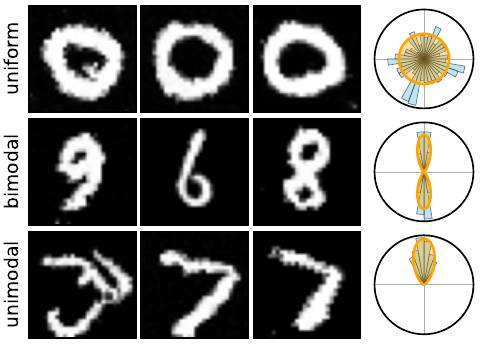}
    \caption{\textbf{MNIST rotation optimization.} One row per target
    (uniform, bimodal, unimodal): the three optimized digits $\xhat$ with the
    lowest \swd{} of $15$ runs, and the polar histogram of $\Pcond{\xhat}$ for
    the first digit (blue) against $\target$ (orange; $0^\circ$ at the top,
    clockwise).}
    \label{fig:mnist}
\end{figure}

\begin{figure*}[t]
    \centering
    \includegraphics[width=\textwidth]{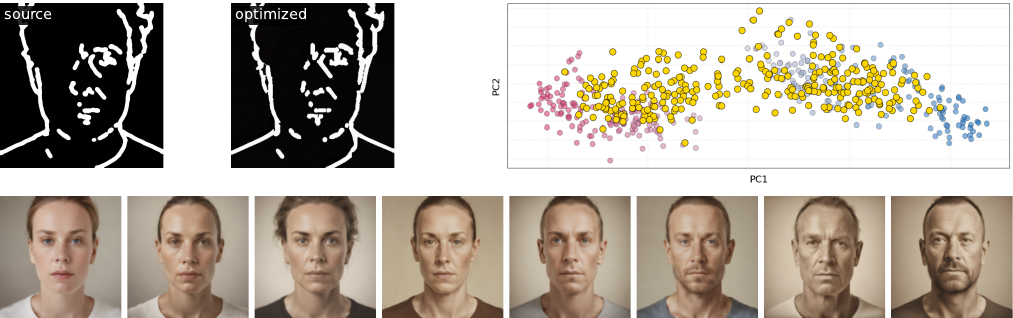}
    \caption{
        \textbf{Top left:} the source male-portrait scribble and the scribble
        \mlgdf{} optimized for the balanced ($50/50$) target; the pixel-level
        change is small despite the large shift in gender distribution
        (Appendix~\ref{app:sd:interp}).
        \textbf{Top right:} principal components (PC) of the \clip{} embeddings
        of the gender-interpolation target (background, man in blue to woman in
        red along PC1) and of the \mlgdf{} outputs (gold), which spread over the
        whole range.
        \textbf{Bottom:} $8$ of $100$ images generated from the
        interpolation-target scribble, ordered by \clip{} PC1 from feminine (left) to
        most masculine (right).
    }
    \label{fig:sd_main}
\end{figure*}

\subsection{Stable Diffusion: Distribution-Guided Image Generation}
\label{sec:exp:sd}
\paragraph{Data, models and targets.}
$X \in \mathbb{R}^{\dimx}$, $\dimx = 16{,}384$, is the Stable Diffusion XL
(\sdxl{})~\citep{podell2024sdxl} latent of a \emph{scribble}, an edge map of a
face; the source is the scribble of a generated male portrait, a different one
per target (e.g.\ Figure~\ref{fig:sd_main}, top left). $Y \in \mathbb{R}^{\dimy}$, $\dimy = 768$,
is the \clip{} embedding~\citep{clip} of a portrait generated from the scribble.
The prior $\stheta$ is the \sdxlbase{} UNet, a generic image prior run on scribble
latents. The conditional sampler $\fphi$ is a frozen chain: the \sdxl{} decoder,
\controlnetscribble{}~\citep{controlnet} with
\sdxlturbo{}~\citep{sauer2023adversarialdiffusiondistillation} in $\Ks = 2$
steps, and the \clip{} image encoder; the gradient is backpropagated through all
of it (Appendix~\ref{app:sd:hparams}). Each target $\target$ is a sample of
\clip{} embeddings of portraits generated with gender or age prompts
(Appendix~\ref{app:sd:targets}): \textbf{balanced} ($50\%$ male),
\textbf{skewed} ($25\%$ male), and \textbf{gender} and \textbf{age
interpolation}, which approximate continuous one-dimensional supports.

\paragraph{Setup.}
$\distHat$ is the \mmd{} in \clip{} space with $\ncond = 100$
($120$ for age; Appendix~\ref{app:sd:loss}) and Uniform selection of half the samples, which
cuts the time per step by $42\%$ (Table~\ref{tab:select-runtime}). Optimization
starts from the source scribble noised halfway, as in \sdedit{}~\citep{sdedit} (Appendix~\ref{app:sd:init}).
To our
knowledge no existing method addresses \cdm{} at this scale (the closest,
GB-FESO, optimizes low-dimensional designs without a prior), so we compare against two baselines (Appendix~\ref{app:sd:baselines}):
\textbf{Average prompt}, which generates from the target components' averaged \clip{}
text embedding~\citep{yang2024impusimagemorphingperceptuallyuniform}, and
\textbf{\sdedit{} Best}, the best of unguided \sdedit{} candidates under the same
wall-clock budget. Every method is evaluated on $2{,}000$ portraits generated
from its scribble, by \mmd{} and \swd{} against a fresh target sample and, for the
discrete targets, by a \clip{} gender classifier (Appendix~\ref{app:sd:eval}).

\paragraph{Results.}
Across all four targets \mlgdf{} attains the lowest \mmd{}, $31$--$37\%$ below the
source scribble, and the lowest \swd{} (Tables~\ref{tab:sd_results},
\ref{tab:sd_swd}; Appendix~\ref{app:sd:swd}). The induced male fraction is
$49.3\%$ for the balanced ($50\%$) target and $24.6\%$ for the skewed ($25\%$) one,
both inside the classifier's $95\%$ interval. Average prompt lowers the \mmd{} by
$11$--$18\%$ but gives $41\%$ and $21\%$ male on these two targets, and \sdedit{}
Best, with an equal time budget, is worse on every target ($65\%$ and $43\%$ male;
Figures~\ref{fig:app_bimodal_qual}--\ref{fig:app_interp_qual}).
Figure~\ref{fig:sd_main} shows the balanced-target scribble and the gender
interpolation induced by the interpolation-target scribble; age grids are in
Appendix~\ref{app:sd:qualitative}. The few-step sampler makes this feasible: with
the $30$-step \sdxlbase{} in the inner loop the same run would need $372$\,GB of
GPU memory and $66$\,hr, against $43$\,GB and $4.4$\,hr (Appendix~\ref{app:ablation}).

\section{Conclusion}
\label{sec:conclusion}

We introduced \cdm{}, a problem class that generalizes inverse design from a point $y^*$ to a target distribution,
seeking $\xstar$ whose induced conditional $\Pcond{\xstar}$ matches
$\target(Y)$. We formalized its sampling and optimization variants, cast both as
loss-guided diffusion, and introduced \mlgd{}, whose plug-in estimator makes the
distribution-matching loss and its gradient tractable at inference time with any
differentiable conditional sampler.

\paragraph{Limitations and future work.}
Solution quality is bounded by the accuracy of the conditional sampler $\fphi$,
and the guided trajectory only approximates $\Qb$ (Remark~\ref{rem:lgd-approx});
twisted SMC removes this bias at extra compute (Figure~\ref{fig:sim_2d}(e)). Because the loss is estimated from samples, a single run can land in a poor
basin, and several runs are often needed. The backpropagated gradient also
requires a sampler that is differentiable in $x$; the finite-difference variant
relaxes this, but its cost grows with the dimension of $X$. Discrete $X$ is the
natural next step: e.g., designing a protein sequence whose structure ensemble
matches a prescribed mixture of folds, as fold-switching (``metamorphic'')
proteins do~\citep{murzin2008metamorphic, porter2018extant}, is \cdm{} with a
discrete $X$.
\clearpage %

\clearpage %
\bibliography{references}
\bibliographystyle{plainnat}

\clearpage
\appendix
\onecolumn
\begin{center}{\LARGE\bfseries Inverse Design for Conditional Distribution Matching:\\ Supplementary Materials\par}\end{center}

\section{Code and Reproducibility}
\label{app:code}
All code, configurations and reference results are in one repository,
\href{\repomain}{\texttt{\reponame}}, organized by experiment; each experiment
directory has a README with the command, runtime and hardware:
\begin{itemize}
  \item \textbf{Synthetic simulations} (Appendix~\ref{sec:app:sim}):
  \href{\reposim}{\texttt{experiments/synthetic/}}: the scripts in
  \texttt{pipeline/}, cluster jobs in \texttt{slurm/}, fixed mixture parameters
  in \texttt{configs/} and the outputs of our runs in \texttt{reference/}; its
  README gives the command of every table, figure and ablation.
  \item \textbf{MNIST} (Appendix~\ref{sec:app:mnist}):
  \href{\repomnist}{\texttt{experiments/mnist/}}, with the network training,
  the runs and the script that builds Table~\ref{tab:mnist}. The
  unconditional UNet and conditional \ict{} checkpoints are at
  \href{\hfmnist}{\texttt{MNIST}} on Hugging Face.
  \item \textbf{Stable Diffusion} (Appendix~\ref{app:sd}):
  \href{\reposd}{\texttt{experiments/stable\_diffusion/}}, with the
  optimization for the four targets, the baselines, the evaluation, the
  saliency analysis, the interpolation figures and the per-step cost;
  the Jacobian-error notebook is  \href{\repoeps}{\texttt{eps\_g\_experiment.ipynb}}
  (Appendix~\ref{app:MLGD-ESTIMATOR:teacher-student-connection}).
  All pretrained models (\sdxlbase{}, \sdxlturbo{},
  \controlnetscribble{}, \clip{} ViT-L/14) are public; licenses are
  listed in Appendix~\ref{app:licenses}.
\end{itemize}

\section{Synthetic Simulations: Experimental Details}
\label{sec:app:sim}

This appendix provides the full experimental setup for the synthetic
simulations described in Section~\ref{sec:exp:sim}.
\subsection{\mmd{} Loss: Population and Sample Forms}
\label{app:mmd-def}

The \mmd{} loss is defined once, here;
the Stable Diffusion variant (Appendix~\ref{app:sd:loss}) and the
theory (Appendix~\ref{app:theory}) refer back to it.
The distance $\dist$ of Section~\ref{sec:problem} is instantiated
with the squared Maximum Mean
Discrepancy~\citep{gretton2008kernelmethodtwosampleproblem}. The
\emph{population} loss is
\[
\loss(x) = \operatorname{MMD}^2\!\bigl(\Pcond{x},\,\target\bigr)
= \mathbb{E}\bigl[k(Y,Y')\bigr] - 2\,\mathbb{E}\bigl[k(Y,\tilde{Y})\bigr] +
\mathbb{E}\bigl[k(\tilde{Y},\tilde{Y}')\bigr],
\]
where $Y,Y' \sim \Pcond{x}$ and
$\tilde{Y},\tilde{Y}' \sim \target$ are independent pairs, and $k$
is a characteristic positive-definite kernel.

Given conditional samples $\{y_i\}_{i=1}^{n}$ and target samples
$\{\tilde{y}_j\}_{j=1}^{m}$, two sample versions $\distHat$ are used, and the
sample loss is $\lossHat(x) = \distHat(\{y_i\}, \{\tilde{y}_j\})$. The unbiased
$U$-statistic is
\begin{equation}
\widehat{\operatorname{MMD}}^2_U\!\bigl(\{y_i\},\{\tilde{y}_j\}\bigr)
= \frac{1}{n(n-1)}\sum_{i \neq i'} k(y_i,y_{i'})
- \frac{2}{nm}\sum_{i,j} k(y_i,\tilde{y}_j)
+ \frac{1}{m(m-1)}\sum_{j \neq j'} k(\tilde{y}_j,\tilde{y}_{j'}),
\label{eq:mmd-ustat-def}
\end{equation}
and the biased $V$-statistic replaces the two within-sample averages
by $\frac{1}{n^2}\sum_{i,i'}$ and $\frac{1}{m^2}\sum_{j,j'}$:
\begin{equation}
\widehat{\operatorname{MMD}}^2_V\!\bigl(\{y_i\},\{\tilde{y}_j\}\bigr)
= \frac{1}{n^2}\sum_{i,i'} k(y_i,y_{i'})
- \frac{2}{nm}\sum_{i,j} k(y_i,\tilde{y}_j)
+ \frac{1}{m^2}\sum_{j,j'} k(\tilde{y}_j,\tilde{y}_{j'}).
\label{eq:mmd-vstat-def}
\end{equation}
Gradients flow only through the generated samples $\{y_i\}$; the
target samples $\{\tilde{y}_j\}$ are detached. In the synthetic
experiments the kernel is the multi-bandwidth RBF
\begin{equation}
k(a,b) = \sum_{\ell=1}^{5} \exp\!\Bigl(-\frac{\|a-b\|^2}{\sigma \cdot
2^{\ell-3}}\Bigr),
\label{eq:mmd-rbf}
\end{equation}
where $\sigma$ is the mean pairwise squared distance of the current merged
sample. The analysis (Appendix~\ref{app:theory}) treats $\sigma$ as fixed; in the
reported synthetic runs the gradient also flowed through it, which changes no
result (Appendix~\ref{sec:app:sim:step}, paragraph ``Kernel bandwidth and
schedules'').
The synthetic experiments use the $V$-statistic~\eqref{eq:mmd-vstat-def} with
the kernel~\eqref{eq:mmd-rbf}. The Stable Diffusion experiment uses the
$U$-statistic~\eqref{eq:mmd-ustat-def} with a single RBF kernel whose bandwidth
is set from the median pairwise distance (Appendix~\ref{app:sd:loss}). The MNIST
experiment uses the \swd{} (Appendix~\ref{app:mnist:swd}).
Appendix~\ref{app:theory} states its results for the $U$-statistic. The
$V$-statistic differs from it only through the diagonal kernel terms, by at most
$2K_{\max}(1/n + 1/m)$ for a kernel bounded by $K_{\max}$, so the bounds hold
for it with this additional term.
\subsection{\texorpdfstring{$L^2$}{L2} Distance Between GMMs}
\label{sec:app:sim:l2}
Because both $\Pcond{\xhat}$ and $\target(Y)$ are
Gaussian Mixture Models in the synthetic setting, we evaluate solution
quality with the exact $L^2$ distance between their densities $\rho$ and $\gamma$:
\[
\|\rho - \gamma\|_{L^2}^2 = \int \bigl(\rho(y) - \gamma(y)\bigr)^2\, dy
= \|\rho\|_{L^2}^2 - 2\langle \rho, \gamma \rangle_{L^2} + \|\gamma\|_{L^2}^2,
\]
where each inner product $\langle p_1, p_2\rangle_{L^2}$ between two
GMMs decomposes into a sum of Gaussian--Gaussian inner products.
Each such inner product admits the closed-form evaluation
\citep{petersen2012matrix}:
\[
\langle \mathcal{N}(\mu_1,\Sigma_1),\, \mathcal{N}(\mu_2,\Sigma_2)\rangle_{L^2}
= \mathcal{N}(\mu_1;\,\mu_2,\,\Sigma_1+\Sigma_2).
\]
This gives an exact evaluation of solution quality,
used as our primary metric for comparing \mlgdf{} against \mlgd{} in all
synthetic settings. We use $L^2$ GMM only for \emph{evaluation}.

\subsection{Total Variation Distance}
\label{sec:app:sim:tv}
The primary metric of the grid scenarios (Tables~\ref{tab:sim_main}
and~\ref{tab:sim_gbfeso}) is the total variation
distance between the conditional induced by the returned design and the target,
\[
\mathrm{TV}\bigl(\Pcond{\xhat}, \target\bigr)
= \tfrac12 \int \bigl|\rho(y) - \gamma(y)\bigr|\, dy \;\in\; [0, 1],
\]
It is bounded and invariant to invertible
reparametrizations of $Y$, so its values are comparable across $\dimy$ and
are unaffected by the whitening of Appendix~\ref{sec:app:sim:params}.
Here $\rho$ and $\gamma$ are the densities of $\Pcond{\xhat}$ and $\target$.
Both are known in closed form: $\target$ is a Gaussian mixture, and
$\Pcond{\xhat}$ is the Gaussian mixture obtained by conditioning every component
of $\PP(X,Y)$ on $X = \xhat$ and reweighting it by its marginal density at
$\xhat$ (components with conditional weight below $10^{-8}$ are dropped).
The integral has no closed form, so we estimate it by Monte Carlo under the
equal mixture $m = \tfrac12(\rho + \gamma)$:
\[
\mathrm{TV} = \mathbb{E}_{y \sim m}\!\left[\frac{|\rho(y) - \gamma(y)|}{\rho(y) + \gamma(y)}\right]
\approx \frac{1}{N} \sum_{s=1}^{N} \frac{|\rho(y_s) - \gamma(y_s)|}{\rho(y_s) + \gamma(y_s)},
\]
with $N = 20{,}000$ points, half drawn from $\Pcond{\xhat}$ and half from $\target$, and the exact
densities in the ratio. Each term lies in $[0, 1]$, so the estimator is unbiased
with standard error at most $1/(2\sqrt{N}) \approx 0.004$. A run is scored by the
TV of the $\xhat$ it return. The metric is used only for evaluation; the optimizer never sees it.

\subsection{Joint Distribution Parameters}
\label{sec:app:sim:params}
\paragraph{Grid scenarios (Table~\ref{tab:sim_main}).}
For each $(\dimx, \dimy)$ the joint $\PP(X,Y)$ is a random four-component
Gaussian mixture in $d = \dimx + \dimy$ dimensions with uniform weights:
component means are drawn as $\mathcal{N}(0, 25 I_d)$, rejected while closer
than $5$ to an earlier mean, and covariances as $A A^\top + 10^{-4} I_d$ with
$A$ a $d \times d$ standard Gaussian matrix.
Two components are then given the same $x$-part of their mean; that shared
vector is $\xstar$, and $\target(Y) = \Pcond{\xstar}$ is bimodal, obtained by
keeping the components whose conditional weight at $\xstar$ exceeds $0.01$ and
renormalizing.
Finally the mixture is whitened block-wise: with $m$ its mean and $C$ its
covariance (both exact), every component is mapped by $\mu \mapsto W(\mu - m)$,
$\Sigma \mapsto W \Sigma W^\top$ with $W =
\operatorname{blockdiag}(C_{xx}^{-1/2}, C_{yy}^{-1/2})$, so that $X$ and $Y$
each have zero mean and identity covariance.
This is an affine reparametrization of the same problem: $X$ and $Y$ are
transformed separately, so the conditional structure, $\xstar$ and the
affine-invariant distances of Table~\ref{tab:sim_main} are unchanged, while the
networks see data on the unit scale their noise schedules assume.

\subsection{Model Architectures and Training}
\label{sec:app:sim:arch}
The two diffusion models (the prior and the conditional sampler) are residual
MLPs with 3 blocks of 128 units, each block two linear layers with layer
normalization and Swish activations. A sinusoidal time embedding, projected by
a two-layer MLP, is added to the hidden state once before the blocks. The
consistency model is a plain MLP with 3 hidden layers of 128 units and ReLU
activations, with a sinusoidal time embedding added to its input layer.
The diffusion models are optimized with AdamW
\citep{loshchilov2019decoupledweightdecayregularization} (learning rate
$10^{-3}$, weight decay $10^{-2}$) under a cosine-annealing learning-rate
schedule over the training steps, and the consistency model with AdamW
(learning rate $10^{-4}$, weight decay $10^{-2}$) at a constant learning rate.
For the grid scenarios of Table~\ref{tab:sim_main}, every network is trained on
fresh draws from the analytic mixture, one batch of $1{,}024$ per step, for
$10{,}000$ steps.
\mlgd{}$^\star$ uses a more accurate conditional sampler: a diffusion model of
the same form with 4 blocks of 256 units, trained in the same way for
$100{,}000$ steps. It uses the same prior, data and target as the other methods.

\paragraph{Unconditional diffusion model.}
A \ddpm{} \citep{DDPMsong} trained on marginal samples $x \sim \PP(X)$, with a
cosine noise schedule ($s = 0.008$) and $T = 100$ steps. Sampling uses
\ddim{} \citep{DDIMSong} with $\eta = 0$, where $\eta$ scales the noise
injected at each reverse step ($\eta = 0$ is deterministic, $\eta = 1$
recovers \ddpm{} ancestral sampling).

\paragraph{Conditional diffusion model.}
A conditional \ddpm{} for $\Pcond{x}$ with the same noise schedule. It denoises
the full vector $(x, y)$ given $x$, and only the $y$ part of its output is used.
Conditioning is a linear projection of $x$ added to the input projection of the
MLP. During training the condition is dropped with probability $0.2$, as in
classifier-free guidance \citep{classifierfree}; sampling uses the conditional
prediction only, without a guidance weight, by \ddim{} with $\eta = 0$.

\paragraph{Conditional consistency model.}
An improved Consistency Training (\ict{}) model \citep{ConsistencyModelts2}
approximates $\Pcond{x}$. We use the \ict{} discretization curriculum
\begin{equation}
    N(k) = \min\bigl(s_0 \cdot 2^{\lfloor k/K'\rfloor},\, s_1\bigr) + 1,
    \qquad K' = K / \lfloor \log_2(s_1/s_0) \rfloor,
\end{equation}
with $s_0 = 10$, $s_1 = 1280$ and $K$ the number of training steps; the $N(k)$
training noise levels span $\sigma \in [0.002, 7]$. The loss is the
pseudo-Huber loss ($c = 0.00054\sqrt{\dimy}$) with weighting
$\lambda(t) = 1/(t_{n+1} - t_n)$, the online network serving as its own target
(no EMA), and noise levels drawn from the lognormal proposal of
\citet{ConsistencyModelts2} ($P_{\mathrm{mean}} = -1.1$,
$P_{\mathrm{std}} = 2$). Conditioning is a linear projection of $x$ added to
the input projection. Sampling follows the multistep \ict{} procedure with the
15 noise levels $[80, 40, 20, 10, 5, 2, 1, 0.5, 0.25, 0.125, 0.062, 0.031,
0.015, 0.007, 0.002]$.

\subsection{Network Quality Against the Analytic Mixture}
\label{sec:app:sim:quality}
Before the search stage every trained network is compared with the analytic mixture it approximates
.
For the prior: the score network against the exact score of the noised mixture,
which is itself a mixture with means $\sqrt{\bar\alpha_t}\,\mu_k$ and
covariances $\bar\alpha_t \Sigma_k + (1 - \bar\alpha_t) I$, as the relative
Fisher divergence $\mathbb{E}\|\stheta(x_t, t) - \nabla \log p_t(x_t)\|^2 /
\mathbb{E}\|\nabla \log p_t(x_t)\|^2$ at twelve noise levels, and its unguided
samples against $\PP(X)$ by the normalized sliced Wasserstein distance
(Appendix~\ref{sec:app:swd}).
For each conditional sampler: its draws at $x$ against the exact $\Pcond{x}$, by
the normalized sliced Wasserstein distance, averaged over $100$ random $x \sim
\PP(X)$ and at $\xstar$ itself.
The value at $\xstar$ is the floor of the plug-in loss: a search that reaches
the sampler's minimum still returns an $\xhat$ whose induced conditional differs
from $\target$ by at least the sampler's own error.
Table~\ref{tab:model_quality} lists these numbers for the nine scenarios.
At $(10,1)$ the conditional \ddpm{} is far from $\Pcond{\xstar}$ ($0.88$), so the
plug-in loss has a high floor there; the retrained sampler of \mlgd{}$^\star$
lowers this error to $0.015$.

\begin{table*}[t]
\centering
\caption{Trained networks against the analytic mixture, per scenario: worst
relative Fisher error of the score network over noise levels $t > 1$; normalized
\swd{} of the prior's unguided samples to $\PP(X)$; normalized \swd{} of each
conditional sampler to $\Pcond{\xstar}$ at $\xstar$ (the loss floor) and
averaged over random $x$. The last two columns are the retrained sampler used by
\mlgd{}$^\star$. Lower is better.}
\label{tab:model_quality}
\begin{tabular}{@{}l c c cc cc cc@{}}
\toprule
& score net & prior & \multicolumn{2}{c}{conditional \ddpm{}} & \multicolumn{2}{c}{\ict{}}
  & \multicolumn{2}{c}{\ddpm{}, $256$ units} \\
\cmidrule(lr){4-5} \cmidrule(lr){6-7} \cmidrule(lr){8-9}
Setting & Fisher (rel.) & \swd{} & at $\xstar$ & random $x$ & at $\xstar$ & random $x$ & at $\xstar$ & random $x$ \\
\midrule

$\dim X=1$, $\dim Y=1$ & 0.001 & 0.039 & 0.033 & 0.121 & 0.153 & 1.230 & 0.027 & 0.148 \\
$\dim X=1$, $\dim Y=5$ & 0.005 & 0.023 & 0.034 & 0.043 & 0.168 & 0.155 & 0.025 & 0.029 \\
$\dim X=1$, $\dim Y=10$ & 0.003 & 0.034 & 0.035 & 0.042 & 0.104 & 0.101 & 0.031 & 0.030 \\
$\dim X=5$, $\dim Y=1$ & 0.006 & 0.027 & 0.103 & 0.206 & 0.112 & 0.685 & 0.050 & 0.085 \\
$\dim X=5$, $\dim Y=5$ & 0.010 & 0.034 & 0.051 & 0.100 & 0.227 & 0.285 & 0.029 & 0.038 \\
$\dim X=5$, $\dim Y=10$ & 0.009 & 0.031 & 0.038 & 0.073 & 0.084 & 0.175 & 0.029 & 0.037 \\
$\dim X=10$, $\dim Y=1$ & 0.010 & 0.032 & 0.882 & 1.169 & 0.502 & 1.673 & 0.015 & 0.116 \\
$\dim X=10$, $\dim Y=5$ & 0.015 & 0.028 & 0.031 & 0.182 & 0.178 & 0.356 & 0.028 & 0.049 \\
$\dim X=10$, $\dim Y=10$ & 0.017 & 0.030 & 0.067 & 0.177 & 0.294 & 0.374 & 0.028 & 0.053 \\
\bottomrule
\end{tabular}
\end{table*}

\subsection{Optimization Hyperparameters}
\label{sec:app:sim:opt}
\paragraph{Guidance step rule.}
\label{sec:app:sim:step}
Equation~\ref{eq:guidance-step} with a constant step $\zetat = \zeta'$ on the
raw gradient is not scale-free: in the whitened coordinates of the grid
scenarios the gradient norm grows with the dimension, and no single $\zeta'$
works across them.
With the consistency sampler and three runs per cell, the best constant is
$\zeta' \approx 0.03$--$0.1$ at $(\dimx, \dimy) = (1,1)$ ($\mathrm{TV}$
$0.09$--$0.11$) but $0.3$ at $(5,5)$ ($0.50$), and every constant fails at
$(10,10)$ ($\ge 0.90$); $\zeta' = 1$ sends trajectories off the support in most
scenarios ($\mathrm{TV} \approx 1$).
We therefore replace the raw gradient by its Adam-normalized version
\citep{belardi2026adaptivemomentssurprisinglyeffective}: with $g_t =
\nabla_{x_t}\lossHat(\xzero)$, $m_t = \beta_1 m_{t+1} + (1-\beta_1) g_t$, $v_t =
\beta_2 v_{t+1} + (1-\beta_2) g_t^2$ (bias-corrected, $\beta_1 = 0.9$, $\beta_2
= 0.999$) and $x_{t-1} = \mathrm{DDIM}(x_t) - \zeta'\, m_t / (\sqrt{v_t} +
10^{-8})$, so that $\zeta'$ is a displacement in $x$ units.
A single $\zeta' = 0.05$ then works across the grid (same setup: $\mathrm{TV}$
$0.10$ at $(1,1)$, $0.47$ at $(5,5)$, $0.39$ at $(1,10)$; steps $0.02$--$0.2$
are within noise of each other except $0.2$ at $(1,1)$), and
Table~\ref{tab:sim_main} uses it for every method.

\paragraph{Post-refinement ($\algRefine{}$).}
After the reverse process, each returned $x_0$ is polished by $100$ Adam steps
(learning rate $0.02$) on the plug-in loss $\lossHat(x)$ of
Algorithm~\ref{alg:inner-estimator} evaluated at $x$ itself, with $1{,}000$
conditional samples per step instead of $250$ and no diffusion.
It removes the positioning noise of the guided trajectory within its basin:
Table~\ref{tab:sim_full}.
For \mlgd{}$^\star$ the refinement uses $3{,}000$ conditional samples per step
and returns the iterate with the smallest loss estimate rather than the last
one, $\arg\min_{0 \le k < 100} \lossHat(x_0^{(k)})$ over the Adam iterates
$x_0^{(k)}$ (with $x_0^{(0)} = x_0$). The estimates are those already computed
for the gradient steps, so this costs nothing extra. Returning the last iterate
can end on a worse point than the start, because each step uses a fresh, noisy
loss estimate.

\paragraph{Kernel bandwidth and schedules.}
The \mmd{} of Table~\ref{tab:sim_main} uses the multi-bandwidth RBF kernel with
bandwidths set from the pairwise distances of the current samples; in the
reported runs the gradient also flowed through the bandwidths, a legacy of the
implementation.
Repeating the full grid with the bandwidths detached (now the default of the
library) reproduces every cell within run-to-run noise (e.g.\ \mlgd{}
$\mathrm{TV}$ $0.053 \pm 0.030$ against $0.060 \pm 0.018$ at $(5,5)$ and $0.33
\pm 0.13$ against $0.20 \pm 0.05$ at $(10,10)$), so the choice is immaterial for
the conclusions.
Likewise, ramping $\ncond$ from $50$ to $250$ along the trajectory together with
a cosine decay of the guidance step worsens no cell at $\dimx \le 5$ by more
than $0.04$ in TV and hurts \mlgd{} at $(10,10)$ ($0.20 \to 0.43$; it helps \mlgdf{} at
$(5,1)$, $0.25 \to 0.13$), so Table~\ref{tab:sim_main} keeps the constant
settings.
Replacing the \mmd{} by the sliced Wasserstein loss at $(5,1)$ changes nothing
either (\mlgd{} $0.072$ against $0.076$, \mlgdf{} $0.20$ against $0.25$).

\paragraph{Sampler training.}
The gap between \mlgdf{} and \mlgd{} at $\dimx \ge 5$ persists when the
networks are trained harder.
Retraining all three networks of the cells $(5,5)$, $(5,10)$, $(10,5)$,
$(10,10)$ for $50{,}000$ instead of $10{,}000$ batches, or with $256$ instead of
$128$ hidden units, leaves the consistency sampler's normalized \swd{} at
$\xstar$ in the same range ($0.09$--$0.36$ against $0.09$--$0.28$ for the
reported models, with the diffusion sampler at $0.04$--$0.06$ throughout) while \mlgdf{} moves by much more, in both directions
($(5,5)$: $0.50 \to 0.82$ and $0.63$; $(5,10)$: $0.41 \to 0.41$ and $0.32$;
$(10,5)$: $0.30 \to 0.64$ and $0.47$; $(10,10)$: $0.92 \to 0.88$ and $0.96$),
while \mlgd{} stays within noise of Table~\ref{tab:sim_main}.
The plug-in loss of \mlgdf{} is dominated by the region where the consistency
model's error is smallest rather than by the target; closing the gap requires a
better few-step sampler, e.g.\ one distilled from the trained conditional
diffusion model.

\subsection{Results}
\label{sec:app:sim:results}
\paragraph{How to read TV.}
The conditional targets are sharp, so $\Pcond{x}$ moves quickly with $x$: in the
whitened coordinates, $\mathrm{TV}(\Pcond{\xstar + \delta u}, \target)$ averaged
over random unit directions $u$ is about $0.04$, $0.12$, $0.37$ and $0.60$ at
$\delta = 0.01, 0.03, 0.1, 0.3$ for $(1,1)$, and $0.09$, $0.25$, $0.52$, $0.78$
for $(10,10)$.
A displacement of a tenth of a standard deviation from $\xstar$ thus already gives
$\mathrm{TV} \approx 0.4$--$0.5$.
The same sensitivity makes a random draw from the prior nearly useless as a
baseline ($\mathrm{TV} \approx 0.7$ at $(1,1)$ and $\approx 1.0$ at $(10,10)$).

Table~\ref{tab:sim_full} complements Table~\ref{tab:sim_main} with the distance of
$\xhat$ to $\xstar$, the relative squared $L^2$ distance $\|\rho - \gamma\|^2 / \|\gamma\|^2$
between $\Pcond{\xhat}$ and $\target$ (closed form; also invariant to the scale
of $Y$), and the ablations.
Per-run values, the fraction of restarts that end in the basin of $\xstar$ and
the raw $L^2$ distances.

\paragraph{The $(10, 1)$ cell.}
Every lever was tried on this scenario with \mlgd{}, Adam-normalized guidance and
refinement: the conditional sampler trained for $50{,}000$ instead of $10{,}000$
batches (its error at $\xstar$ drops from $0.88$ to $0.16$, normalized SWD) gives $\mathrm{TV}$
$0.58 \pm 0.30$; the exact conditional as sampler gives $0.89 \pm 0.10$ (no
refinement); the sliced-Wasserstein loss, whose gradient does not vanish between
well-separated modes, gives $0.81 \pm 0.27$; the table value with the standard
sampler is $0.49 \pm 0.08$.
The geometry of a scalar output against a ten-dimensional input causes the
failure: inside each mixture component the conditional depends on $x$ through a
single direction, so the loss is flat in the other nine and the basin is decided
by the prior at the start of the trajectory.
The cell is a search problem and a larger search budget solves it: with $100$
restarts per run instead of $10$ (three runs), \mlgd{} reaches $\mathrm{TV}$
$0.15 \pm 0.12$ ($70/300$ trajectories below $0.5$, at $10\times$ the cost), the
exact-conditional oracle $0.39 \pm 0.10$ and \mlgdf{} $0.52 \pm 0.12$;
Table~\ref{tab:sim_main} keeps the ten-restart setup of the other cells.

\begin{table}[h]
\centering
\caption{Appendix companion to Table~\ref{tab:sim_main}, whose TV values are
not repeated here (mean $\pm$ std over five runs). Distance of $\xhat$ to
$\xstar$ in the whitened coordinates (standard deviations of $X$; not a measure
of quality when $\dimy < \dimx$, where the minimizer is not unique) and the
relative $L^2$ distance between $\Pcond{\xhat}$ and $\target$, both after
refinement. Ablations, TV of the returned $\xhat$ without refinement: the plain
search, and the search with the exact
conditional of the mixture as sampler (oracle $\fphi$). A run (ten trajectories
plus refinement) takes $496$--$526$\,s with the conditional-diffusion sampler
and $46$--$55$\,s with the consistency sampler on an NVIDIA L40S.}
\label{tab:sim_full}
\setlength{\tabcolsep}{4pt}
\begin{tabular}{@{}l l c c c c@{}}
\toprule
& & \multicolumn{2}{c}{refined} & \multicolumn{2}{c}{TV, no refinement} \\
\cmidrule(lr){3-4} \cmidrule(l){5-6}
Setting & Method & $\|\xhat - \xstar\|$ & relative $L^2$ & plain & oracle $\fphi$ \\
\midrule
\multirow{2}{*}{$\dimx=1$, $\dimy=1$}
 & \mlgd{} & $0.004 \pm 0.003$ & $0.003 \pm 0.003$ & $0.01 \pm 0.01$ & $0.02 \pm 0.01$ \\
 & \mlgdf{} & $0.019 \pm 0.003$ & $0.039 \pm 0.011$ & $0.08 \pm 0.05$ & -- \\
\midrule
\multirow{2}{*}{$\dimx=1$, $\dimy=5$}
 & \mlgd{} & $0.018 \pm 0.011$ & $0.002 \pm 0.002$ & $0.02 \pm 0.01$ & $0.02 \pm 0.01$ \\
 & \mlgdf{} & $0.078 \pm 0.022$ & $0.040 \pm 0.020$ & $0.08 \pm 0.05$ & -- \\
\midrule
\multirow{2}{*}{$\dimx=1$, $\dimy=10$}
 & \mlgd{} & $0.015 \pm 0.014$ & $0.054 \pm 0.087$ & $0.12 \pm 0.08$ & $0.07 \pm 0.04$ \\
 & \mlgdf{} & $0.081 \pm 0.011$ & $0.684 \pm 0.138$ & $0.42 \pm 0.08$ & -- \\
\midrule
\multirow{2}{*}{$\dimx=5$, $\dimy=1$}
 & \mlgd{} & $0.938 \pm 0.249$ & $0.028 \pm 0.007$ & $0.10 \pm 0.06$ & $0.17 \pm 0.17$ \\
 & \mlgdf{} & $1.336 \pm 0.636$ & $0.297 \pm 0.222$ & $0.22 \pm 0.11$ & -- \\
\midrule
\multirow{2}{*}{$\dimx=5$, $\dimy=5$}
 & \mlgd{} & $0.038 \pm 0.023$ & $0.014 \pm 0.007$ & $0.28 \pm 0.09$ & $0.24 \pm 0.15$ \\
 & \mlgdf{} & $0.479 \pm 0.016$ & $0.465 \pm 0.153$ & $0.49 \pm 0.03$ & -- \\
\midrule
\multirow{2}{*}{$\dimx=5$, $\dimy=10$}
 & \mlgd{} & $0.080 \pm 0.012$ & $0.043 \pm 0.007$ & $0.22 \pm 0.06$ & $0.18 \pm 0.08$ \\
 & \mlgdf{} & $0.252 \pm 0.017$ & $0.746 \pm 0.289$ & $0.42 \pm 0.09$ & -- \\
\midrule
\multirow{2}{*}{$\dimx=10$, $\dimy=1$}
 & \mlgd{} & $2.234 \pm 0.947$ & $0.952 \pm 0.331$ & $0.59 \pm 0.24$ & $0.89 \pm 0.10$ \\
 & \mlgdf{} & $1.802 \pm 0.445$ & $1.237 \pm 0.187$ & $0.61 \pm 0.08$ & -- \\
\midrule
\multirow{2}{*}{$\dimx=10$, $\dimy=5$}
 & \mlgd{} & $0.587 \pm 0.275$ & $0.233 \pm 0.122$ & $0.31 \pm 0.07$ & $0.54 \pm 0.15$ \\
 & \mlgdf{} & $1.102 \pm 0.255$ & $0.876 \pm 0.049$ & $0.43 \pm 0.10$ & -- \\
\midrule
\multirow{2}{*}{$\dimx=10$, $\dimy=10$}
 & \mlgd{} & $0.227 \pm 0.041$ & $0.168 \pm 0.100$ & $0.66 \pm 0.03$ & $0.54 \pm 0.20$ \\
 & \mlgdf{} & $1.321 \pm 0.008$ & $2.920 \pm 0.079$ & $0.83 \pm 0.11$ & -- \\
\bottomrule
\end{tabular}
\end{table}

\subsection{Comparison with GB-FESO}
\label{sec:app:sim:gbfeso}
\upd{
\paragraph{Method.}
GB-FESO \citep{zick2026gradient} also searches for a design whose induced
conditional matches a target distribution, but optimizes $x$ directly: from a
starting design it draws a batch from the conditional sampler at $x$, scores it
against target samples with a distributional loss, and takes Adam steps on $x$ by
back-propagating through the sampler. It uses no prior over $X$ and no diffusion
over $X$. Its loss is a kernel-density estimate of the symmetric KL divergence
between the generated and the target samples.

\paragraph{Setup.}

To compare the search algorithms rather than the networks, GB-FESO runs on exactly
the inputs of \mlgd{}$^\star$: the same wider conditional diffusion sampler
(4 blocks, $256$ units) and the same data and target $\target$ in every cell of
Table~\ref{tab:sim_main}, with five runs per cell. GB-FESO keeps its own optimizer,
with the published defaults: two phases of $201$ Adam steps at learning rate $0.05$,
the first with Gaussian perturbations of standard deviation $0.05$ decaying linearly
to $0$, the best iterate of each phase kept, and $x$ clamped to the range of $X$ in
the training data. The authors run the second phase only when the first ends above
a KL threshold. That threshold has no \mmd{} counterpart, so here the second phase
always runs. Every other setting is matched to \mlgd{}$^\star$:
\begin{itemize}
\item \emph{Starts}: $30$ per run, matching the $30$ restarts of \mlgd{}$^\star$;
      each starts at a random design from the training data, since GB-FESO has
      no prior to draw from.
\item \emph{Samples}: $250$ conditional samples per first-phase step (\mlgd{}$^\star$:
      $\ncond = 250$ per guided step) and $3{,}000$ per second-phase step
      (\mlgd{}$^\star$: per refinement step), each against as many fresh target samples.
\item \emph{Selection}: each start is scored on $1{,}000$ fresh samples; the $5$
      lowest are re-scored on $5{,}000$ samples, averaged over $5$ draws, and the
      lowest average is returned, the selection rule of \mlgd{}$^\star$.
\end{itemize}
The KL is not essential to GB-FESO: \citet{zick2026gradient} introduce it as ``a natural
and widely used choice'' of differentiable distributional metric, and note in their
conclusion that higher-dimensional settings may require alternative distributional
objectives. We therefore run GB-FESO with two losses: with its own KDE-KL (bandwidth
$0.5$, symmetric) it is the published method, and with the \mmd{} of
Eq.~\eqref{eq:mmd-rbf} it shares the loss of \mlgd{}$^\star$, so the two then differ
only in the search. A run takes
$27$--$31$\,min on an NVIDIA L40S, against $21$--$27$\,min for \mlgd{}$^\star$.
\begin{table}[h]
\centering
\caption{\mlgd{}$^\star$ against GB-FESO on the same conditional sampler, data and
budget: TV between $\Pcond{\xhat}$ and $\target$, mean $\pm$ std over five runs of
$30$ starts each (lower is better; lowest mean in bold, ties all bold).}
\label{tab:sim_gbfeso}
\begin{tabular}{@{}l ccc@{}}
\toprule
$(\dimx, \dimy)$ & \mlgd{}$^\star$ & GB-FESO (\mmd{}) & GB-FESO (KL) \\
\midrule
(1, 1)   & $\mathbf{0.02 \pm 0.02}$ & $0.03 \pm 0.02$ & $0.05 \pm 0.01$ \\
(1, 5)   & $\mathbf{0.02 \pm 0.01}$ & $\mathbf{0.02 \pm 0.01}$ & $0.04 \pm 0.01$ \\
(1, 10)  & $0.05 \pm 0.03$ & $\mathbf{0.04 \pm 0.03}$ & $0.18 \pm 0.09$ \\
(5, 1)   & $\mathbf{0.02 \pm 0.01}$ & $\mathbf{0.02 \pm 0.01}$ & $\mathbf{0.02 \pm 0.01}$ \\
(5, 5)   & $0.06 \pm 0.02$ & $0.07 \pm 0.02$ & $\mathbf{0.04 \pm 0.01}$ \\
(5, 10)  & $\mathbf{0.07 \pm 0.03}$ & $0.09 \pm 0.02$ & $0.18 \pm 0.05$ \\
(10, 1)  & $\mathbf{0.07 \pm 0.03}$ & $0.67 \pm 0.41$ & $0.55 \pm 0.16$ \\
(10, 5)  & $0.07 \pm 0.02$ & $0.08 \pm 0.03$ & $\mathbf{0.04 \pm 0.01}$ \\
(10, 10) & $\mathbf{0.19 \pm 0.18}$ & $0.41 \pm 0.15$ & $0.30 \pm 0.18$ \\
\bottomrule
\end{tabular}
\end{table}

\paragraph{Same loss: the search.}
Under the \mmd{} loss the two methods differ only in how they search, and
\mlgd{}$^\star$ is never worse. In seven cells both reach the basin of $\xstar$ and
the TVs agree within one standard deviation, with the mean of \mlgd{}$^\star$
lower in six of them. The two cells where they differ are the hardest:
$(10,1)$, where \mlgd{}$^\star$ reaches $0.07$ against $0.67$, and $(10,10)$, with
$0.19$ against $0.41$.
These are the cells where the loss gives little direction: at $(10,1)$ it is flat
in nine of the ten directions of $x$ (Appendix~\ref{sec:app:sim:results}), so
gradient descent from a training design stays in the basin it starts in. Guided
reverse diffusion starts from noise and is steered toward the regions of high
prior density that fit the target, so it reaches the right basin.
\mlgd{}$^\star$ needs more machinery than direct optimization, a prior over $X$
and a reverse process, and this is where it pays off.

\paragraph{The KL loss.}
With its own KDE-KL loss GB-FESO is better at $(5,5)$ and $(10,5)$ ($0.04$
against $0.06$ and $0.07$), where that loss suits the target well; GB-FESO with
the \mmd{} is not better in the same cells, so the gain comes from the loss, not
from the search.
It is clearly worse whenever $\dimy = 10$ ($0.18$ against $0.05$ at $(1,10)$ and
$0.18$ against $0.07$ at $(5,10)$) and at $(10,1)$.
The failure at $\dimy = 10$ is the one expected of a kernel density estimate with
a fixed bandwidth: its accuracy degrades quickly with the dimension of $Y$, and the
bandwidth would have to be retuned for every output dimension.
The sample \mmd{} converges at a rate that does not depend on the dimension
\citep{gretton2008kernelmethodtwosampleproblem} and sets its bandwidths from the
data, which is why \mlgd{} uses it.
The loss is a plug-in choice in Algorithm~\ref{alg:inner-estimator}, so a KDE-KL
could replace it when $\dimy$ is small.
}

\subsection{Toy Example: Optimization Result and \cdms{} $\beta$-Sweep}
\label{sec:app:toy}

\paragraph{Setup.}
$\PP(X,Y)$ is a Gaussian mixture with four clusters in $x$, centred at
$x=-6,-2,2,6$, each carrying a quarter of the mass, split equally over its
components.
The cluster at $x=2$ has two components with $y$-means $\pm 2$ ($y$-variance
$0.4$); conditioning on $\xstar = 2$ gives the bimodal target $\target(Y) =
\Pcond{2}$.
The other three clusters are decoys that induce a different conditional: three
narrow modes at $y = 2, 0, -2$ (variance $0.15$) at $x=-6$, one mode at $y=2$
(variance $0.9$) at $x=-2$, and one broad mode at $y=0$ (variance $3$) at $x=6$.
Every component has $x$-variance $0.25$ and a non-zero $x$--$y$ correlation
($\rho = 0.5$, $-0.7$, $\pm 0.4$, $-0.5$ for the four clusters; the two target
components tilt in opposite directions), so that $\Pcond{x}$ moves with $x$
inside each cluster.
With diagonal covariances the loss would be exactly flat over the whole $\xstar$
cluster and $\xstar$ would be a set; with the tilt it is a unique minimizer and
the loss is a bowl.
The decoys are arranged so that closeness in $x$ is uninformative: the
population loss at the cluster centers is $0.12$ at $x=-6$, $0.21$ at $x=-2$ and $0.20$ at $x=6$.

\paragraph{Models.}
All methods share the same prior, an \emph{unconditional} \ddpm{}
on the marginal $\PP(X)$, and differ in the conditional
sampler $\fphi$ for $\Pcond{x}$: a \emph{conditional}
\ddpm{} for \mlgd{}, and a \emph{conditional consistency model}
(\ict{}) for \mlgdf{}. All three networks use the MLP architecture of
Section~\ref{sec:app:sim:arch} (3 blocks, 128 units) and are trained
for $3{,}000$ epochs with batch size $1{,}024$ on fresh draws from the mixture.
Checked against the analytic mixture, the
prior's score network has a relative Fisher error below $0.02$ 
at every noise level but the smallest ($t = 1$, where it is $1.56$), unguided samples match $\PP(X)$ (normalized \swd{} $0.02$),
and the conditional-diffusion sampler reproduces $\Pcond{\xstar}$ to
Wasserstein-1 distance $0.06$; the consistency sampler is much less accurate at
$\xstar$ (Wasserstein-1 $0.56$), which bounds how closely \mlgdf{}'s returned
conditional can match $\target$.

\paragraph{Optimization result.}
Each method runs ten trajectories from $x_T \sim \mathcal{N}(0,1)$ with the
\mmd{} loss ($\ncond = 250$), and returns the one with the smallest final loss.
In a first run (on a GPU), \mlgd{} reaches the $\xstar$ basin in one restart out of ten (exact $L^2 =
6\cdot 10^{-4}$); 
the other nine ($x \approx 5.8$--$6.6$) and all ten \mlgdf{} restarts ($x \approx 5.4$) end in the broad decoy, because under the sample loss that basin is wide
and only moderately worse than $\xstar$ (Figure~\ref{fig:sim_2d}(b)) while the $\xstar$ basin is
narrow and sits between two walls.
Which method's restarts enter that narrow basin depends on the run: in the run
of Figure~\ref{fig:sim_2d} (same settings and seeds, on a CPU) one \mlgdf{}
restart reaches it (exact $L^2 = 8\cdot 10^{-3}$) and all ten \mlgd{} restarts
end in the decoy at $x \approx 6$; in both runs about one restart in ten does, and
selection by the smallest loss returns it.

\paragraph{Effect of guidance scale $\beta$.}
\label{sec:app:cdms_beta}
Figure~\ref{fig:sim_2d}(e, f) illustrates the \cdms{} formulation: the
transition from prior sampling to target optimization as $\beta$ increases.
The unconditional \ddpm{} above is guided by the additive gradient step of
Equation~\ref{eq:guidance-step} with step size $\beta\,\beta_t$ (the \dps{}
convention), $\eta = 1$, over its full schedule, with fifty independent
populations of $2{,}000$ trajectories ($10^5$ samples) per $\beta$.
The guidance signal is the
analytical $L^2$-GMM loss, so that the comparison
against the tempered posterior
$\Qb(x)\propto\PP(x)\exp(-\beta\,\loss(x))$ isolates the sampler's own bias.
At $\beta=0$ the sampler recovers $\PP(X)$; increasing $\beta$
concentrates mass toward $\xstar$, but the samples do not follow $\Qb$: at
$\beta = 8$ and $16$ they almost leave the $x=-2$ cluster, where $\Qb$ keeps
$7\%$ and $2\%$ of its mass, and over-populate the broad cluster at $x=6$.

\paragraph{Correcting the bias with twisted SMC.}
The bias comes from the sampler and can be removed by treating guided reverse
diffusion as a proposal inside a twisted sequential Monte Carlo (SMC) sampler
\citep{wu2023twisted}.
The guided transition $\tilde M_t(x_{t-1} \mid x_t) = \mathcal{N}(\mu_t(x_t) -
\beta\beta_t \nabla_{x_t} \loss(\xzero(x_t)),\, \sigma_t^2 I)$ is a shifted copy
of the unguided \ddpm{} transition $M_t(x_{t-1} \mid x_t)$; with twisting
potentials $G_t(x_t) = \exp(-\beta \loss(\xzero(x_t)))$ each trajectory carries
the weight
$\log w \mathrel{+}= \log M_t(x_{t-1} \mid x_t) - \log \tilde M_t(x_{t-1} \mid x_t) + \log
G_{t-1}(x_{t-1}) - \log G_t(x_t)$,
each trajectory starts with weight $G_T(x_T)$ at its first guided step,
the population is resampled at every step, and the potentials telescope to $G_0(x_0)$ so that
the weighted terminal population targets $\PP(x_0)\, G_0(x_0) = \Qb(x_0)$
exactly (up to the trained score).
The proposal is free, and the guided shift $\beta\beta_t \nabla \loss$ of the
plain sampler is a poor one where it exceeds $\sigma_t$ by far (the
ratio $M_t/\tilde M_t$ is then log-normal with a variance of $(\text{shift}/\sigma_t)^2$,
and the particles in the far clusters are all but discarded); the corrected arm
therefore bounds the shift to $\sigma_t/4$ and lets the weights and the
resampling do the rest.
With the correction (Figure~\ref{fig:sim_2d}(e, f)), the mass that the plain
sampler leaves at $x = \pm 6$ moves back onto $\xstar$ and onto the $x=-2$
cluster, and the total variation distance to $\Qb$ drops from $0.17$ to $0.025$
at $\beta = 2$, from $0.19$ to $0.026$ at $\beta = 8$ and from $0.57$ to $0.07$ at
$\beta = 16$ ($0.025$ at $\beta = 0$ is the noise floor of the TV estimate).
The correction costs a resampled population; the optimization
experiments need only the minimizer and use plain guidance.

\subsection{Infeasible Targets and Distributional vs.\ Pointwise
Guidance}
\label{sec:app:nostar}

\paragraph{Motivation.}
This experiment establishes two properties.
\textbf{(P1)~Infeasible targets are allowed.} Problem~\ref{prob:opt} and
Algorithms~\ref{alg:mlgd-outer}--\ref{alg:inner-estimator} apply for any value
of
$\min_x\loss(x) \ge 0$. Guidance uses only the loss landscape: at each step it
forms $\lossHat(\xzero)$ and its gradient and moves toward lower
loss, with no feasibility test. This holds for \mlgd{} and \mlgdf{} alike, as
they
differ only in the inner sampler $\fphi$.
\textbf{(P2)~Pointwise losses can be problematic; a distributional loss can fix
this.}
Even for a point target, a pointwise loss can select a design that never places
mass on it. The synthetic settings of Section~\ref{sec:app:sim} have
$\min_x\loss=0$ by construction, so here we choose a target that no $x$
realizes.

\paragraph{Setup.}
$\PP(X,Y)$ is a 3-component Gaussian mixture with means $(-2,2)$, $(-2,4)$,
$(2,3)$, weights $(\tfrac14,\tfrac14,\tfrac12)$ and covariance
$\mathrm{diag}(0.7^2,0.2^2)$. At $x_{\mathrm{bi}}=-2$ the conditional is bimodal
on
$\{2,4\}$: it hits $y{=}4$ exactly, but only half the time. At
$x_{\mathrm{uni}}=2$ it is $\mathcal{N}(3,0.2^2)$: always close, never exactly
on
$4$. The target is the point mass $\target=\delta(y-4)$, realized as
$\mathcal{N}(4,10^{-6})$. No $x$ matches it, so $\min_x\loss>0$; the
closed-form landscape (Appendix~\ref{sec:app:sim:l2}) has a plateau at
$L^\star{=}280.805$ around $x_{\mathrm{bi}}$ and a higher one at $283.505$
around
$x_{\mathrm{uni}}$.

\paragraph{Methods.}
We use \mlgdf{} and compare two guidance losses with everything
else fixed: \mmd{} (Eq.~\ref{eq:mmd-rbf}) and a pointwise loss
$\mathbb{E}[(y-y^*)^2]$ estimated by the same plug-in mechanism, at
$\ncond{=}250$ and $\ncond{=}1$. Each configuration is run
$R{=}25$ times; we report the fraction of runs landing closer to
$x_{\mathrm{bi}}$ than to $x_{\mathrm{uni}}$.

\paragraph{Results.}
Table~\ref{tab:nonstar_results} and Figure~\ref{fig:nonstar} show the outcome.
\textbf{(P1)}~Although $\min_x\loss>0$, \mmd{} guidance lands at $x_{\mathrm{bi}}$, the true minimizer, in all $25$ restarts ($100\%$): the update depends
only on the loss landscape.
\textbf{(P2)}~The natural pointwise baseline is $\ncond{=}1$: when the
target is a point $y^*$, one treats the output as a point, and a single sample
per step is the guidance-time analog of standard point-target inverse design.
It is noisy ($4\%$ at $x_{\mathrm{bi}}$). The $\ncond{=}250$ pointwise run
removes this noise: it
uses the same sample budget as \mmd{} and optimizes a different objective.
Expected
squared error compares the samples to $y^*$ one by one and never compares the
distribution as a whole, so its minimizer is $x_{\mathrm{uni}}$, which is
reliably
close to $y^*$, rather than $x_{\mathrm{bi}}$, which is exactly right half the
time. Accordingly, it converges to $x_{\mathrm{uni}}$ in every restart ($0\%$ at
$x_{\mathrm{bi}}$), and more samples make it converge more reliably to the wrong
design. Everything but the loss is held fixed, so this is a property of the
pointwise objective. It is the
point-versus-distribution asymmetry of Section~\ref{sec:introduction}: in this example the distributional loss
prefers the design that puts mass exactly on the target, whereas the
pointwise loss prefers one concentrated nearby.

\begin{table}[h]
\centering
\caption{No-star target ($R{=}25$ restarts each). \% at $x_{\mathrm{bi}}$:
fraction
of restarts closer to $x_{\mathrm{bi}}$ (true optimum, $L^\star{=}280.805$) than
to
$x_{\mathrm{uni}}$ ($283.505$). Time is per run.}
\label{tab:nonstar_results}
\begin{tabular}{lccc}
\toprule
\textbf{Loss} & $\ncond$ & \textbf{\% at $x_{\mathrm{bi}}$} & \textbf{Time (s)} \\
\midrule
\mmd{}    & 250 & $\mathbf{100\%}$ & $1.11 \pm 0.09$ \\
Pointwise & 250 & $0\%$            & $1.03 \pm 0.09$ \\
Pointwise & 1   & $4\%$            & $0.99 \pm 0.14$ \\
\bottomrule
\end{tabular}
\end{table}

\begin{figure}[h]
  \centering
  \includegraphics[width=\textwidth]{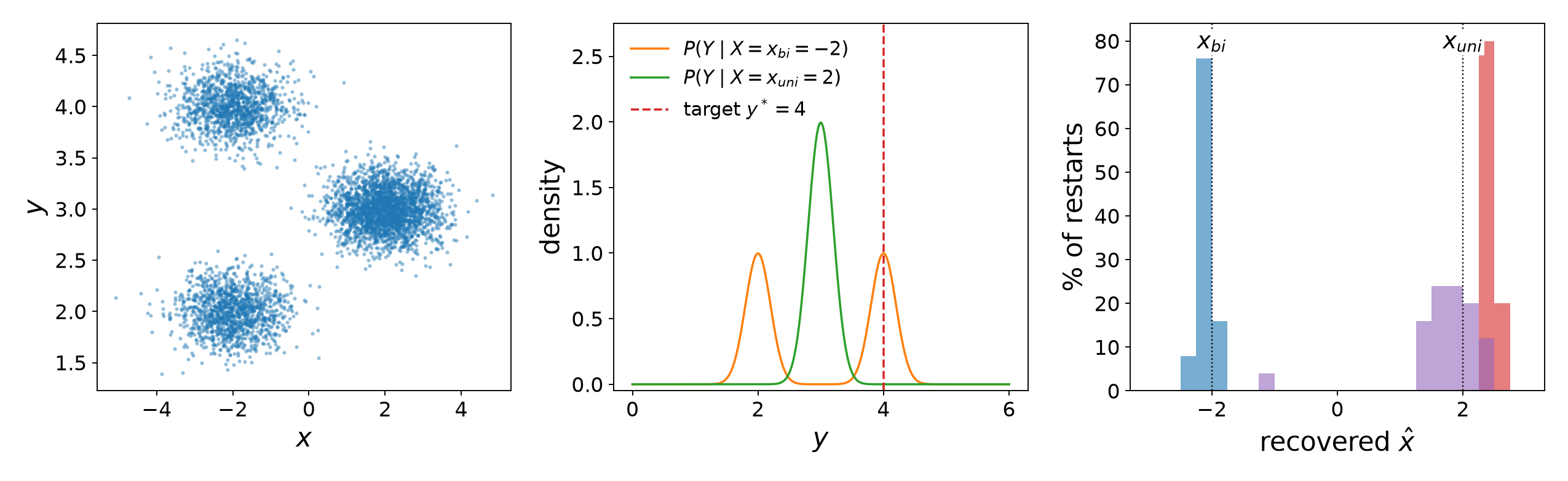}
\caption{\textbf{Infeasible Target} (Appendix~\ref{sec:app:nostar}).
\textbf{Left:} joint $\PP(x,y)$. \textbf{Middle:} the two candidate
conditionals $\Pcond{x_{\mathrm{bi}}}$ (bimodal) and
$\Pcond{x_{\mathrm{uni}}}$ (unimodal) against the point
target $y^*{=}4$ (dashed): neither matches it exactly.
\textbf{Right:} distribution of the recovered $\hat{x}$ over 25 restarts per
arm. \mmd{} guidance ($\ncond{=}250$, blue) lands at $x_{\mathrm{bi}}$, the
better design, in every restart; pointwise guidance lands at or just beyond
$x_{\mathrm{uni}}$ with $\ncond{=}250$ (red) in every restart, and scatters,
mostly toward $x_{\mathrm{uni}}$, with $\ncond{=}1$ (purple).}
\label{fig:nonstar}
\end{figure}

\FloatBarrier
\section{MNIST: Experimental Details}
\label{sec:app:mnist}
\label{app:mnist}

Full architectural and training details are given in the subsections
below. 

\subsection{Wasserstein and Sliced Wasserstein Distances: Population and Sample
Forms}
\label{sec:app:swd}

\subsubsection{Wasserstein Distance (WD)}
The Wasserstein
distance~\citep{bischoff2024practicalguidesamplebasedstatistical} measures the
minimum cost of transforming one probability distribution into another. For two
probability distributions $P$ and $Q$ on $\mathbb{R}^d$, the 1-Wasserstein
distance $W_1$ is defined as:
\begin{equation}
W_1(P, Q) = \inf_{\gamma \in \Gamma(P, Q)} \mathbb{E}_{(x,y) \sim \gamma}[\|x -
y\|_1]
\end{equation}
where $\Gamma(P, Q)$ denotes the set of all joint distributions with marginals
$P$ and $Q$.

We include the 1D Wasserstein distance formulation here because the Sliced
Wasserstein Distance (\swd{}), described in the next section, reduces the
high-dimensional problem to multiple 1D Wasserstein distance computations
through random projections.

For one-dimensional distributions $P$ and $Q$ on $\mathbb{R}$, the 1-Wasserstein
distance has a closed-form expression in terms of the cumulative distribution
functions (CDFs):
\begin{equation}
W_1(P, Q) = \int_{-\infty}^{\infty} |F_P(x) - F_Q(x)| \, dx
\end{equation}
where $F_P$ and $F_Q$ are the CDFs of $P$ and $Q$, respectively.

Given samples $S_1 = \{s_1, \ldots, s_m\}$ from $P$ and $S_2 = \{t_1, \ldots,
t_n\}$ from $Q$ (assuming $m = n$ for simplicity), let $s_{(1)} \leq s_{(2)}
\leq \cdots \leq s_{(m)}$ and $t_{(1)} \leq t_{(2)} \leq \cdots \leq t_{(n)}$
denote the sorted samples. The 1-Wasserstein distance can be estimated as:
\begin{equation}
W_1(S_1, S_2) = \frac{1}{m} \sum_{i=1}^{m} |s_{(i)} - t_{(i)}|.
\end{equation}
\subsubsection{Sliced Wasserstein Distance (\swd{})}
\label{app:mnist:swd}
We use the Sliced Wasserstein Distance
(\swd{})~\citep{bischoff2024practicalguidesamplebasedstatistical}, a
computationally efficient lower bound on the Wasserstein distance.

The \swd{} between distributions $P$ and $Q$ is defined as:
\begin{equation}
\text{SWD}(P, Q) = \int_{\mathbb{S}^{d-1}} W_1(\pi_\theta \# P, \pi_\theta \# Q)
\, d\theta
\end{equation}
where $\mathbb{S}^{d-1}$ is the unit sphere in $\mathbb{R}^d$, $\pi_\theta$
denotes projection onto direction $\theta$, $\pi_\theta \# P$ and $\pi_\theta \#
Q$ are the distributions of the projected points along that direction, and $W_1$
is the 1-Wasserstein distance.

We approximate this integral by averaging the Wasserstein distance over $L$
random one-dimensional projections:
\begin{equation}
\text{SWD}(P, Q) \approx \frac{1}{L} \sum_{\ell=1}^{L} W_1(\pi_{\theta_\ell} \#
P, \pi_{\theta_\ell} \# Q)
\end{equation}
where $\theta_1, \ldots, \theta_L$ are uniformly sampled from
$\mathbb{S}^{d-1}$. Notably, for one-dimensional data, \swd{} is equivalent to
the 1-Wasserstein distance since the projection is the identity mapping,
allowing us to use the same implementation for both metrics.

In practice, when working with finite samples $\{x_i\}_{i=1}^n$ and
$\{y_j\}_{j=1}^m$, we compute the 1-Wasserstein distance between the empirical
1D projections by sorting the projected points. This sample version is the
$\distHat$ used in the MNIST experiment, where the samples are the circular
embeddings $(\cos\theta, \sin\theta) \in \mathbb{R}^2$ of the rotation angles
and $L = 50$ projections are drawn for each loss evaluation.

\subsection{Dataset Construction}
\label{app:mnist:data}

We construct an augmented MNIST dataset that captures rotational
ambiguity across digit classes.
\subsubsection{Conditional Model }
For the conditional training the pipeline proceeds in two stages.

\paragraph{Stage 1: Logical-angle entries.}
We start from the standard MNIST training split (60,000 images).
For each base image, we always retain a copy at logical angle $0^\circ$.
We additionally include rotated copies at $90^\circ$, $180^\circ$, and
$270^\circ$
if a pretrained classifier assigns the rotated image to a digit class
with confidence $\geq 0.9999$. This threshold ensures that only
genuinely rotationally ambiguous digits (e.g., ``0'', ``1'', ``8'',
``6''/``9'') contribute additional logical-angle entries. The
classifier used is the dataset-construction CNN described below,
trained on the standard MNIST training split.

\paragraph{Stage 2: Uniform random rotation.}
Each logical-angle entry is then rotated by a uniform random angle
$\theta \sim \mathrm{Uniform}[0^\circ, 360^\circ)$. The stored label angle is
$({\rm logical\_angle} + \theta) \bmod 360^\circ$. This produces a dataset
of 82,483 samples with continuous, uniformly distributed rotation
angles within each rotationally ambiguous class.

\subsubsection{Unconditional Model }
The unconditional diffusion model is trained on a separate dataset
consisting of all 70,000 MNIST images (train + test splits), each
independently rotated by a uniform random angle at every epoch
access, with no classifier filtering.

\subsection{Model Architectures}
\label{app:mnist:arch}

\paragraph{Dataset-construction classifier.}
This classifier is used only during dataset creation to determine
which rotated copies of each digit are rotationally ambiguous.
It is a three-block CNN trained on the standard MNIST training split
for 15 epochs with batch size 128, Adam optimizer
(lr $= 10^{-3}$, weight decay $= 10^{-4}$), and learning-rate
reduction on plateau (factor $0.5$, patience $2$).
Training augmentation consists of random rotation ($\pm10^\circ$),
random affine transforms (translation up to 10\%, scale $0.9$--$1.1$),
and standard MNIST normalization ($\mu=0.1307$, $\sigma=0.3081$). Threshold for
classification is 0.9999.

The architecture is:
\begin{itemize}
    \item \textbf{Block 1:} Conv(1→32, 3×3) → BN → ReLU →
          Conv(32→32, 3×3) → BN → ReLU → MaxPool(2×2) →
          Dropout2D(0.25)
    \item \textbf{Block 2:} Conv(32→64, 3×3) → BN → ReLU →
          Conv(64→64, 3×3) → BN → ReLU → MaxPool(2×2) →
          Dropout2D(0.25)
    \item \textbf{Block 3:} Conv(64→128, 3×3) → BN → ReLU →
          MaxPool(2×2) → Dropout2D(0.25)
    \item \textbf{FC:} Linear(1152→256) → BN → ReLU → Dropout(0.5)
          → Linear(256→128) → BN → ReLU → Dropout(0.5)
          → Linear(128→10)
\end{itemize}

\paragraph{Digit classifier (final results table).}
Since the recovered images $\xhat$ may be arbitrarily rotated and are
synthetically generated (and thus potentially lower in contrast than
real MNIST digits), a classifier trained solely on standard upright
MNIST would be unsuitable for evaluation.
We therefore train a more robust variant designed to handle arbitrary
orientations and mild image degradation. It shares the same three-block CNN
backbone as the dataset-construction
classifier, with three modifications for noise robustness:
(i)~an \textbf{input} Dropout2D(0.1) applied before the first
convolution;
(ii)~a \textbf{mid-block} Dropout2D(0.1) inserted between the two
convolutions in each of Blocks~1 and~2;
and (iii)~\textbf{label smoothing} ($\epsilon=0.1$) in the
cross-entropy loss during training.
Training uses Adam (lr~$=10^{-3}$, weight decay~$=10^{-4}$) for 10 epochs with batch size 128. Training augmentation consists of
random rotation ($\pm20^\circ$), random affine transforms (translation
up to 10\%, scale $0.9$--$1.1$), Gaussian noise ($\sigma=0.15$),
and standard MNIST normalization ($\mu=0.1307$, $\sigma=0.3081$).
Threshold for classification is 0.8. 

\paragraph{Unconditional diffusion model.}
We train an unconditional \ddpm{} UNet on randomly rotated MNIST images.
The architecture is a standard UNet2D with the following
configuration:

\begin{table}[H]
\centering
\caption{Unconditional UNet architecture.}
\begin{tabular}{lc}
\toprule
\textbf{Parameter} & \textbf{Value} \\
\midrule
Input/output channels   & 1 \\
Sample size             & $28 \times 28$ \\
Layers per block        & 2 \\
Channel widths          & (32, 64, 128) \\
Down blocks & DownBlock2D, AttnDownBlock2D, AttnDownBlock2D \\
Up blocks   & AttnUpBlock2D, AttnUpBlock2D, UpBlock2D \\
Dropout                 & 0.1 \\
\bottomrule
\end{tabular}
\end{table}

\paragraph{Conditional consistency model.}
The conditional model is an improved Consistency Training (\ict{}) model
that predicts the rotation angle of an MNIST image as a circular
$(\cos\theta, \sin\theta)$ vector, conditioned on the image.
The image conditioning encoder is a two-block CNN that maps
$28\times28$ images to a $128$-dimensional embedding, which is added
to the noisy angle representation at each forward pass.
The full architecture is:
\begin{itemize}
    \item \textbf{Condition encoder:} Conv(1→32, 3×3) →
          GroupNorm(8) → SiLU → Dropout2D(0.1) →
          Conv(32→32, 3×3) → GroupNorm(8) → SiLU →
          MaxPool(2×2) → Conv(32→64, 3×3) → GroupNorm(8) → SiLU →
          Dropout2D(0.15) → Conv(64→64, 3×3) → GroupNorm(8) →
          SiLU → MaxPool(2×2) → Dropout2D(0.2) → Flatten →
          Dropout(0.3) → Linear(3136→128) → SiLU
    \item \textbf{Input projection:} Linear(2→128) → LayerNorm
    \item \textbf{Time embedding:} Sinusoidal embedding (dim=128)
    \item \textbf{MLP backbone:} 5 hidden layers of width 128 with
          ReLU activations
    \item \textbf{Output:} LayerNorm → Linear(128→2), followed by
          \ict{} skip connection and $\ell_2$ normalization
\end{itemize}
The output is normalized to the unit circle and decoded via
$\hat{\theta} = \mathrm{atan2}(\sin, \cos) \bmod 360^\circ$.

\subsection{Training Details}
\begin{table}[h]
\centering
\caption{Training hyperparameters for the MNIST models.}
\label{tab:mnist_hparams}
\begin{tabular}{lcc}
\toprule
\textbf{Hyperparameter} & \textbf{Unconditional} & \textbf{Conditional (\ict{})} \\
\midrule
Optimizer               & AdamW          & AdamW \\
Learning rate           & $10^{-3}$      & $10^{-4}$ \\
Weight decay            & $10^{-2}$      & $10^{-4}$ \\
LR schedule             & None           & Cosine annealing \\
Batch size              & 256            & 256 \\
Epochs                  & 100            & 500 \\
Noise schedule & squaredcos\_cap\_v2 & \ict{} curriculum ($s_0{=}10$, $s_1{=}1280$) \\
Training timesteps      & 1000           & -- \\
$\epsilon$ (\ict{})        & --             & $0.002$ \\
Loss                    & MSE (noise)    & Smooth Huber \\
\bottomrule
\end{tabular}
\end{table}
    
The \ict{} model uses a discretization curriculum
$N(k) = \min(s_0 \cdot 2^{\lfloor k/K' \rfloor}, s_1) + 1$
that progressively increases the number of consistency training
boundaries during training. Noise levels are sampled according to a
log-normal proposal distribution with $P_{\rm mean} = -1.1$ and
$P_{\rm std} = 2.0$. During training of the conditional model,
light augmentation is applied to the conditioning images:
Gaussian noise ($\sigma = 0.05$) and random pixel dropout (probability $0.1$).

\subsection{\mlgdf{} Inference Hyperparameters}
We report the hyperparameters used for all MNIST rotation experiments 
in Table~\ref{tab:mnist_lgd_hparams}. The hyperparameters were chosen by hand from a few trial runs per target, keeping the setting under which the recovered images were classified as digits by the rotation-robust classifier (Appendix~\ref{app:mnist:arch})  and had a low \swd{} to the target.
The images are not clamped during the optimization; for display they are clamped to $[-1, 1]$ and mapped to $[0, 1]$, and the classifier rescales each image to $[0, 1]$.
\begin{table}[h]
\centering
\caption{\mlgdf{} hyperparameters for the MNIST rotation experiments.
Guidance step size is $\zetat = r_t/(1+r_t^2) + 5t/1000$;
the unimodal and uniform experiments use $2\cdot\zetat$.}
\label{tab:mnist_lgd_hparams}
\begin{tabular}{lccc}
\toprule
\textbf{Hyperparameter} & \textbf{Unimodal} & \textbf{Bimodal} & \textbf{Uniform} \\
\midrule
Distributional loss                & \multicolumn{3}{c}{Sliced Wasserstein Distance (\swd{})} \\
Number of projections (\swd{})        & \multicolumn{3}{c}{50} \\
Independent runs                   & \multicolumn{3}{c}{15} \\
\ict{} sampling steps                 & \multicolumn{3}{c}{$[150, 50, 20, 10, 5, 1]$} \\
\midrule
Conditional samples $n_{\rm cond}$ & 1500 & 1500 & 600  \\
\ddim{} inference steps $T$           & 130  & 125  & 300  \\
MC perturbations $n_{\rm MC}$      & 3    & 3    & 3    \\
Guidance step size                 & $2\cdot\zetat$ & $\zetat$ & $2\cdot\zetat$ \\
Guidance applied for               & all $t$ & all $t$ & $t \ge 250$ \\
Target variance $\sigma^2$         & 515  & 300  & --   \\
\bottomrule
\end{tabular}
\end{table}
\subsection{Full Results Across All Seeds}
\label{app:mnist:fullresults}

Figures~\ref{fig:app_bimodal_full}--\ref{fig:app_unimodal_full} show
all 15 seeds for each target $\target$, ordered by seed index.
The three seeds with the lowest final loss per target are shown in
Figure~\ref{fig:mnist} of the main text.

\begin{figure}[h]
    \centering
    \includegraphics[width=\textwidth]{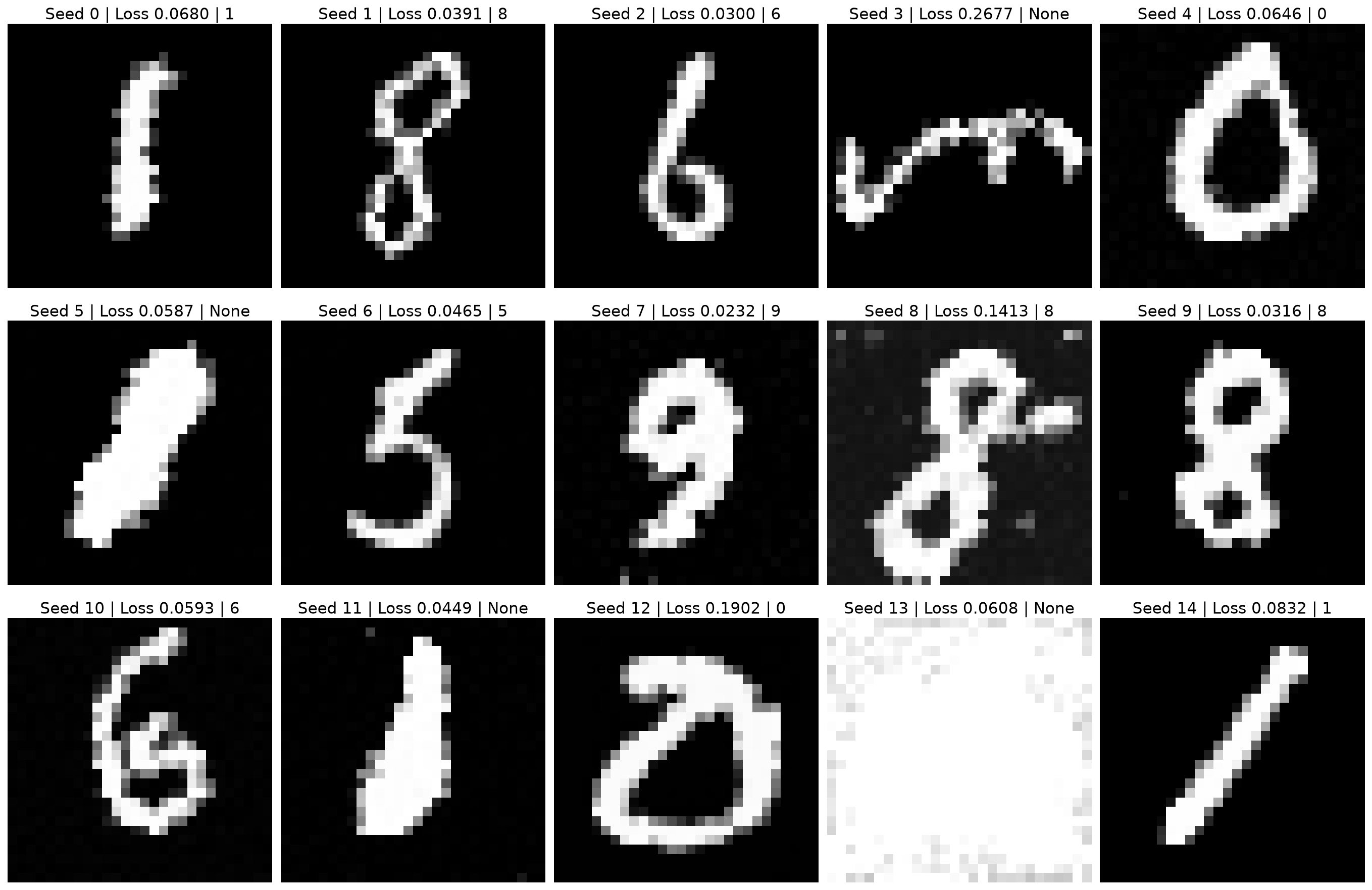}
    \caption{
        \textbf{Bimodal target $\target_{\mathrm{bimodal}}$, all 15 seeds.}
        Each panel shows the optimized digit $\xhat$ with its seed
        index, final \swd{} loss, and classifier label (``None''
        indicates no digit exceeded the confidence threshold).
        Rows show seeds 0--4, 5--9, and 10--14, from top to bottom.
        The optimizer consistently recovers rotationally symmetric
        digits such as ``1'', ``8'', ``0'', ``6'', and ``9'', which remain
        visually valid under $180^\circ$ rotation, consistent with the
        bimodal target placing equal mass at $0^\circ$ and $180^\circ$.
    }
    \label{fig:app_bimodal_full}
\end{figure}

\begin{figure}[h]
    \centering
    \includegraphics[width=0.8\textwidth]{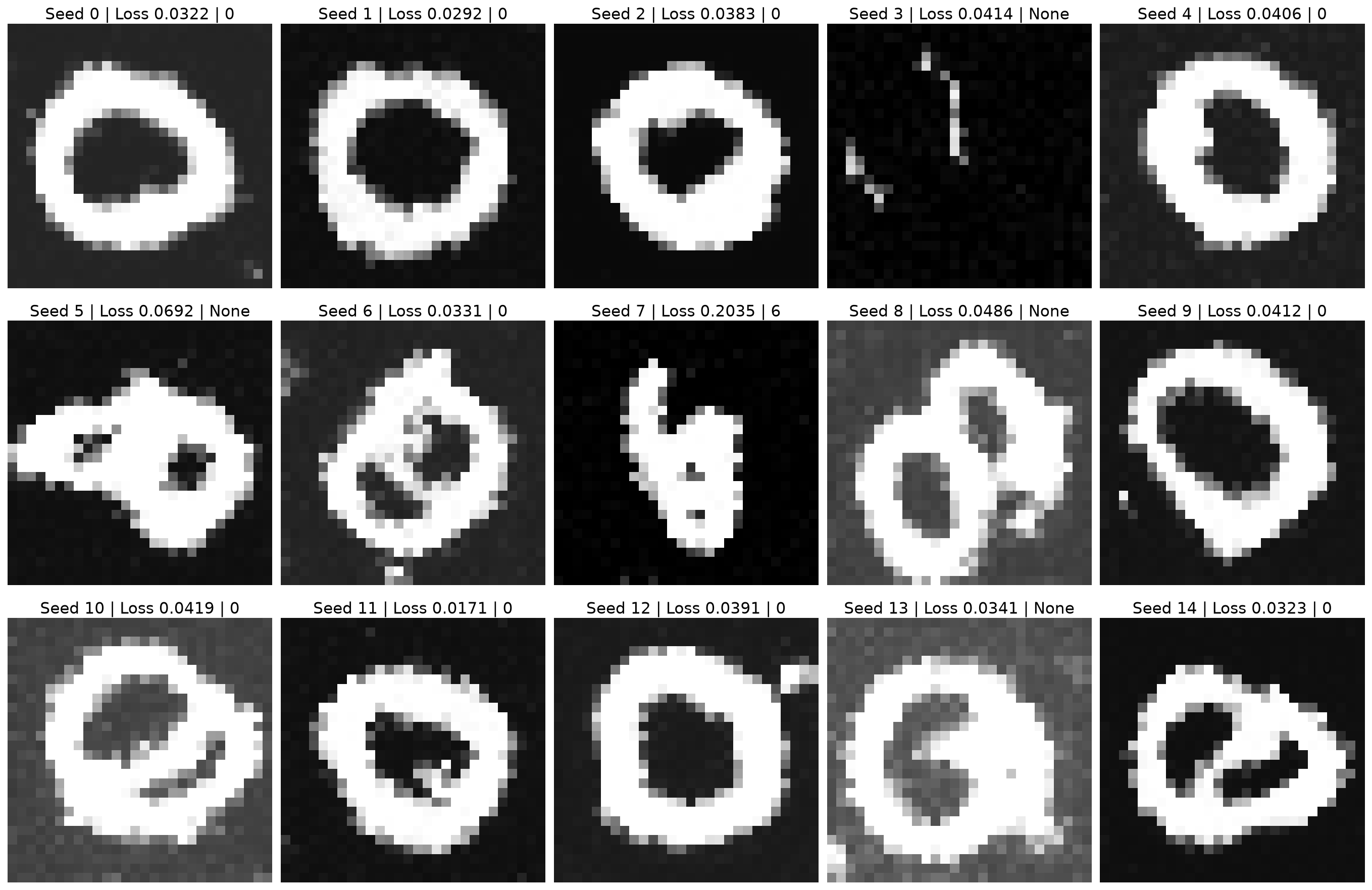}
    \caption{
        \textbf{Uniform target $\target_{\mathrm{uniform}}$, all 15 seeds.}
        Panels are annotated as in Figure~\ref{fig:app_bimodal_full};
        rows show seeds 0--4, 5--9, and 10--14, from top to bottom.
        The optimizer overwhelmingly recovers circular ``0'' digits
        across seeds, the only digit class that is plausibly valid
        at every rotation angle $\theta \in [0^\circ, 360^\circ)$.
    }
    \label{fig:app_uniform_full}
\end{figure}

\begin{figure}[h]
    \centering
    \includegraphics[width=0.8\textwidth]{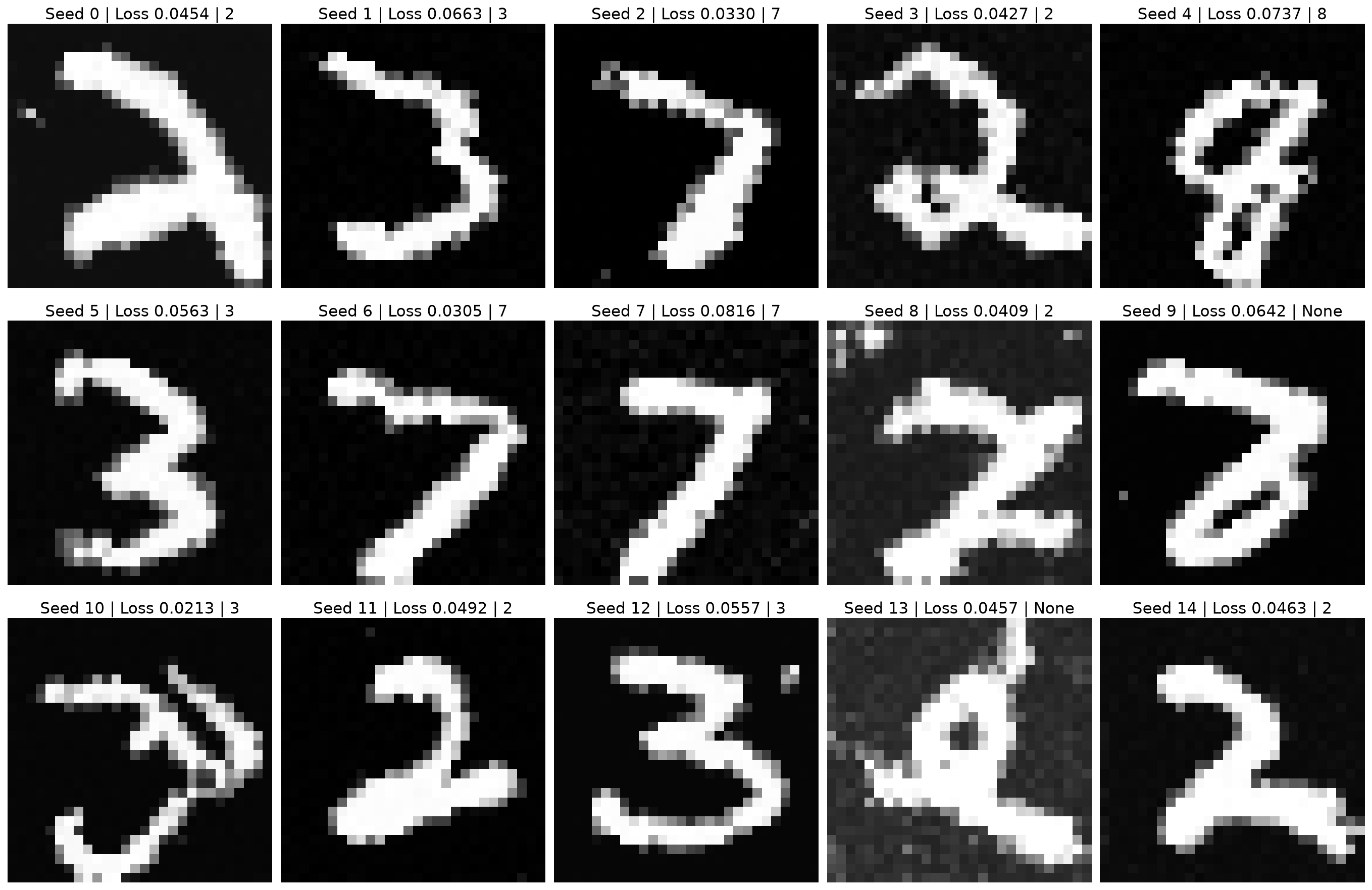}
    \caption{
        \textbf{Unimodal target $\target_{\mathrm{unimodal}}$, all 15 seeds.}
        Panels are annotated as in Figure~\ref{fig:app_bimodal_full};
        rows show seeds 0--4, 5--9, and 10--14, from top to bottom.
        The optimizer recovers canonical upright digits such as
        ``2'', ``3'', and ``7'', consistent with the unimodal target
        placing mass near $0^\circ$ (upright orientation).
    }
    \label{fig:app_unimodal_full}
\end{figure}

\subsection{Compute}
   \label{app:mnist:compute}

All MNIST generative models were trained on an NVIDIA T4 GPU using
PyTorch with CUDA. Training the unconditional UNet takes
approximately 178 minutes (100 epochs); training the conditional
\ict{} model takes approximately 162 minutes (500 epochs).

\mlgdf{} optimization was run on an NVIDIA L4 GPU. Each seed takes
about 8--18 seconds depending on the target $\target$, and the
15 seeds of a target complete in under 5 minutes.
\FloatBarrier

\section{Stable Diffusion: Experimental Details and Analysis}
\label{app:sd}

\paragraph{On initialization.}\label{app:sd:init}
 We treat the SD task as editing,
starting from an existing male portrait scribble rather than
noise, as is common at this scale in peptide,
antibody and small-molecule design
\citep{vazqueztorres2024denovo,schneuing2024structurebaseddrugdesignequivariant,luo2022antigenspecific}.
Optimizing from noise, without a source scribble, is left to future work.

\subsection{Loss Function}
\label{app:sd:loss}
The Stable Diffusion experiment uses the $U$-statistic of
Appendix~\ref{app:mmd-def} on \clip{} ViT-L/14 embeddings ($\mathbb{R}^{768}$),
with a single RBF kernel whose bandwidth is set from the median pairwise
distance. Gradients flow only through the generated samples. The optimized loss
and the reported \mmd{} (Table~\ref{tab:sd_results}) are both
$\sqrt{|\widehat{\mathrm{MMD}}^2_U| + \epsilon}$ with $\epsilon = 10^{-8}$, which
keeps the gradient of the square root finite near zero.
\subsection{Target Distribution Construction}
\label{app:sd:targets}
\label{app:sd:target}

All target images are generated using \sdxlturbo{} + \controlnetscribble{}
with the source male portrait scribble as conditioning and
\texttt{controlnet\_conditioning\_scale = 0.5} with two inference generation
steps.
\clip{} ViT-L/14 embeddings are computed from the generated images to form
the empirical target distribution in $\mathbb{R}^{768}$.
For each scenario, $N_{\text{target}} = 2{,}000$ images are drawn at
evaluation time.

\paragraph{Balanced and skewed targets.}
Images are sampled from two gender-specific prompts at the target
proportions ($50\%/50\%$ or $25\%/75\%$):
\begin{itemize}
  \item \textit{Male}:\texttt{"a superrealistic portrait photograph of a man,
  studio lighting"}
  \item \textit{Female}:\texttt{"a superrealistic portrait photograph of a
  woman, studio lighting"}
\end{itemize}

\paragraph{Gender interpolation target.}
The target discretizes a one-dimensional feminine-to-masculine axis
in \clip{} space using four anchor prompts at equal weight ($25\%$ each),
spanning the full spectrum:
\begin{enumerate}
  \item \textit{Woman}: \texttt{"superrealistic portrait photograph
        of a woman, extremely feminine features, studio lighting"}
  \item \textit{Woman with masculine features}: \texttt{"a superrealistic
        portrait photograph of a woman with masculine features,
        heavy brow ridge, studio lighting"}
  \item \textit{Man with feminine features}: \texttt{"a superrealistic
        portrait photograph of a man with extremely feminine features,
        soft delicate face, high cheekbones, studio lighting"}
  \item \textit{Man}: \texttt{"a superrealistic portrait photograph
        of a man, extremely masculine features, studio lighting"}
\end{enumerate}
This four-anchor discretization approximates a continuous 1-D target
on the gender-axis submanifold of \clip{} space.

\paragraph{Age interpolation target.}
The target is a uniform distribution over male portrait ages
$\{40, 41, \ldots, 79\}$. For each integer age, images are generated
with the prompt:
\begin{quote}
\texttt{"a superrealistic portrait photograph of a \{age\}-year-old
man, studio lighting, sharp focus, photographic"}
\end{quote}
\subsection{Optimization Hyperparameters}
\label{app:sd:hparams}

Table~\ref{tab:sd_hparams} summarizes all hyperparameters.
\mlgdf{} is initialized via \sdedit{}~\citep{sdedit}: the source
male portrait scribble is VAE-encoded, noised to the starting
timestep $t_{\text{start}}$, and denoised with \mlgdf{} guidance
from that point onward. All subsequent steps of
Algorithm~\ref{alg:mlgd-outer} are unchanged. The \ddim{} scheduler
runs for $T$ total steps; \mlgdf{} executes $T - t_{\text{start}}$
guided denoising steps per run. At each step, $\ncond$
conditional samples are drawn via \sdxlturbo{} + \controlnetscribble{}
to estimate the distributional loss gradient, and $n_{\text{target}}$
images are used to represent the target distribution $\target$. For each target we ran \mlgdf{} with seeds $1$--$10$ and selected the
run with the lowest \mmd{} to the target; for the two discrete targets (balanced and skewed) the selection also took into account the closeness of the male fraction to the target fraction. The seed row of Table~\ref{tab:sd_hparams} gives the selected seed.

\begin{table}[H]
\centering
\caption{\mlgdf{} hyperparameters for the Stable Diffusion experiments.}
\label{tab:sd_hparams}
\vspace{2mm}

\setlength{\tabcolsep}{3.5pt} 
\begin{tabular}{@{}lcccc@{}} %
\toprule
\textbf{Hyperparameter} & \textbf{Balanced} & \textbf{Skewed}
    & \textbf{Gender Interp.} & \textbf{Age Interp.} \\
\midrule
$T$ (total \ddim{} steps)
    & 250 & 250 & 250 & 250 \\
$t_{\text{start}}$ (\sdedit{} init step)
    & 125 & 125 & 125 & 125 \\
$\ncond$
    & 100 & 100 & 100 & 120 \\
$k$ (backpropagated variations)
    & 50 & 50 & 50 & 60 \\
$n_{\text{target}}$
    & 100 & 100 & 100 & 120 ($3\times40$) \\
Base strength $\zeta$
    & $7.5$ & $7.5$ & $7.5$ & $7.5$ \\
Target classes
    & 2 (M/F) & 2 (M/F) & 4 & 40 ages \\
Target proportions (M/F)
    & 50/50\% & 25/75\% & 25\% each & uniform \\
Source portrait age
    & -- & -- & -- & 40 \\
Seed
    & 4 & 6 & 8 & 2 \\
\bottomrule
\end{tabular}
\end{table}

\subsection{Baseline Details}
\label{app:sd:baselines}

\paragraph{Source scribble.}
The unmodified male portrait scribble extracted from one of the
generated male portraits via \texttt{lllyasviel/Annotators}
HED detector in scribble mode. This serves as the lower-bound
reference for improvement.

\paragraph{Average prompt.}
This baseline uses the target's prompts directly, without any search over
scribbles. For each component of the target, with prompt $c_g$ and weight
$\pi_g$, we encode $c_g$ with the two \sdxl{} text encoders and form the
weighted averages of the token-wise embeddings and of the pooled embeddings,
\[
  \bar{\mathbf{e}} = \sum_g \pi_g\, \mathbf{e}(c_g), \qquad \bar{\mathbf{p}} =
  \sum_g \pi_g\, \mathbf{p}(c_g).
\]
For the discrete and gender-interpolation targets, $c_g$ and $\pi_g$ are the
prompts and weights of Appendix~\ref{app:sd:targets}; for the age target, the
components are the $40$ ages from $40$ to $79$ with equal weights. We generate a
single portrait from the source scribble conditioned on $(\bar{\mathbf{e}},
\bar{\mathbf{p}})$, with \controlnetscribble{} and \sdxlturbo{} ($\Ks = 2$
steps, conditioning scale $0.5$, fixed seed), and take its HED edge map as the
baseline scribble.
The construction is inspired by~\citet{yang2024impusimagemorphingperceptuallyuniform}, who linearly interpolate \clip{} text embeddings, relying on the local linearity of the \clip{} text-embedding space.  We use a weighted average of the full embeddings as a
simple version of this idea, expecting the averaged embedding to carry
attributes of every target population, so that the resulting scribble induces
portraits from all of them in roughly the target proportions.

\paragraph{\sdedit{} Best.}
Unguided \sdedit{} is applied to the source scribble with the same
noise strength ($t_{\text{start}} = 125$, $T = 250$) and
$\text{CFG} = 0.0$. The number of candidates is chosen so that the total computation time matches the wall-clock time of \mlgdf{} with standard (non-deterministic) GPU kernels,
ensuring a fair time-budget comparison. Concretely, the time per
candidate is measured on the first candidate, and
\[
  N_{\text{candidates}}
  = \left\lfloor
      \frac{T_{\text{\methodname{}}}}
           {t_{\text{candidate}}}
    \right\rfloor,
\]
where $T_{\text{\methodname{}}} = 178$ minutes for age interpolation and
$150$ minutes for the rest.
Each candidate is scored by proxy \mmd{} on $n = 250$ images;
the best-scoring candidate is re-evaluated on $N = 2{,}000$ images.

\subsection{Evaluation}
\label{app:sd:eval}
For each method and scenario, $N = 2{,}000$ portrait images are
generated from the candidate scribble using \sdxlturbo{} +
\controlnetscribble{} with the neutral prompt:
\begin{quote}
\texttt{"a superrealistic professional photograph of"}
\end{quote}
This neutral prompt is used deliberately so that the conditional
distribution $\Pcond{\xstar}$ is driven by the
scribble geometry alone.

To ensure a fair comparison across methods, all $N = 2{,}000$
images for every method are generated with the \emph{same noise sequence}:
a single random generator, seeded with $42$ and restarted for every method,
so image $i$ receives the same noise for all methods alike. Consequently, the outputs of different methods
differ only through the
candidate scribble $\xhat$.

\mmd{} is computed in \clip{} ViT-L/14 embedding space against a
freshly drawn target sample of $N = 2{,}000$ images.
For the discrete attribute targets (balanced and skewed), gender
classification is additionally performed via two-way \clip{} cosine
softmax against the male and female text prompts. The male
proportion $\hat P_{\mathrm{male}}$ is reported with a $95\%$
normal-approximation confidence interval:
\[
  \hat P_{\mathrm{male}} \pm 1.96 \sqrt{\frac{\hat P_{\mathrm{male}}(1-\hat P_{\mathrm{male}})}{N}}.
\]

\begin{table}[t]
\centering
\caption{Distributional targets ($N=2{,}000$ images; best per scenario in
bold). $\mmd{} = (\widehat{\operatorname{MMD}}^2)^{1/2}$ (stabilized as in
Appendix~\ref{app:sd:loss}); $\Delta$ is its
relative improvement (\%) over the source scribble; \% male is from the
gender classifier, $\pm$ the half-width of its binomial 95\% CI. The target
fraction lies inside \mlgdf{}'s interval in both discrete scenarios.}
\label{tab:sd_results}
\setlength{\tabcolsep}{2pt}
\begin{tabular*}{\columnwidth}{@{\extracolsep{\fill}} l ccc cc @{}}
\toprule
& \multicolumn{3}{c}{\textbf{Discrete}} & \multicolumn{2}{c}{\textbf{Continuum}} \\
\cmidrule(lr){2-4} \cmidrule(l){5-6}
Method & \mmd{}$\downarrow$ & $\Delta$\% & \% male & \mmd{}$\downarrow$ & $\Delta$\% \\
\midrule
\multicolumn{4}{@{}l}{\textit{Balanced ($50\%$ male)}} & \multicolumn{2}{l@{}}{\textit{Gender}} \\
Avg prompt        & $.453$ & $+10.6$ & $41.2{\pm}2.2$ & $.437$ & $+18.3$ \\
\sdedit{} Best    & $.439$ & $+13.2$ & $65.1{\pm}2.1$ & $.455$ & $+15.0$ \\
\textbf{\mlgdf{}} & $\mathbf{.348}$ & $\mathbf{+31.3}$ & $\mathbf{49.3{\pm}2.2}$ & $\mathbf{.349}$ & $\mathbf{+34.8}$ \\
\midrule
\multicolumn{4}{@{}l}{\textit{Skewed ($25\%$ male)}} & \multicolumn{2}{l@{}}{\textit{Age}} \\
Avg prompt        & $.476$ & $+18.1$ & $20.7{\pm}1.8$ & $.447$ & $+14.0$ \\
\sdedit{} Best    & $.443$ & $+23.8$ & $43.0{\pm}2.2$ & $.481$ & $+7.5$ \\
\textbf{\mlgdf{}} & $\mathbf{.368}$ & $\mathbf{+36.8}$ & $\mathbf{24.6{\pm}1.9}$ & $\mathbf{.345}$ & $\mathbf{+33.6}$ \\
\bottomrule
\end{tabular*}
\end{table}

\subsection{Cross-Metric Validation: \swd{} Improvements}
\label{app:sd:swd}

Since \mlgdf{} optimizes \mmd{} directly (Appendix~\ref{app:sd:loss}),
we additionally report the Sliced Wasserstein Distance (\swd{}; see
Appendix~\ref{sec:app:swd}) between each method's generated portraits
and the target distribution, computed on the same $N=2{,}000$ images as
Table~\ref{tab:sd_results}. \mlgdf{} optimizes the \mmd{} only, so the \swd{} is
an independent
check of the improvements.
$\Delta$\swd{} (\%) is defined analogously to $\Delta$\mmd{} (\%) in
Table~\ref{tab:sd_results}.

\begin{table}[h]
\centering
\caption{\swd{} to the target across all four scenarios ($N=2{,}000$ images;
the \mmd{} values are in Table~\ref{tab:sd_results}). \mlgdf{} achieves the best
\swd{} in every scenario, confirming the \mmd{} improvements under an
independent metric.}
\label{tab:sd_swd}
\setlength{\tabcolsep}{4pt}
\begin{tabular}{@{}l cc@{}}
\toprule
\textbf{Method} & \textbf{\swd{}} $\downarrow$ & \textbf{$\Delta$\swd{} (\%)} \\
\midrule
\multicolumn{3}{l}{\textit{Balanced ($50\%$ male)}} \\
Source & $0.0111$ & $+0.0\%$ \\
Avg prompt & $0.0098$ & $+11.3\%$ \\
\sdedit{} Best & $0.0111$ & $+0.0\%$ \\
\textbf{\mlgdf{}} & $\mathbf{0.0080}$ & $\mathbf{+28.1\%}$ \\
\midrule
\multicolumn{3}{l}{\textit{Skewed ($25\%$ male)}} \\
Source & $0.0130$ & $+0.0\%$ \\
Avg prompt & $0.0118$ & $+9.0\%$ \\
\sdedit{} Best & $0.0103$ & $+20.3\%$ \\
\textbf{\mlgdf{}} & $\mathbf{0.0093}$ & $\mathbf{+28.0\%}$ \\
\midrule
\multicolumn{3}{l}{\textit{Gender interpolation}} \\
Source & $0.0133$ & $+0.0\%$ \\
Avg prompt & $0.0100$ & $+24.7\%$ \\
\sdedit{} Best & $0.0108$ & $+18.4\%$ \\
\textbf{\mlgdf{}} & $\mathbf{0.0082}$ & $\mathbf{+38.1\%}$ \\
\midrule
\multicolumn{3}{l}{\textit{Age interpolation}} \\
Source & $0.0131$ & $+0.0\%$ \\
Avg prompt & $0.0093$ & $+29.0\%$ \\
\sdedit{} Best & $0.0118$ & $+10.4\%$ \\
\textbf{\mlgdf{}} & $\mathbf{0.0077}$ & $\mathbf{+41.5\%}$ \\
\bottomrule
\end{tabular}
\end{table}

\mlgdf{} has the lowest \swd{} in every scenario, and the largest
$\Delta$\swd{} in all four.

\subsection{Compute}
\label{app:sd:compute}

\mlgdf{} optimization runs were executed on a single NVIDIA L40S GPU
(48\,GB) on a Linux cluster (CPython 3.11, PyTorch 2.11 with CUDA 12.8).
With standard GPU kernels a run takes about 178 minutes for the age interpolation and 150
minutes for the rest; the released runs use deterministic kernels for exact
reproducibility, which roughly doubles this (4.7--5.7\,hr).
The \sdedit{} Best baseline was evaluated on the same hardware;
the time per candidate was measured empirically and the candidate
budget was set to match \mlgdf{}'s wall-clock time exactly
(Section~\ref{app:sd:baselines}).
The Average Prompt scribble baseline requires no optimization and runs
in under one minute. Final validation on $2{,}000$ samples was performed on the same NVIDIA L40S (about 30 minutes per scenario).
\subsection{Scribble Edits Made by \mlgdf{}}
\label{app:sd:interp}

\paragraph{Where \mlgdf{} edits the scribble.}
Figure~\ref{fig:scribble_change_regions} shows the source scribble, the
\mlgdf{}-optimized scribble, and their signed pixel-wise difference (green:
stroke added, red: stroke removed). The background is left unchanged and all
edits lie on or next to existing strokes. Two kinds of edit are visible. Along
almost every contour (the head outline, hair, neck and shoulders) the strokes
are slightly displaced and reshaped, visible as adjacent pairs of added and
removed pixels: a low-magnitude but widespread change. The local edits
are concentrated on the facial features and the hair outline: strokes are removed
and added around the eyes and eyebrows, the mouth and chin, and along the side of
the hair.

Yet even with the pixel-difference heatmap in hand, it remains unclear
\emph{why} these subtle edits suffice to shift the generated
distribution toward a balanced male--female output. From a human
perceptual standpoint, the changes are too small and too localized to
constitute a legible gender signal. We hypothesize that \mlgdf{} has
identified a structural sensitivity in the \controlnet{}--\sdxl{} pipeline:
a region of the conditioning space where small perturbations
disproportionately influence the generative process, not because they
are semantically meaningful to a human observer, but because the
diffusion model's internal representations are unexpectedly responsive
to them. This points to a \textbf{potential vulnerability in the underlying
generative model}, one that \mlgdf{} successfully exploits to achieve
distributional control, but that also raises questions about the
robustness and interpretability of \controlnet{}-based conditioning more
broadly.

\begin{figure}[h]
    \centering
    \includegraphics[width=\textwidth]{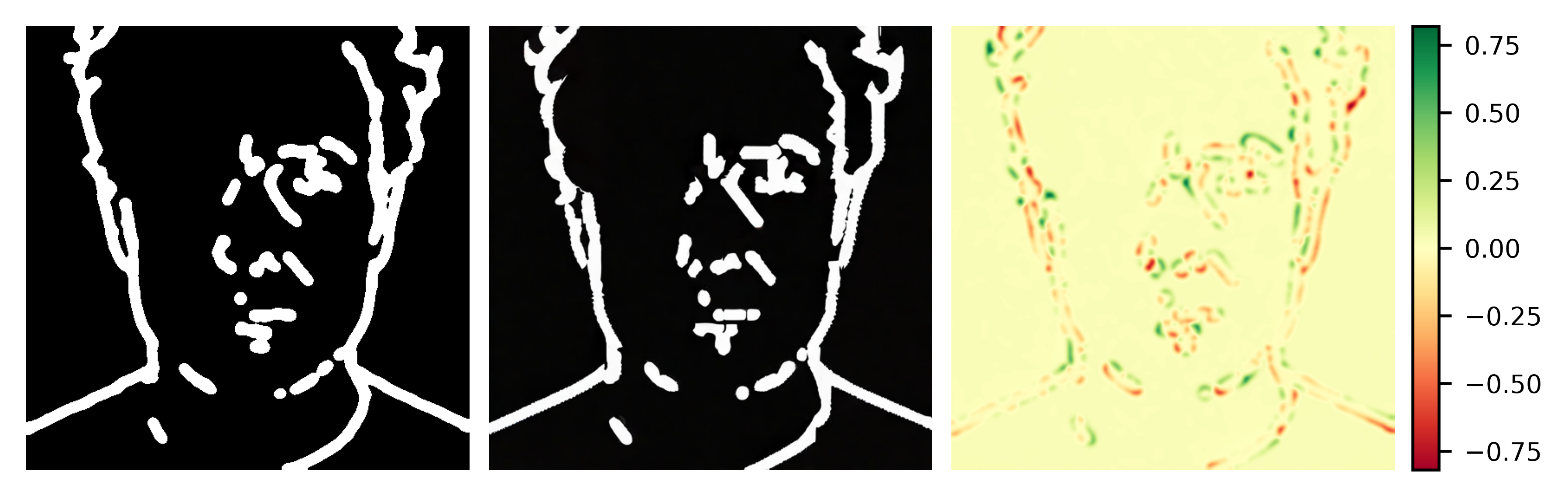}
    \caption{
        \textbf{Where \mlgdf{} edits the scribble (balanced target).}
        \textbf{Left:} Source scribble used to initialize the \sdedit{}
        trajectory. \textbf{Middle:} \mlgdf{}-optimized scribble.
        \textbf{Right:} Signed difference (green $=$ added, red $=$ removed).
        Small edits run along almost every contour; the strongest concentrate
        around the eyes and eyebrows, the mouth and chin, and the hair outline.
        Quantitatively, $91\%$ of pixels are unchanged, while $4.0\%$ are added
        and $4.6\%$ are removed, where a pixel is considered changed if its
        Gaussian-smoothed signed difference ($\sigma = 2.5$) exceeds a
        threshold of $\tau = 0.1$ in absolute value.
    }
    \label{fig:scribble_change_regions}
\end{figure}

\paragraph{Why those regions: gradient saliency.}
To test whether these edits are mechanistically meaningful rather than
incidental, we ask which scribble pixels the \clip{} gender signal is
\emph{gradient-sensitive} to, independently of the edits \mlgdf{} made.
Define the per-pixel saliency as $|\partial s / \partial x|$, where
$s(x) = \cos\Bigl(\mathrm{CLIP}\bigl(I(x)\bigr),\, t_{\text{man}}\Bigr) -
\cos\Bigl(\mathrm{CLIP}\bigl(I(x)\bigr),\, t_{\text{woman}}\Bigr)$
is the \clip{} gender score of the image $I(x)$ generated from scribble
$x$, and $t_{\text{man}}, t_{\text{woman}}$ are the text embeddings of
``a photo of a man''/``a photo of a woman''. The gradient is obtained
by backpropagating $s$ from the generated image of the source scribble
through the full differentiable pipeline (scribble $\to$
\controlnetscribble{} and \sdxlturbo{} $\to$ VAE decoder $\to$ \clip{} image
encoder) and averaged over $20$ random seeds to marginalize over sampler noise.
Large saliency values indicate pixels whose local perturbation would
most change the downstream gender score.
We restrict the analysis to the line pixels of the source scribble,
i.e.\ the non-background pixels of the input, since
our question is whether, among the strokes that were actually drawn,
the ones more predictive of the target attribute receive proportionally
larger edits. Figure~\ref{fig:saliency_and_correlation} shows
the saliency map on the source scribble's line support; the
highest-saliency pixels concentrate on the eye and eyebrow strokes, the
mouth and the hairline. We quantify the alignment by
binning line pixels by saliency magnitude and plotting the mean \mlgdf{}
modification per bin (Figure~\ref{fig:saliency_and_correlation}). The
binned change is flat over the least salient strokes, rises steeply with
saliency and then increases more slowly at the highest saliency
(Spearman $\rho = 0.53$ over the $36{,}668$ line pixels): pixels of the source
stroke that the gender score is more sensitive to receive proportionally
larger edits.

To check whether the concentration of edits on existing strokes is merely an
artefact of the latent geometry, we encoded the source and \mlgdf{}-optimized
scribbles through the VAE (with the $[-1,1]$ input range used by the pipeline)
and computed, at each position of the $64{\times}64$ latent grid, the source
latent norm (L2 norm across the four channels) and the edit magnitude (L2 norm
of the latent difference). The two maps are essentially uncorrelated (Pearson
$r = 0.03$): where \mlgdf{} edits is not explained by where the source latent is
large, which supports the saliency account above.

\begin{figure}[h]
    \centering
    \includegraphics[width=\textwidth]{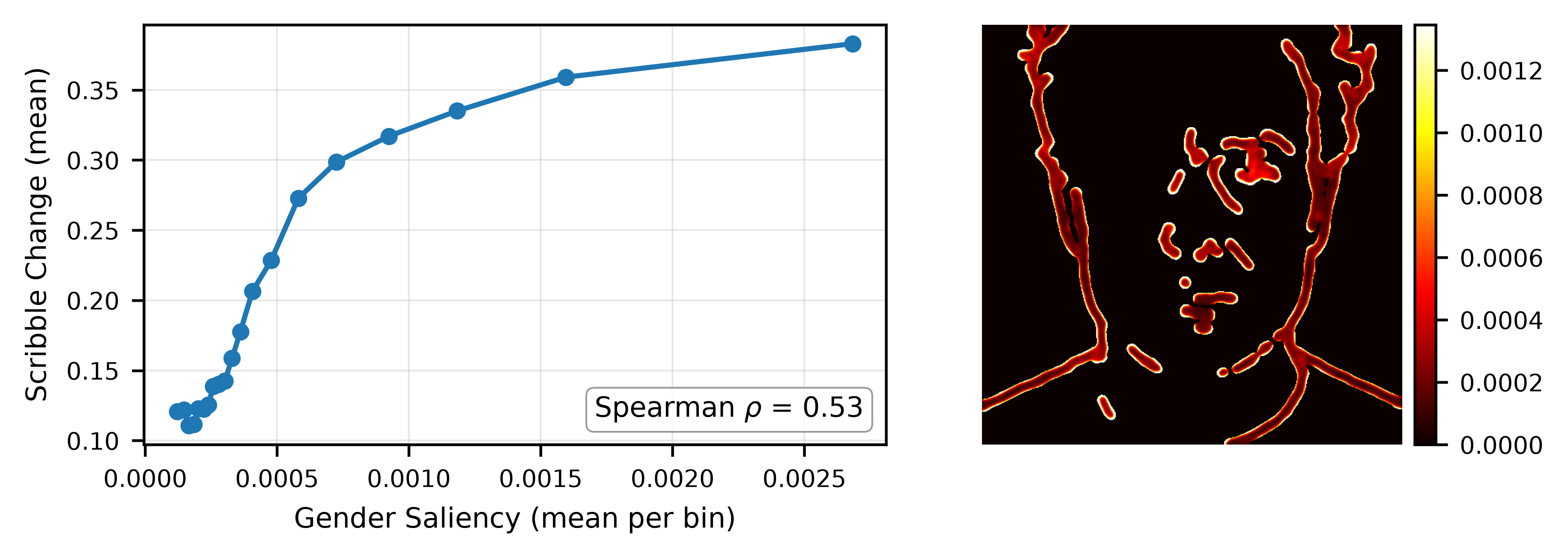}
    \caption{
        \textbf{Why those regions: gradient saliency and its alignment
        with \mlgdf{} edits.} \textbf{Left:} Binned mean \mlgdf{} scribble
        change versus binned mean saliency, on line pixels only. The monotone
        relationship (Spearman $\rho = 0.53$) shows that \mlgdf{} preferentially
        edits the source-scribble strokes that are most gradient-sensitive for
        the downstream gender score. \textbf{Right:} Gender saliency
        $|\partial s/\partial x|$ on the source scribble's line
        pixels, averaged over $20$ seeds; $s$ is the \clip{} gender
        score backpropagated through the full scribble $\to$
        \controlnet{} and \sdxlturbo{} $\to$ VAE $\to$ \clip{} pipeline.
        High-saliency regions concentrate on the eye and eyebrow strokes, the
        mouth and the hairline.
    }
    \label{fig:saliency_and_correlation}
\end{figure}

\paragraph{Takeaway.}
The two analyses together resolve the apparent paradox of small pixel
edits producing a large distributional shift. Among the strokes of
the source scribble, \mlgdf{} injects an
\emph{aligned} gradient signal that preferentially modifies the
subset of strokes the downstream \clip{} score is structurally sensitive
to. This validates the inner estimator of
Algorithm~\ref{alg:inner-estimator} as an interpretable \emph{mechanism}:
despite the long differentiable chain from scribble
through \controlnet{}, VAE, and \clip{}, the pulled-back gradient at each
denoising step retains enough semantic structure to localize onto the
anatomical features actually predictive of the target attribute.
The disproportion
between the small edits and the large distributional effect is the sensitivity
of the pipeline noted above.
The saliency analysis and latent correlation were computed on a single
NVIDIA L40S GPU, with the saliency maps averaged over $20$ random seeds.
We report this analysis for the balanced target; we expect the same
mechanism in the skewed and gender-interpolation scenarios, but have not
verified it quantitatively.

\FloatBarrier
\subsection{Scope of the Gender and Age Targets}
\label{app:sd:scope}
Gender and age serve here only as convenient, visually salient attributes for
building discrete and continuous target distributions; the experiment makes no
claim about gender or age as human characteristics. The targets are defined by
text prompts and \clip{} embeddings, and the reported male fraction comes from a
binary \clip{} zero-shot classifier, so both reflect how these pretrained models
represent the prompts, including their biases, not a ground truth. All portraits
are synthetic, and no real person or personal data is involved. The scribble
edits found by \mlgdf{} are not human-interpretable gender or age cues: they
are small perturbations to which the \controlnet{}--\sdxl{} pipeline is
sensitive (Appendix~\ref{app:sd:interp}). The same mechanism can be used to
balance the demographic distribution of a generator's outputs or to skew it;
we use it only to test distributional control.

\subsection{Qualitative Examples}
\label{app:sd:qualitative}

Figures~\ref{fig:app_bimodal_qual}--\ref{fig:app_age_qual} show,
side by side, the scribbles produced by each method for all four
scenarios, along with representative portrait images generated
from each scribble.
\begin{figure}[!htbp]
    \centering
    \includegraphics[width=0.98\textwidth]{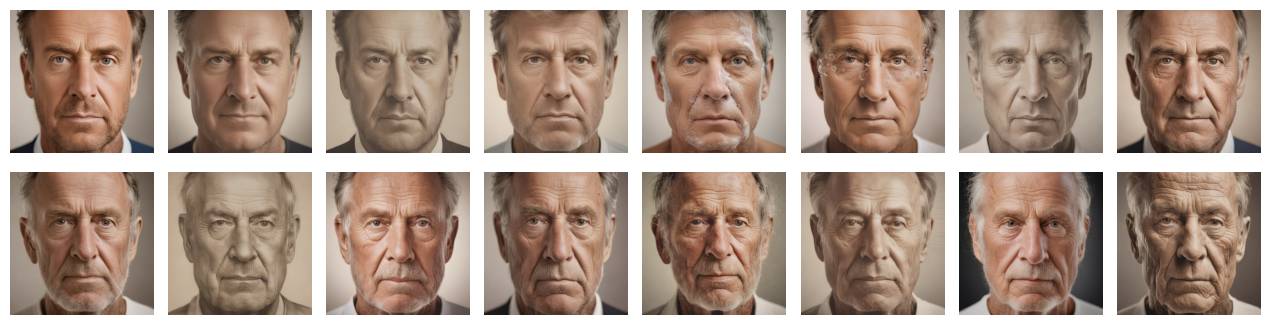}
    \caption{
        \textbf{Age interpolation }
        Sixteen representative portrait images generated by \mlgdf{} for a
        uniform
        age target over $[40, 80]$ ordered by the first clip PCA projection,
        (out of 100 total). All images are generated
        with the same optimized scribble and a neutral prompt.
    }
    \label{fig:app_age_grid}
\end{figure}

\begin{figure}[p]
    \centering
    \includegraphics[width=0.98\textwidth]{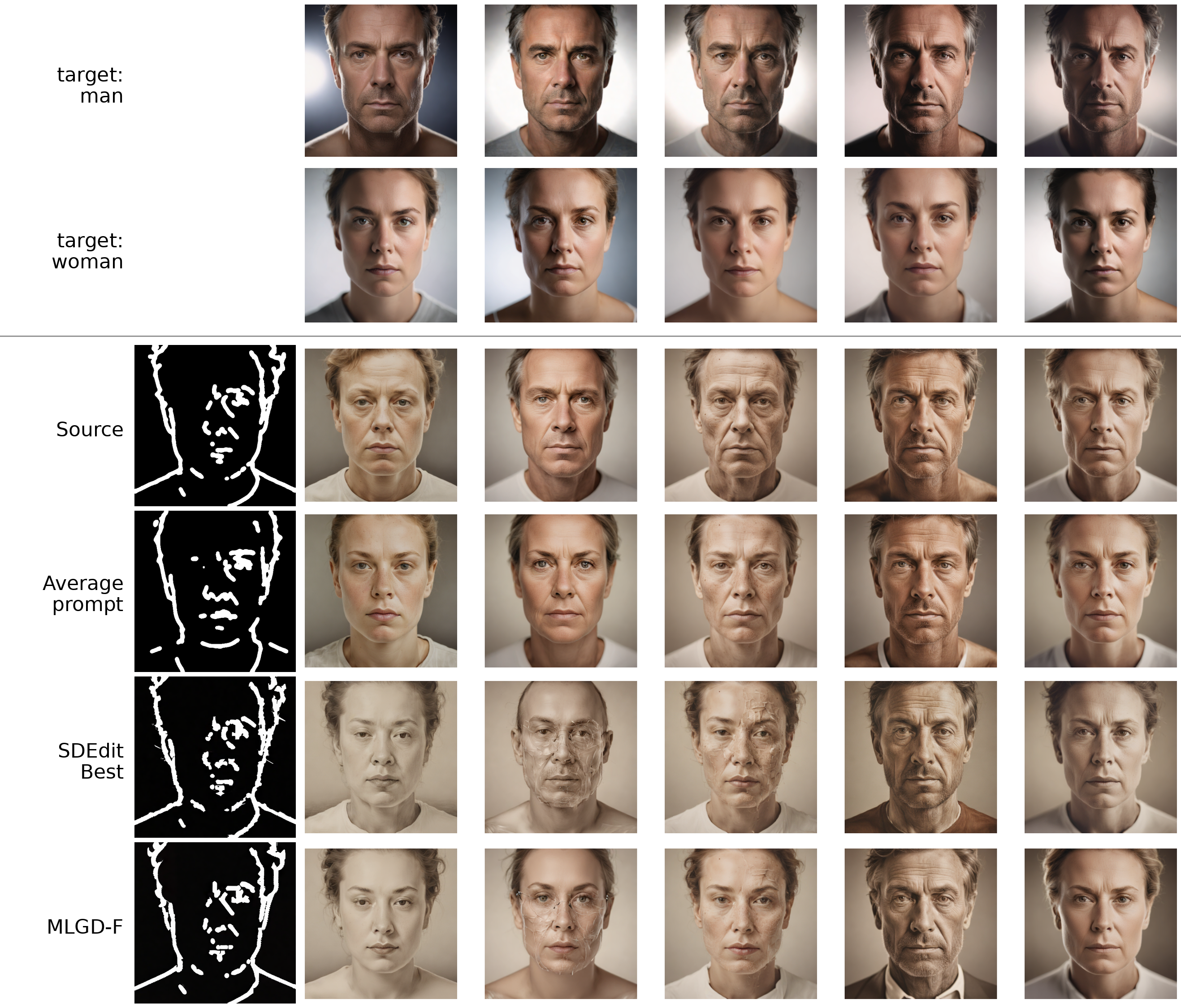}
    \caption{\textbf{Balanced target} $\target_{bal} = 0.5\,\delta_{\text{male}} + 0.5\,\delta_{\text{female}}$. Top: five target portraits per target component. Below: for each method, its
scribble and five portraits generated from it (\sdxlturbo{} with
\controlnetscribble{}, neutral prompt); each column uses the same seed, so the
differences between rows are due to the scribble alone. \mlgdf{} produces a balanced mix of male and female portraits.}
    \label{fig:app_bimodal_qual}\label{fig:app_bimodal_cond}
\end{figure}

\begin{figure}[p]
    \centering
    \includegraphics[width=0.98\textwidth]{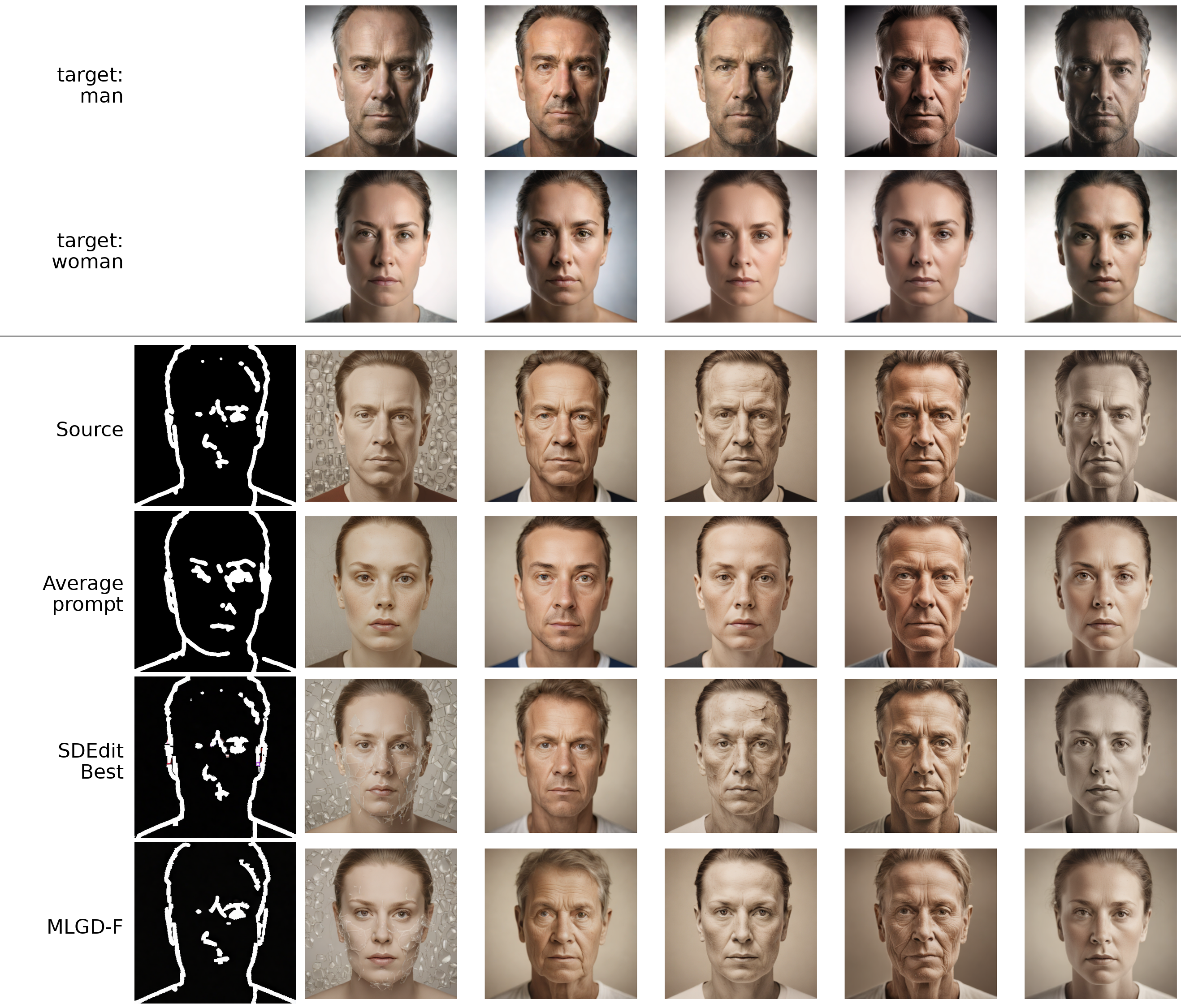}
    \caption{\textbf{Skewed target} $\target_{skew} = 0.25\,\delta_{\text{male}} + 0.75\,\delta_{\text{female}}$. Top: five target portraits per target component. Below: for each method, its
scribble and five portraits generated from it (\sdxlturbo{} with
\controlnetscribble{}, neutral prompt); each column uses the same seed, so the
differences between rows are due to the scribble alone. \mlgdf{} produces mostly female portraits, consistent with the $75\%$ female target.}
    \label{fig:app_ratio_qual}\label{fig:app_ratio_cond}
\end{figure}

\begin{figure}[p]
    \centering
    \includegraphics[width=0.98\textwidth]{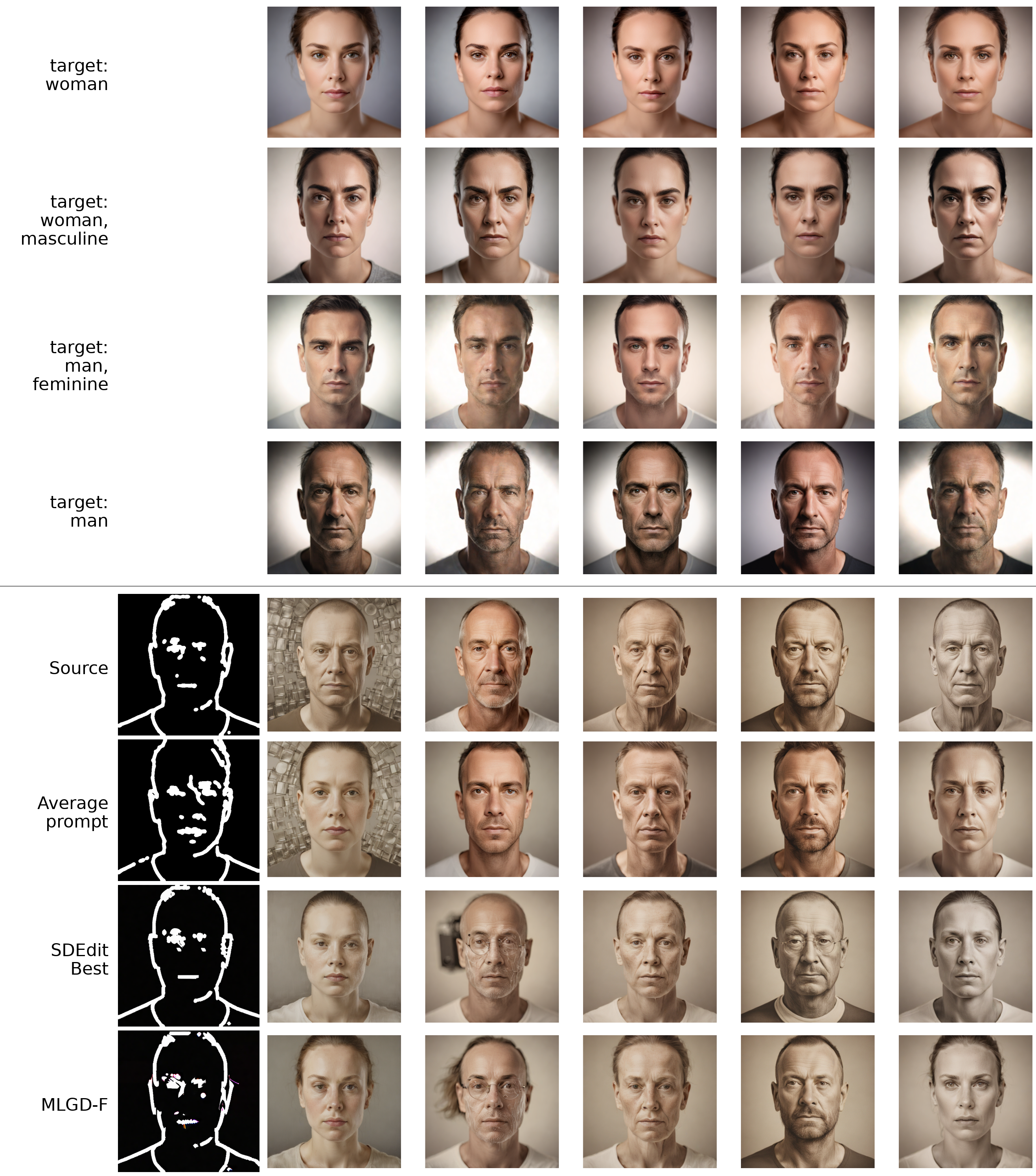}
    \caption{\textbf{Gender interpolation target} $\target_{interpGender} = \frac{1}{4}\sum_{k=1}^{4}\delta_{c_k}$, four equally weighted anchors along the feminine-to-masculine axis in \clip{} space. Top: five target portraits per target component. Below: for each method, its
scribble and five portraits generated from it (\sdxlturbo{} with
\controlnetscribble{}, neutral prompt); each column uses the same seed, so the
differences between rows are due to the scribble alone. The second \mlgdf{} portrait combines a masculine facial structure with long hair, an intermediate point on the continuum rather than one of its ends.}
    \label{fig:app_interp_qual}\label{fig:app_interp_cond}
\end{figure}

\begin{figure}[p]
    \centering
    \includegraphics[width=0.98\textwidth]{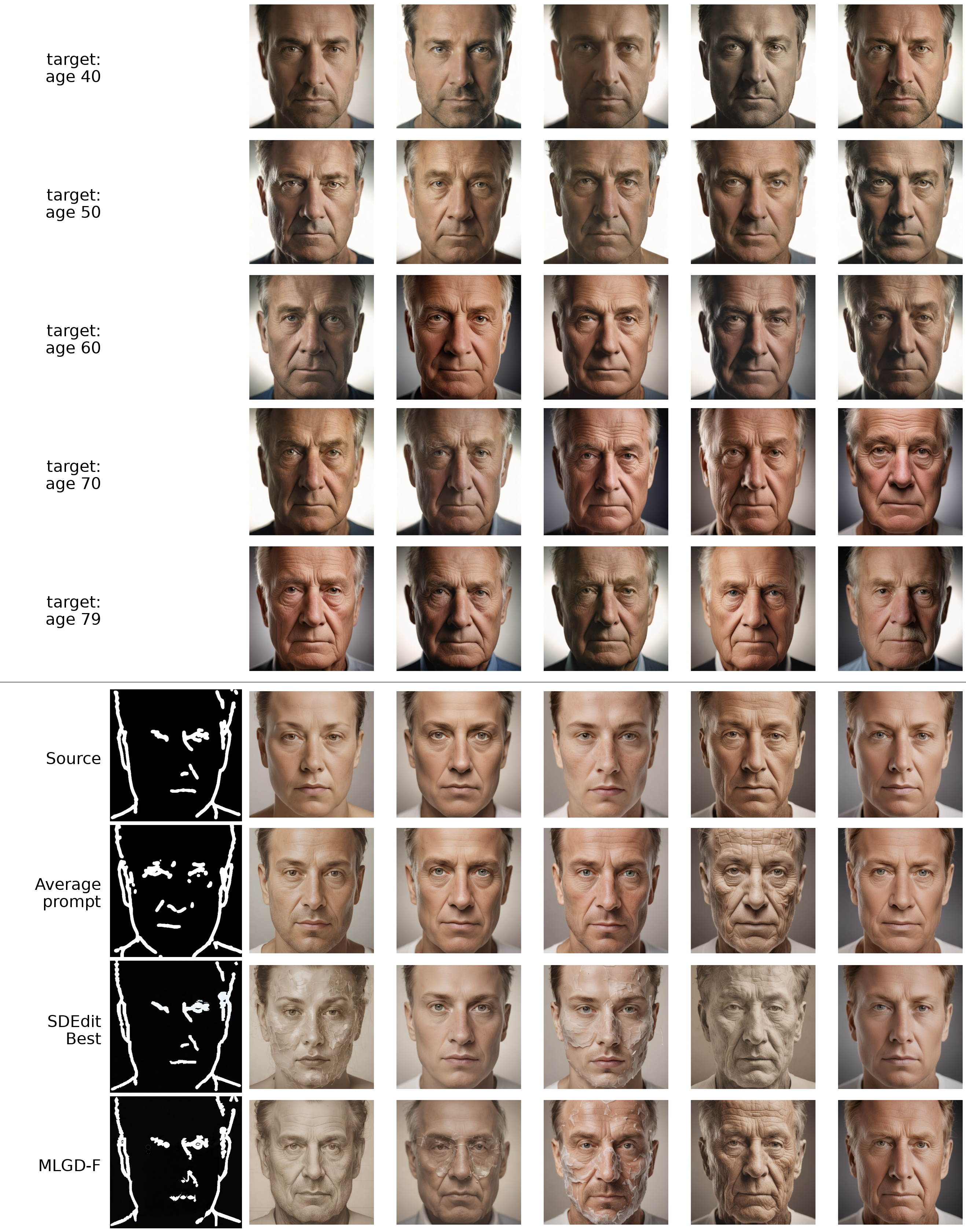}
    \caption{\textbf{Age interpolation target} $\target_{interpAge}$, male portraits with age uniform on $40$--$79$. Top: five target portraits per target component. Below: for each method, its
scribble and five portraits generated from it (\sdxlturbo{} with
\controlnetscribble{}, neutral prompt); each column uses the same seed, so the
differences between rows are due to the scribble alone.}
    \label{fig:app_age_qual}\label{fig:app_age_cond}
\end{figure}

\newpage
\FloatBarrier 
\section{Full Algorithms}
\label{sec:app:fullalg}

Algorithms~\ref{alg:mlgd-outer}--\ref{alg:inner-estimator} abstract the
implementation. Algorithm~\ref{alg:mlgd-d-full} is the procedure as run: $\nmc$
Gaussian perturbations of $\xzero$ at every reverse step, aggregated by a
log-sum-exp, the guidance update, the restarts, the refinement and the final
selection. Algorithm~\ref{alg:gradest} gives the two gradient estimators, and
Algorithm~\ref{alg:refine} the refinement. In the synthetic and MNIST experiments
a fresh target sample is drawn for every loss evaluation; in the Stable
Diffusion experiment $\SG$ is a fixed set of target embeddings, drawn once.

\begin{algorithm}[h]
\caption{\mlgd{} as run}
\label{alg:mlgd-d-full}
\raggedright
\textbf{Input:} $\stheta$, $\{\abar_t\}$, $\fphi$, $\SG$, $\distHat$, $\ncond$ and
$\mathit{mode}$ as in Algorithms~\ref{alg:mlgd-outer}--\ref{alg:inner-estimator};
$\nmc$ perturbations of width $r_t$ ($r_t = g(t)/\sqrt{1+g(t)^2}$, $g$ the
diffusion coefficient of the variance-preserving SDE); guidance step sizes
$\zetat$ (Equation~\ref{eq:guidance-step} or its Adam form;
Table~\ref{tab:guidance-rules}); $R$ restarts; for \cdmo{}, the hyperparameters of
$\algRefine{}$ (Algorithm~\ref{alg:refine}) and the number $R'$ of best restarts
re-estimated with $n_{\mathrm{eval}}$ fresh samples.\\
\textbf{Output:} \cdms{}: $\{x_0^{(r)}\}_{r=1}^R$; \cdmo{}: $\xhat$.\\[3pt]
\begin{algorithmic}[1]
\For{$r = 1, \ldots, R$}
  \State $x_T \sim \mathcal{N}(0, I)$
  \For{$t = T, \ldots, 1$}
    \State $\xzero \leftarrow \tfrac{1}{\sqrt{\abar_t}}\bigl(x_t + (1-\abar_t)\,\stheta(x_t,t)\bigr)$
    \For{$m = 1, \ldots, \nmc$}
      \State $\xzero^{(m)} \leftarrow \xzero + r_t\,\xi_m$, \; $\xi_m \sim \mathcal{N}(0, I)$
      \State $\lossHat^{(m)},\, \nabla_{x_t}\lossHat^{(m)} \leftarrow \algInnerLoss{}(x_t, \xzero^{(m)})$ \Comment{Algorithm~\ref{alg:inner-estimator}}
    \EndFor
    \State $\bar{\mathcal{L}} \leftarrow -\log \tfrac{1}{\nmc} \sum_m e^{-\lossHat^{(m)}}$, \;
           $\nabla_{x_t} \bar{\mathcal{L}} \leftarrow \sum_m w_m \nabla_{x_t}\lossHat^{(m)}$, \; $w_m \propto e^{-\lossHat^{(m)}}$
    \State $x_{t-1} \leftarrow \algGuide{}\bigl(x_t, \xzero, t,\, \bar{\mathcal{L}},\, \nabla_{x_t}\bar{\mathcal{L}},\, \zetat\bigr)$
  \EndFor
  \State $x_0^{(r)} \leftarrow \algRefine{}(x_0)$ \Comment{\cdmo{} only}
\EndFor
\State \textbf{if} \cdms{}: \Return $\{x_0^{(r)}\}_{r=1}^R$
\State $\mathcal{C} \leftarrow$ the $R'$ restarts with the smallest final $\lossHat$
\State \Return $\xhat = \arg\min_{x \in \mathcal{C}} \lossHat(x)$, re-estimated with $n_{\mathrm{eval}}$ fresh samples
\end{algorithmic}
\end{algorithm}

\begin{algorithm}[h]
\caption{\algGradEst{}: the two gradient estimators of Algorithm~\ref{alg:inner-estimator}}
\label{alg:gradest}
\raggedright
\textbf{Input:} $x_t$, $\xzero$, the noise $\eta_1, \ldots, \eta_{\ncond}$ and
samples $\Scond = \{\fphi(\xzero, \eta_i)\}$ of Algorithm~\ref{alg:inner-estimator},
$\SG$, $\distHat$; $\mathit{mode} \in \{\textsc{Backprop}, \textsc{FiniteDiff}\}$.
\textsc{Backprop}: budget $\nsel \le \ncond$ and selection rule
(Appendix~\ref{app:select}). \textsc{FiniteDiff}: direction type, number of
directions $M$, step $h$.\\
\textbf{Output:} $\nabla_{x_t} \lossHat$.\\[3pt]
\begin{algorithmic}[1]
\If{$\mathit{mode} = \textsc{Backprop}$}
  \State $\Ssel \leftarrow \algSelect{}\bigl(\Scond,\, \SG,\, \nsel\bigr)$
  \State \Return $c_{\mathrm{sel}} \sum_{y \in \Ssel} \partial_y \distHat(\Scond, \SG)^{\top} \partial_{x_t} y$, \; $c_{\mathrm{sel}} = \ncond/\nsel$ for Uniform, $1$ for All (Appendix~\ref{app:select})
\EndIf
\State $\mathit{mode} = \textsc{FiniteDiff}$: directions $U = [u_1, \ldots, u_M]$ and scale $c$:
\Statex \hspace{1.2em} coordinate: $u_k = e_k$, $M = \dimx$, $c = 1$;
\Statex \hspace{1.2em} orthonormal: $U$ from the QR decomposition of a $\dimx \times M$ Gaussian matrix, $c = \dimx / M$;
\Statex \hspace{1.2em} tangent: $U$ orthonormalizes $J v_1, \ldots, J v_M$, $v_k \sim \mathcal{N}(0, I)$, $J = \partial \xzero / \partial x_t$, $c = 1$
\State $\ell_k^{\pm} \leftarrow \distHat\bigl(\{\fphi(\xzero \pm h u_k, \eta_i)\}_i,\, \SG\bigr)$, $k = 1, \ldots, M$ \Comment{same $\eta_i$}
\State \Return $J^{\top} c \sum_k \tfrac{\ell_k^{+} - \ell_k^{-}}{2h}\, u_k$, \; $J = \partial \xzero / \partial x_t$
\end{algorithmic}
\end{algorithm}

\begin{algorithm}[h]
\caption{\algRefine{}: local refinement of a returned design}
\label{alg:refine}
\raggedright
\textbf{Input:} $x_0$; number of steps $K_r$; Adam learning rate; $n_{\mathrm{ref}}$
samples per step; \algGradEst{} (Algorithm~\ref{alg:gradest}) with the same $\mathit{mode}$, applied
at $x$ itself ($\xzero = x$, $J = I$); keep-best flag.\\
\textbf{Output:} the refined design.\\[3pt]
\begin{algorithmic}[1]
\State $x \leftarrow x_0$, \; $(\ell^\ast, x^\ast) \leftarrow (\infty, x_0)$
\For{$k = 1, \ldots, K_r$}
  \State $\lossHat(x) \leftarrow \distHat\bigl(\{\fphi(x, \eta_i)\}_{i=1}^{n_{\mathrm{ref}}},\, \SG\bigr)$ \Comment{fresh $\eta_i$, $\SG$}
  \State \textbf{if} $\lossHat(x) < \ell^\ast$: $(\ell^\ast, x^\ast) \leftarrow (\lossHat(x), x)$
  \State $x \leftarrow$ Adam step on $x$ with $\nabla_x \lossHat(x)$ from \algGradEst{}
\EndFor
\State \Return $x^\ast$ with keep-best (\mlgd{}$^\star$), $x$ otherwise
\end{algorithmic}
\end{algorithm}

\subsection{Instantiations of GUIDE}
\label{app:guidance-rules}
Table~\ref{tab:guidance-rules} lists the guidance rules that can fill
the $\algGuide{}$ slot of Algorithm~\ref{alg:mlgd-outer}.
Throughout, $x'_{t-1}$ denotes the unguided \ddim{} step of
Equation~\ref{eq:guidance-step} and $\xzero$ the Tweedie estimate.
All our runs use the \lgd{} rule (second row), with log-sum-exp
aggregation over the perturbed $\xzero$ (Algorithm~\ref{alg:mlgd-d-full}).
The step is normalized as in \dps{} (first row) in the Stable Diffusion
experiment and as in AdamDPS (fourth row) in the synthetic experiments, and
follows a noise-level schedule in MNIST (Section~\ref{sec:methods:outer}).

\begin{table}[h]
\centering
\caption{Guidance rules usable as $\algGuide{}$. All take the
loss $\loss(\xzero)$ and its gradient from
Algorithm~\ref{alg:inner-estimator}; they differ in how the gradient is
formed and applied.}
\label{tab:guidance-rules}
\setlength{\tabcolsep}{4pt}
\begin{tabular}{@{}p{2.5cm} p{6.4cm} p{3.2cm} p{1.6cm}@{}}
\toprule
\textbf{Rule} & \textbf{Update} & \textbf{Hyperparameters} & \textbf{Used here} \\
\midrule
\dps{}~\citep{DiffusionPosteriorSampling} &
$x_{t-1} = x'_{t-1} - \zetat \nabla_{x_t}\loss(\xzero)$ with
$\zetat = \zeta'/\loss(\xzero)$ (step size normalized by the
current loss, the residual norm in the original). &
$\zeta'$ & yes \\
\addlinespace
\lgd{}~\citep{lossGuidedDiffusion} &
Same additive step, with $\loss(\xzero)$ replaced by
$-\log \tfrac{1}{\nmc}\sum_{m} \exp\bigl(-\loss(\xzero + r_t\xi_m)\bigr)$,
$\xi_m \sim \mathcal{N}(0,I)$, a Monte-Carlo average over the
posterior of $x_0$ given $x_t$. &
$\zeta'$, $\nmc$, $r_t$ & yes \\
\addlinespace
\tfg{}~\citep{ye2024tfgunifiedtrainingfreeguidance} &
Unified design space: mean guidance on $\xzero$ (strength $\mu_t$),
variance guidance on $x_t$ (strength $\rho_t$), $N_{\mathrm{recur}}$
recurrence passes (re-noise and re-denoise), and an implicit-dynamic
Gaussian smoothing of the predictor (strength $\bar\gamma$); \dps{} and
\lgd{} are special cases (their Thm.~3.2), and the schedules are chosen
by hyperparameter search. &
$\mu_t$, $\rho_t$, $N_{\mathrm{recur}}$, $\bar\gamma$ & no \\
\addlinespace
AdamDPS~\citep{belardi2026adaptivemomentssurprisinglyeffective} &
\dps{} step with $\nabla_{x_t}\loss$ replaced by an Adam direction
whose first and second moments are accumulated across denoising steps. &
$\zeta'$, $\beta_1$, $\beta_2$ & yes \\
\bottomrule
\end{tabular}
\end{table}

\section{Ablation Study: Full Details}
\label{app:ablation}
\label{sec:ablation}

\paragraph{Summary.}
We compare our few-step inner sampler against standard diffusion
inner-loop sampling. Both configurations share an identical backbone
(UNet, \controlnet{}, \clip{} text encoders) and differ only in step
count ($\Ks=2$ for \sdxlturbo{} vs.\ $\Kstar=30$ for \sdxlbase{};
parity details in Appendix~\ref{app:ablation:parity}), and the ablation
uses the All selection rule ($\nsel = \ncond$),
the worst case for our method. On an NVIDIA L40S (48\,GB) with
$\ncond=100$, \sdxlturbo{} peaks at 43\,GB VRAM and
takes 4.4\,hr for a 125-step run; \sdxlbase{} is projected
(per-sample activation cost scaled by $\Kstar/\Ks$,
Appendix~\ref{app:ablation:memory}) to need ${\approx}372$\,GB,
$7.8\times$ the available memory ($4.7\times$ that of an 80\,GB A100), and 66\,hr, rendering it
infeasible. Proposition~\ref{prop:memory} (Appendix~\ref{app:theory})
formalizes this: the few-step student reduces reverse-mode graph depth
by $\Kstar/\Ks$ at the cost of an extra student-teacher gradient
discrepancy $O(\varepsilon_{g,\mathrm{dist}}+\varepsilon_{\mathrm{dist}})$.

\subsection{Architectural Parity and Scaling}
\label{app:ablation:parity}

To isolate the impact of the denoising step count, both
configurations share the same backbone: the UNet ($\approx$2.6B parameters),
\controlnet{}~\citep{controlnet}, and \clip{} text encoders~\citep{clip}
remain constant. Notably, \sdxlturbo{} is obtained by distilling \sdxlbase{}
with the architecture unchanged: the two models differ only in their
weights, not their structure~\citep{sauer2023adversarialdiffusiondistillation}.
A single forward pass therefore costs the same in both, and the differences in
memory and speed come from the step count alone: the \sdxlbase{} gradient graph
is $\Kstar/\Ks = 15\times$ deeper. Gradient checkpointing trades stored
activations for recomputation but still scales linearly in $K$.

\subsection{Memory Cost Decomposition}
\label{app:ablation:memory}

Peak memory decomposes as $M_{\text{peak}} = M_{\text{models}} +
M_{\text{fixed}} + \ncond \cdot M_{\text{var}}(K)$, with $M_{\text{models}}$ the
weights, $M_{\text{fixed}}$ the unconditional forward pass and
$M_{\text{var}}(K)$ the per-sample graph of a $K$-step sampler, i.e.\
$\Theta(\ncond \cdot K)$ (Section~\ref{sec:methods:inner}).

For the Stable Diffusion experiment (Section~\ref{sec:exp:sd}), we
measured peak memory empirically on an NVIDIA L40S (48\,GB), the GPU of
the optimization runs and the final evaluation, using $\ncond=100$
conditional samples per step.

\section{Selection Strategies for the Inner Estimator}
\label{app:select}

The \algSelect{} step of the \textsc{Backprop} estimator (Algorithm~\ref{alg:gradest}) maps the
conditional sample $\Scond = \{y_i\}_{i=1}^{n}$, $y_i = \fphi(\xzero, \eta_i)$
with $n = \ncond$, the target sample $\SG = \{z_j\}_{j=1}^{m}$ and a budget
$\nsel \le n$ to a subset $\Ssel \subseteq \Scond$ with $|\Ssel| = \nsel$. Only
the samples in $\Ssel$ carry gradient; the other $n - \nsel$ are detached, so
they enter the loss as constants and no backward pass is run through $\fphi$ for
them. The loss value is always computed on all $n$ samples. In this appendix we
write $\kappa$ for the kernel.

\paragraph{All.} $\Ssel = \Scond$, i.e.\ $\nsel = \ncond$. This is the estimator
of Algorithm~\ref{alg:inner-estimator} with \textsc{Backprop} and no subsampling, and the selection
rule used throughout the ablation of Section~\ref{sec:ablation}.

\paragraph{Uniform.} $\Ssel$ is a uniformly random $\nsel$-subset of $\Scond$,
drawn without replacement, and the gradient through the retained samples is
rescaled by $n/\nsel$. Each sample is retained with probability $\pi_\ell =
\nsel/n$

\paragraph{Cost.} Both rules generate all $\ncond$ samples, since the loss is
evaluated on all of them, and differ in how many are backpropagated. Uniform
draws $\Ssel$ before generation, generates the $n - \nsel$ non-selected samples
without recording their computation graph, and keeps the graph only for the
$\nsel$ retained ones. Activation memory and backward-pass cost therefore scale
with $\nsel$ rather than $n$. At Stable Diffusion scale, where the backward pass
through the \controlnetscribble{}/\sdxlturbo{}/\clip{} chain dominates,
$\nsel/\ncond = 0.5$ reduces the per-step time by about $42\%$
(Table~\ref{tab:select-runtime}). The saving is less than one half because the
forward pass over all $\ncond$ samples and the outer-loop overhead are
unchanged.
\begin{table}[t]
\centering
\caption{Per-step wall-clock time (mean $\pm$ std over 5 repeats) of the Stable
Diffusion inner estimator with All and Uniform selection, $\nsel/\ncond = 0.5$,
measured in isolation (one guidance step, no outer loop, no logging).}
\label{tab:select-runtime}
\begin{tabular}{lcc}
\toprule
$\ncond$ & All (s/step) & Uniform, $\nsel/\ncond{=}0.5$ (s/step) \\
\midrule
100 & $126.24 \pm 0.87$ & $72.99 \pm 0.81$ \\
120 & $151.08 \pm 0.47$ & $87.39 \pm 0.71$ \\
\bottomrule
\end{tabular}
\end{table}

\subsection{The \mmd{} gradient uses all samples}
\label{app:select:mmd}

Subsampling the backward pass discards less information for the \mmd{} than the
budget $\nsel$ suggests. For the U-statistic of
Equation~\eqref{eq:mmd-ustat-def}, the gradient with respect to one conditional
sample is
\begin{equation}
  \frac{\partial\, \widehat{\operatorname{MMD}}^2_U}{\partial y_\ell}
  \;=\; \frac{2}{n(n-1)}\sum_{i \neq \ell} \nabla_1 k(y_\ell, y_i)
  \;-\; \frac{2}{nm}\sum_{j=1}^{m} \nabla_1 k(y_\ell, \tilde{y}_j),
  \label{eq:mmd-sample-grad}
\end{equation}
which depends on all conditional samples $\{y_i\}$ and all target samples
$\{\tilde{y}_j\}$, not only on those in $\Ssel$.

For a symmetric kernel, its gradient with respect to one conditional sample is
\begin{equation}
  g_\ell \;:=\; \frac{\partial \distHat}{\partial y_\ell}
  \;=\; \frac{2}{n}\Biggl(\underbrace{\frac{1}{n}\sum_{i=1}^{n} \nabla_1
  \kappa(y_\ell, y_i)}_{\text{all conditional samples}}
  \;-\; \underbrace{\frac{1}{m}\sum_{j=1}^{m} \nabla_1 \kappa(y_\ell,
  z_j)}_{\text{all target samples}}\Biggr).
  \label{eq:mmd-sample-grad-v}
\end{equation}
The first sum pulls $y_\ell$ away from the other conditional samples and the
second pulls it toward the target sample, and both run over every sample, not
only over $\Ssel$.

\section{Analysis of the \mlgdf{} Estimator}
\label{app:theory}
 
This appendix analyzes the plug-in loss estimator and the gradient estimator
used inside
the \mlgdf{} inner estimator (Algorithm~\ref{alg:inner-estimator} with the
\textsc{Backprop} estimator of Algorithm~\ref{alg:gradest}). It quantifies \emph{how far} the practical
estimator
is from the population quantity we would ideally optimize, and in particular
makes
explicit the two error sources introduced by the distilled conditional sampler
$\fphi$:
(i) the \emph{sample-level} error $\epss$ with which $\fphi$ reproduces the
true conditional, and (ii) the \emph{gradient} error $\epsg$ with which the
Jacobian of $\fphi$ in the conditioning input $x$ matches the true
reparameterizing
map (Proposition~\ref{prop:grad}); a direct student-vs-teacher analog is the
subject of Proposition~\ref{prop:memory}, which ties the bounds to the memory
argument of Section~\ref{sec:ablation}: a $\Ks$-step student replacing a
$\Kstar$-step teacher cuts reverse-mode graph depth by the ratio $\Kstar/\Ks$
at the cost of an additional student-teacher gradient discrepancy controlled
by $\varepsilon_{g,\mathrm{dist}}, \varepsilon_{\mathrm{dist}}$.

\paragraph{Relation to convergence results.} Neither $\epss$ nor $\epsg$ follows from existing convergence results for
diffusion or consistency models~\citep{chen2023sampling, benton2024nearly,
debortoli2022convergence, ConsistencyModelts1}; we verify the student--teacher
component of both empirically on the teacher--student pair \sdxlbase{} and
\sdxllightning{} (Remark~\ref{rem:eps-g-training},
Appendix~\ref{app:MLGD-ESTIMATOR:teacher-student-connection}), and the
teacher-vs-truth gap remains open.

\paragraph{Scope.} Throughout this appendix the loss is the
squared-MMD U-statistic with a \emph{fixed} kernel $k$
(Eq.~\eqref{eq:mmd-rbf} for synthetic experiments) at a fixed input
$x$. Three implementation details used in some experiments are not
covered by the formal statements below: (a) the \swd{} instantiation;
(b) the stabilized square-root \mmd{} of Appendix~\ref{app:sd:loss} used
in the Stable Diffusion experiment,
$\sqrt{|\widehat{\mathrm{MMD}}^2_U|+\epsilon}$,
whose loss-concentration bound transfers from
Proposition~\ref{prop:loss} via the global Lipschitzness of
$z\mapsto\sqrt{|z|+\epsilon}$, but whose gradient is not differentiable
at $z=0$ and is therefore \emph{not} covered by Proposition~\ref{prop:grad};
and (c) data-dependent
bandwidth selection (median heuristic), which makes $k$ itself a
function of the samples and breaks the bounded-difference structure
exploited below; and (d) differentiating through the
bandwidth: the propositions below fix $k$ (Assumption~\ref{ass:kernel})
and, in particular, treat the bandwidths $\{\sigma_\ell\}$ as constants
independent of $x$. In the synthetic experiments
(Appendix~\ref{sec:app:sim}), $\sigma$ is estimated from the merged
sample $\{y_i\}\cup\{\tilde y_j\}$ with autograd enabled; since
$y_i = \fphi(x,\eta_i)$ depends on $x$, the gradient computed in
practice includes an additional
$\partial_x\sigma \cdot \partial_\sigma\lossHat_\phi$ term
not covered by Proposition~\ref{prop:grad}.
We also analyze the inner estimator in isolation:
the propositions are not statements about the marginal of the full
outer reverse-diffusion trajectory (cf.\ Remark~\ref{rem:lgd-approx}).

\subsection{Setup and notation}

We reuse the paper's notation: $\PP(X,Y)$ is the data joint,
$\Pcond{x}$ the true conditional, $\target(Y)$ the target
distribution, and $\loss(x) = \dist\bigl(\Pcond{x},\, \target(Y)\bigr)$ the
abstract distributional loss of Problem~\ref{prob:opt}. Throughout
this appendix we analyze the \emph{squared-MMD} instantiation of
$\loss$ (i.e.\ $\dist = \operatorname{MMD}^2$) with a fixed
kernel; the unbiased U-statistic of this squared \mmd{} is the object of
the propositions below.

The distilled conditional sampler is a measurable map
$\fphi : \mathcal{X} \times \Omega \to \mathcal{Y}$ with internal randomness
$\eta \sim \pi$; let $\PP_\phi(Y \mid X=x)$ denote the law of
$\fphi(x, \eta)$. Throughout this appendix, $k$ denotes a fixed kernel
of the multi-bandwidth RBF form of Equation~\eqref{eq:mmd-rbf} with
bandwidths $\{\sigma_\ell\}$ treated as fixed; every $\operatorname{MMD}$ and
$\operatorname{MMD}^2$ that follows is understood with respect to this
one $k$. Such a kernel satisfies Assumption~\ref{ass:kernel} with
$K_{\max} = L$ and $L_k$ a bandwidth-dependent Lipschitz constant on
its feature map. We write $\varphi_k$ for the feature map of $k$ into the
RKHS $\mathcal{H}_k$, and set
\[
\loss(x) \;:=\; \operatorname{MMD}^2\!\bigl(\Pcond{x},\, \target\bigr),
\qquad
\loss_\phi(x) \;:=\; \operatorname{MMD}^2\!\bigl(\PP_\phi(Y\mid X=x),\,
\target\bigr).
\]
The formal estimator analyzed below is the unbiased squared-MMD
U-statistic (the experimental implementations differ, see the scope
paragraph above):
\[
\lossHat_\phi(x) \;:=\; \widehat{\operatorname{MMD}}^2_U\!\bigl(\Scond(x),\,
\SG\bigr),
\qquad \text{i.e.\ } \distHat = \widehat{\operatorname{MMD}}^2_U,
\]
where
$\Scond(x) = \bigl\{\fphi(x, \eta_i)\bigr\}_{i=1}^{\ncond}$
are i.i.d.\ samples from $\PP_\phi(Y \mid X=x)$ and
$\SG = \{\tilde{y}_j\}_{j=1}^{n_{\mathrm{target}}}$ are
i.i.d.\ from $\target$, drawn independently of $\Scond$.
The \mlgdf{} gradient estimator $\nabla_x \lossHat_\phi(x)$, which
supplies the black-box guidance term in Algorithm~\ref{alg:mlgd-outer},
is computed by automatic differentiation through $\fphi$.

\subsection{Assumptions}

\begin{assumption}[Kernel regularity]\label{ass:kernel}
There exist constants $K_{\max}, L_k, M_k < \infty$ such that:
\begin{enumerate}\itemsep1pt
\item[(i)] $0 \le k(a,b) \le K_{\max}$ for all $a,b \in \mathcal{Y}$;
\item[(ii)] the feature map $\varphi_k : \mathcal{Y} \to \mathcal{H}_k$ is
continuously differentiable, with operator-norm bound
$\|d\varphi_k(a)\|_{\mathrm{op}} \le L_k$ for every $a \in \mathcal{Y}$
(equivalently, $\varphi_k$ is $L_k$-Lipschitz:
$\|\varphi_k(a) - \varphi_k(b)\|_{\mathcal{H}_k} \le L_k\,\|a-b\|$);
\item[(iii)] $d\varphi_k$ is $M_k$-Lipschitz on $\mathcal{Y}$:
$\|d\varphi_k(a) - d\varphi_k(b)\|_{\mathrm{op}} \le M_k\,\|a-b\|$
for all $a,b\in\mathcal{Y}$ (equivalently, $k$ is $C^2$ with bounded
mixed partial derivatives).
\end{enumerate}
The kernel is characteristic on $\mathcal{Y}$.
(The multi-bandwidth RBF of Eq.~\eqref{eq:mmd-rbf} satisfies all three
clauses: (i) with $K_{\max}=5$ since each summand is in $[0,1]$;
(ii) via
$\|\varphi_k(a) - \varphi_k(b)\|_{\mathcal{H}_k}^2 =
2\sum_\ell\bigl(1-\exp(-\|a-b\|^2/2\sigma_\ell^2)\bigr)
\le \|a-b\|^2 \sum_\ell \sigma_\ell^{-2}$, so $L_k^2 = \sum_\ell
\sigma_\ell^{-2}$, where for the kernel of Eq.~\eqref{eq:mmd-rbf}
$2\sigma_\ell^2 = \sigma\,2^{\ell-3}$;
(iii) by smoothness of the Gaussian, with $M_k$ a similar
bandwidth-dependent constant. Conditional on fixed bandwidths
$\{\sigma_\ell\}$ the multi-bandwidth RBF satisfies all three
clauses. Clauses (ii)--(iii) are needed only
for Proposition~\ref{prop:grad}; Proposition~\ref{prop:loss} uses
(i)--(ii).)
\end{assumption}

\begin{assumption}[Output error of $\fphi$ under shared
noise]\label{ass:samples}
There exists $\epss \ge 0$ such that, for every $x \in \mathcal{X}$,
\[
\mathbb{E}_\eta\bigl\|\fphi(x,\eta) - f_{\mathrm{true}}(x,\eta)\bigr\|^2
\;\le\;\epss^2,
\]
where $f_{\mathrm{true}}$ is the reparameterizing map of
Assumption~\ref{ass:repar} below.
\end{assumption}

\begin{assumption}[Differentiability]\label{ass:diff}
$\fphi(x,\eta)$ is almost surely differentiable in $x$ with
$\mathbb{E}_\eta\,\|\partial_x \fphi(x,\eta)\|_F^2 \le C_f^2 < \infty$.
\end{assumption}

\begin{assumption}[Reparameterizable truth]\label{ass:repar}
There exists a measurable map
$f_{\mathrm{true}} : \mathcal{X} \times \Omega \to \mathcal{Y}$,
almost surely differentiable in $x$ with
$\mathbb{E}_\eta\|\partial_x f_{\mathrm{true}}(x,\eta)\|_F^2 \le C_f^2$,
such that $f_{\mathrm{true}}(x,\eta) \sim \Pcond{x}$
whenever $\eta \sim \pi$.
\end{assumption}

\begin{assumption}[Jacobian error]\label{ass:grad}
There exists $\epsg \ge 0$ such that, for every $x \in \mathcal{X}$,
\[
  \mathbb{E}_\eta\,
   \bigl\|\partial_x \fphi(x,\eta)
         - \partial_x f_{\mathrm{true}}(x,\eta)\bigr\|_F^2
  \;\le\; \epsg^2,
\]
where $\|\cdot\|_F$ is the Frobenius norm on Jacobians.
\end{assumption}

Assumptions~\ref{ass:kernel}, \ref{ass:diff}, and~\ref{ass:repar} are
standard regularity conditions; the multi-bandwidth RBF satisfies the
kernel clauses with $K_{\max}=5$, $L_k^2=\sum_\ell\sigma_\ell^{-2}$, and
$M_k$ a similar bandwidth-dependent constant, and bounded Jacobian
moments are routinely met by the smooth UNet-based samplers we use.
Assumption~\ref{ass:samples} (output error, same-$\eta$ coupling)
and Assumption~\ref{ass:grad} (Jacobian error) are \emph{idealized}
conditions that current distillation-training theory does not directly
imply; we discuss their training-time status in
Section~\ref{app:theory:status} and report empirical verification on
\sdxllightning{} vs.\ \sdxlbase{} in
Section~\ref{app:MLGD-ESTIMATOR:teacher-student-connection}. The
proofs below also use the gradient--expectation interchange (Leibniz
under the integral sign), which we adopt as an additional regularity
hypothesis on $\fphi, f_{\mathrm{true}}$.

\subsection{Concentration of the plug-in loss estimator}

\begin{proposition}[Plug-in loss concentration]\label{prop:loss}
Under Assumptions~\ref{ass:kernel} and~\ref{ass:samples}, for every
$\delta \in (0,1)$, every $x \in \mathcal{X}$, and every
$\ncond, n_{\mathrm{target}} \ge 2$ (so the unbiased
U-statistic denominators are well-defined), with probability at least
$1-\delta$ over the draws of $\Scond(x)$ and $\SG$,
\[
\bigl|\lossHat_\phi(x) - \loss(x)\bigr|
\;\le\;
K_{\max}\sqrt{8\,\log(2/\delta)\!\left(\tfrac{1}{\ncond}+\tfrac{1}{n_{\mathrm{target}}}\right)}
\;+\;
C_{\mathrm{loss}}\,\epss,
\]
where $C_{\mathrm{loss}} = 4\sqrt{K_{\max}}\,L_k$.
Equivalently, using $\sqrt{a+b}\le\sqrt{a}+\sqrt{b}$,
$|\lossHat_\phi - \loss|
\le 2\sqrt{2}\,K_{\max}\bigl(\sqrt{\log(2/\delta)/\ncond}
   + \sqrt{\log(2/\delta)/n_{\mathrm{target}}}\bigr) + C_{\mathrm{loss}}\,\epss$.
\end{proposition}

\begin{proof}
Write $n := \ncond$ and $m := n_{\mathrm{target}}$. By the
triangle inequality,
\begin{equation}
  \bigl|\lossHat_\phi(x) - \loss(x)\bigr|
    \;\le\;
    \underbrace{\bigl|\lossHat_\phi(x)
                     - \loss_\phi(x)\bigr|}_{T_1\ \text{(Monte Carlo)}}
    \;+\;
    \underbrace{\bigl|\loss_\phi(x) - \loss(x)\bigr|}_{T_2\
    \text{(distillation)}},
  \label{eq:loss-triangle}
\end{equation}
and we bound $T_1, T_2$ separately.

\emph{Step 1: Monte Carlo term.}
The unbiased two-sample U-statistic estimator
$\widehat{\operatorname{MMD}}^2_U$ satisfies
$\mathbb{E}\bigl[\lossHat_\phi(x)\bigr] = \loss_\phi(x)$ and
expands as
\[
\lossHat_\phi(x)
 = \tfrac{1}{n(n-1)}\!\!\sum_{i\neq j}\!\! k\bigl(\fphi(x,\eta_i),
 \fphi(x,\eta_j)\bigr)
 + \tfrac{1}{m(m-1)}\!\!\sum_{i\neq j}\!\! k(\tilde y_i, \tilde y_j)
 - \tfrac{2}{nm}\!\!\sum_{i,j}\!\! k\bigl(\fphi(x,\eta_i), \tilde y_j\bigr).
\]
By Assumption~\ref{ass:kernel}(i) each kernel evaluation lies in
$[0, K_{\max}]$, so changing one of its arguments shifts that
evaluation by at most $K_{\max}$.  We compute the bounded-difference
constants for McDiarmid's inequality applied to
$\lossHat_\phi$ as a function of the $n+m$ independent
variables $\eta_1,\ldots,\eta_n,\tilde y_1,\ldots,\tilde y_m$.
Replacing $\eta_i$ (which sets $Y_i = \fphi(x,\eta_i)$) leaves the
$\tilde y$-only term unchanged and affects only:
\begin{itemize}
\item the within-$Y$ term, in which $Y_i$ appears in the $2(n-1)$
ordered pairs $(i,j)$ and $(j,i)$ for $j\neq i$, each with coefficient
$\tfrac{1}{n(n-1)}$ and per-pair shift at most $K_{\max}$, contributing
$\tfrac{2(n-1)\cdot K_{\max}}{n(n-1)}=\tfrac{2K_{\max}}{n}$;
\item the cross term, in which $Y_i$ appears in $m$ pairs $(i,j)$
each with coefficient $\tfrac{2}{nm}$ and per-pair shift at most
$K_{\max}$, contributing $\tfrac{m\cdot 2 K_{\max}}{nm}=\tfrac{2K_{\max}}{n}$.
\end{itemize}
Summing, the bounded-difference constant for an $\eta_i$-change is
$c_\eta = \tfrac{4K_{\max}}{n}$. By the symmetric argument,
$c_{\tilde y} = \tfrac{4K_{\max}}{m}$.  Hence
$\sum_{i=1}^{n} c_\eta^2 + \sum_{j=1}^{m} c_{\tilde y}^2
 = 16\,K_{\max}^2\,(1/n + 1/m)$,
and McDiarmid's inequality~\citep{mcdiarmid1989method} gives, for any $t \ge 0$,
\[
  \mathbb{P}\!\bigl(|T_1| \ge t\bigr)
    \;\le\;
    2\exp\!\left(-\tfrac{t^2}{8K_{\max}^2\,(1/n + 1/m)}\right);
\]
equivalently, with probability at least $1 - \delta$,
\begin{equation}
  |T_1| \;\le\;
    K_{\max}\sqrt{8\,\log(2/\delta)\,(1/n + 1/m)}
    \;\le\;
    2\sqrt{2}\,K_{\max}\!\left(\sqrt{\tfrac{\log(2/\delta)}{n}}
                              + \sqrt{\tfrac{\log(2/\delta)}{m}}\right),
  \label{eq:loss-mc}
\end{equation}
where the second inequality uses $\sqrt{a+b}\le\sqrt{a}+\sqrt{b}$.

\emph{Step 2: Distillation term.}
For a distribution $\mathcal{Q}$ on $\mathcal{Y}$ write
$\mu_\mathcal{Q} := \mathbb{E}_{Y\sim\mathcal{Q}}\bigl[\varphi_k(Y)\bigr]
   \in \mathcal{H}_k$ for its kernel mean embedding. For any
positive-definite kernel $k$ and any two
distributions $\mathcal{Q}_1, \mathcal{Q}_2$,
$\operatorname{MMD}(\mathcal{Q}_1, \mathcal{Q}_2)
  = \|\mu_{\mathcal{Q}_1} -
  \mu_{\mathcal{Q}_2}\|_{\mathcal{H}_k}$,\footnote{This RKHS-norm identity does
  not require characteristicness; the characteristic clause of
  Assumption~\ref{ass:kernel} is used only for the identifiability part of
  Corollary~\ref{cor:lowrank}.}
and $\|\mu_\mathcal{Q}\|_{\mathcal{H}_k} \le \sqrt{K_{\max}}$ for any
$\mathcal{Q}$. The reverse triangle inequality then gives
\begin{equation}
 \bigl|\operatorname{MMD}(\PP_\phi, \target)
    - \operatorname{MMD}(\PP, \target)\bigr|
   \;\le\;
   \|\mu_{\PP_\phi} - \mu_{\PP}\|_{\mathcal{H}_k}.
 \label{eq:loss-reverse-tri}
\end{equation}
Let $(Y, \tilde Y)$ be any coupling of $\PP$ and
$\PP_\phi$. Jensen's inequality in $\mathcal{H}_k$ and the
Lipschitz feature map (Assumption~\ref{ass:kernel}(ii)) yield
\[
  \|\mu_{\PP_\phi} - \mu_{\PP}\|_{\mathcal{H}_k}
    = \bigl\|\mathbb{E}\!\bigl[\varphi_k(\tilde Y) - \varphi_k(Y)\bigr]
                                  \bigr\|_{\mathcal{H}_k}
    \;\le\; \mathbb{E}\|\varphi_k(\tilde Y) - \varphi_k(Y)\|_{\mathcal{H}_k}
    \;\le\; L_k\,\mathbb{E}\|\tilde Y - Y\|.
\]
Taking the infimum over couplings gives
$\|\mu_{\PP_\phi} - \mu_{\PP}\|_{\mathcal{H}_k}
 \le L_k\, W_1(\PP_\phi, \PP)
 \le L_k\, W_2(\PP_\phi, \PP)
 \le L_k\,\epss$
by Assumption~\ref{ass:samples}. Substituting
into~\eqref{eq:loss-reverse-tri} and using $|a^2 - b^2| \le (a+b)|a-b|$
with $a, b \le 2\sqrt{K_{\max}}$ (each \mmd{} is at most twice the maximal
embedding norm),
\begin{equation}
  T_2 \;\le\; 4\sqrt{K_{\max}}\cdot L_k\,\epss
        \;=:\; C_{\mathrm{loss}}\,\epss.
  \label{eq:loss-dist}
\end{equation}

Combining~\eqref{eq:loss-mc}, \eqref{eq:loss-dist}
and~\eqref{eq:loss-triangle} gives the stated bound with probability
at least $1 - \delta$.
\end{proof}

\subsection{Bias and variance of the gradient estimator}

\begin{proposition}[Gradient estimator]\label{prop:grad}
Under Assumptions~\ref{ass:kernel}--\ref{ass:grad}, with
$\ncond, n_{\mathrm{target}} \ge 2$, and assuming the
gradient--expectation interchange in Step~1 of the proof is valid
(e.g., via local dominated differentiability of $\fphi$ and
$f_{\mathrm{true}}$ in $x$), there exist constants $C_g, C_s, C_v$
depending only on $K_{\max}, L_k, M_k, C_f$ such that
\[
\begin{aligned}
\bigl\|\mathbb{E}\bigl[\nabla_x \lossHat_\phi(x)\bigr] - \nabla_x
\loss(x)\bigr\|
&\;\le\;
C_g\,\epsg \;+\; C_s\,\epss, \\[6pt]
\mathbb{E}\,\bigl\|\nabla_x \lossHat_\phi(x) - \mathbb{E}\bigl[\nabla_x
\lossHat_\phi(x)\bigr]\bigr\|^2
&\;\le\;
C_v\!\left(\frac{1}{\ncond} + \frac{1}{n_{\mathrm{target}}}\right).
\end{aligned}
\]

Explicit values of $C_g, C_s, C_v$ are tracked through the proof; they
are stated as generic constants because their exact magnitudes are
not used downstream.
\end{proposition}

\begin{proof}
Throughout, expectations are over $\eta_1, \ldots, \eta_{\ncond}\sim\pi$ i.i.d.\
and $\tilde y_1,\ldots,\tilde y_{n_{\mathrm{target}}}\sim\target$ i.i.d., the
two
samples mutually independent. We bound the bias and the variance in turn,
in five steps:
\textbf{Step~1} establishes a dominating envelope and the
gradient--expectation interchange used throughout.
\textbf{Step~2} fixes notation $(A,B,C,u,v)$ and decomposes the
population gradient gap $\nabla_x\loss_\phi-\nabla_x\loss$
into three RKHS pairings.
\textbf{Step~3} bounds each of the four norms $\|u\|, \|\nabla_x u\|,
\|v\|, \|\nabla_x v\|$ that appear in those pairings.
\textbf{Step~4} assembles the bias bound by Cauchy--Schwarz.
\textbf{Step~5} bounds the variance separately.

\paragraph{Notation for Steps~2--4.}
Write
\[
A(x) := \mu_{\PP(\cdot|x)},\quad
B(x) := \mu_{\PP_\phi(\cdot|x)},\quad
C := \mu_{\target},\quad
u(x) := B(x) - A(x),\quad
v(x) := A(x) - C
\]
for the kernel mean embeddings of the true and student conditionals
and of the target. For an $\mathcal{H}_k$-valued function $\psi(x)$
depending on $x\in\mathbb{R}^{\dimx}$, $\nabla_x \psi(x)$
denotes the linear map $\mathbb{R}^{\dimx}\to\mathcal{H}_k$
given by $\xi\mapsto\sum_i \xi_i\,\partial_{x_i} \psi(x)$, and we
write $\|\nabla_x \psi\|_{\mathcal{H}_k}$ for its Hilbert--Schmidt
norm $(\sum_i \|\partial_{x_i} \psi\|_{\mathcal{H}_k}^2)^{1/2}$.
With this convention each pairing $\langle a,\nabla_x b\rangle_{\mathcal{H}_k}$
is a vector in $\mathbb{R}^{\dimx}$, and Cauchy--Schwarz in
$\mathcal{H}_k$ gives $\|\langle a,\nabla_x b\rangle_{\mathcal{H}_k}\|_2
\le \|a\|_{\mathcal{H}_k}\,\|\nabla_x b\|_{\mathcal{H}_k}$ coordinate
by coordinate.

\paragraph{Step 1a: Dominating envelope.}
Each $x$-dependent summand of $\lossHat_\phi(x)$ has the
form $h(x; \eta, \eta') := k\bigl(\fphi(x,\eta), \fphi(x,\eta')\bigr)$
(``cond--cond'' terms) or
$h'(x; \eta, \tilde y) := k\bigl(\fphi(x,\eta), \tilde y\bigr)$
(``cond--target'' terms); the term not involving $\fphi$ is constant
in $x$. For any fixed $b \in \mathcal{Y}$ the map
$a \mapsto k(a, b) = \langle \varphi_k(a), \varphi_k(b)\rangle_{\mathcal{H}_k}$
has Fréchet derivative $d_a k(\cdot, b): \mathbb{R}^{\dimy} \to
\mathbb{R}$
given by $\xi \mapsto \langle d\varphi_k(a)\xi,
\varphi_k(b)\rangle_{\mathcal{H}_k}$.
Assumption~\ref{ass:kernel}(ii) gives $\|d\varphi_k(a)\|_{\mathrm{op}} \le L_k$
and Assumption~\ref{ass:kernel}(i) gives $\|\varphi_k(b)\|_{\mathcal{H}_k}
= \sqrt{k(b,b)} \le \sqrt{K_{\max}}$, so by Cauchy--Schwarz in $\mathcal{H}_k$,
$\|\nabla_a k(a, b)\|_2 \le L_k \sqrt{K_{\max}}$ for all $a, b \in \mathcal{Y}$.
The chain rule then yields
\begin{equation}
  \bigl\|\nabla_x h(x; \eta, \eta')\bigr\|_2
    \;\le\;
    L_k \sqrt{K_{\max}}\,
        \bigl(\|\partial_x \fphi(x,\eta)\|_F
            + \|\partial_x \fphi(x,\eta')\|_F\bigr),
  \label{eq:grad-dominating}
\end{equation}
and an analogous one-Jacobian bound for $h'$. By Cauchy--Schwarz and
Assumption~\ref{ass:diff},
$\mathbb{E}_\eta\|\partial_x \fphi(x,\eta)\|_F
  \le (\mathbb{E}_\eta \|\partial_x \fphi\|_F^2)^{1/2} \le C_f$;
hence the derivative in~\eqref{eq:grad-dominating} has finite first
moment under $\pi\otimes\pi$, bounded in expectation by
$2 L_k \sqrt{K_{\max}}\,C_f$ (the cond--target analog has the same
bound, treating the second argument as $\target$-distributed).
Differentiating under the expectation uses the local
dominated-differentiability hypothesis of Step~1b.

\paragraph{Step 1b: Gradient--expectation interchange.}
We assume that the gradient--expectation interchange used below holds
via the dominated-convergence form of Leibniz's rule. The first-moment bound of
Step~1a, together with uniform local second-moment control of
$\partial_x \fphi$ (routinely satisfied for the smooth UNet-based
samplers we use), supplies the standard local-dominating random
variable required; we adopt the conclusion as a regularity hypothesis
on $\fphi, f_{\mathrm{true}}$ rather than spelling out the
construction:
\begin{equation}
  \mathbb{E}\bigl[\nabla_x \lossHat_\phi(x)\bigr]
    \;=\; \nabla_x\,\mathbb{E}\bigl[\lossHat_\phi(x)\bigr]
    \;=\; \nabla_x \loss_\phi(x).
  \label{eq:grad-interchange}
\end{equation}
The same envelope applies to
$x\mapsto\mu_{\PP_\phi(\cdot|x)} =
\mathbb{E}_\eta\Bigl[\varphi_k\bigl(\fphi(x,\eta)\bigr)\Bigr]$
and to $x\mapsto\mu_{\PP(\cdot|x)} =
\mathbb{E}_\eta\Bigl[\varphi_k\bigl(f_{\mathrm{true}}(x,\eta)\bigr)\Bigr]$
(using Assumption~\ref{ass:repar} for the latter), so both kernel mean
embeddings are $\mathcal{H}_k$-Fr\'echet-differentiable in $x$.

\paragraph{Step 2: Bias decomposition.}
By~\eqref{eq:grad-interchange}, the bias equals a deterministic
population gap,
$\|\mathbb{E}[\nabla_x \lossHat_\phi] - \nabla_x \loss\|_2
  = \|\nabla_x \loss_\phi - \nabla_x \loss\|_2$.
By the kernel-mean-embedding identity for \mmd{}$^2$ (using the notation
$A,B,C,u,v$ defined above),
$\loss(x) = \|v(x)\|_{\mathcal{H}_k}^2$ and
$\loss_\phi(x) = \|u(x)+v(x)\|_{\mathcal{H}_k}^2$.
Differentiating $\psi\mapsto\|\psi\|_{\mathcal{H}_k}^2$ in $\mathcal{H}_k$ gives
$\nabla\|\psi\|^2 = 2\langle \psi, \nabla \psi\rangle_{\mathcal{H}_k}$, hence
\begin{align}
\nabla_x \loss_\phi - \nabla_x \loss
&= 2\langle u + v, \nabla_x(u+v)\rangle_{\mathcal{H}_k}
   - 2\langle v, \nabla_x v\rangle_{\mathcal{H}_k}\notag\\
&= 2\langle u, \nabla_x u\rangle_{\mathcal{H}_k}
   + 2\langle u, \nabla_x v\rangle_{\mathcal{H}_k}
   + 2\langle v, \nabla_x u\rangle_{\mathcal{H}_k}.
\label{eq:grad-decomp}
\end{align}
Steps~3--4 bound each of the four norms on the right-hand side and
combine them via Cauchy--Schwarz.

\paragraph{Step 3: Bound the four norms.}
We bound each of $\|u\|, \|\nabla_x u\|, \|v\|, \|\nabla_x v\|$ in turn.

\paragraph{(a) $\|u(x)\|_{\mathcal{H}_k} \le L_k\,\epss$.}
By Assumption~\ref{ass:repar}, $A(x) =
\mathbb{E}_\eta\Bigl[\varphi_k\bigl(f_{\mathrm{true}}(x,\eta)\bigr)\Bigr]$
and $B(x) = \mathbb{E}_\eta\Bigl[\varphi_k\bigl(\fphi(x,\eta)\bigr)\Bigr]$ for
the same
$\pi$-distributed $\eta$. Hence
\[
u(x) = B(x) - A(x) =
\mathbb{E}_\eta\!\Bigl[\varphi_k\bigl(\fphi(x,\eta)\bigr)-\varphi_k\bigl(f_{\mathrm{true}}(x,\eta)\bigr)\Bigr].
\]
Jensen's inequality in $\mathcal{H}_k$ gives
$\|u\|_{\mathcal{H}_k} \le
\mathbb{E}_\eta\|\varphi_k(\fphi)-\varphi_k(f_{\mathrm{true}})\|_{\mathcal{H}_k}$,
and $\varphi_k$ being $L_k$-Lipschitz (Assumption~\ref{ass:kernel}(ii))
gives $\|\varphi_k(\fphi)-\varphi_k(f_{\mathrm{true}})\|_{\mathcal{H}_k} \le
L_k\|\fphi-f_{\mathrm{true}}\|$,
so $\|u\|_{\mathcal{H}_k} \le L_k\,\mathbb{E}_\eta\|\fphi-f_{\mathrm{true}}\|$.
By Cauchy--Schwarz, $\mathbb{E}_\eta\|\fphi-f_{\mathrm{true}}\|
\le (\mathbb{E}_\eta\|\fphi-f_{\mathrm{true}}\|^2)^{1/2}$, and by
Assumption~\ref{ass:samples} the latter is at most $\epss$.
Combining,
\[
\|u\|_{\mathcal{H}_k} \;\le\; L_k\,\epss.
\]

\paragraph{(b) $\|\nabla_x u(x)\|_{\mathcal{H}_k} \le L_k\,\epsg + M_k
C_f\,\epss$.}
We bound $\|\nabla_x u\|$ by adding and subtracting
$d\varphi_k(\fphi)\,\partial_x f_{\mathrm{true}}$ and applying the
Lipschitz-derivative constant $M_k$ from Assumption~\ref{ass:kernel}(iii).
By Step~1b's Leibniz justification applied to $u$,
\[
\nabla_x u(x)
 = \mathbb{E}_\eta\!\Bigl[\,
   \underbrace{d\varphi_k\bigl(\fphi(x,\eta)\bigr)\,\partial_x
   \fphi(x,\eta)}_{\text{linear map
   }\mathbb{R}^{\dimx}\to\mathcal{H}_k}
   - d\varphi_k\bigl(f_{\mathrm{true}}(x,\eta)\bigr)\,\partial_x
   f_{\mathrm{true}}(x,\eta)\Bigr].
\]
Add and subtract $d\varphi_k(\fphi)\,\partial_x f_{\mathrm{true}}$
inside the expectation:
\begin{equation}
\nabla_x u
 = \mathbb{E}_\eta\Bigl[
    \underbrace{d\varphi_k(\fphi)\bigl(\partial_x \fphi - \partial_x
    f_{\mathrm{true}}\bigr)}_{\Delta_1(\eta)}
   + \underbrace{\bigl(d\varphi_k(\fphi) -
   d\varphi_k(f_{\mathrm{true}})\bigr)\,\partial_x
   f_{\mathrm{true}}}_{\Delta_2(\eta)}
  \Bigr].
\label{eq:nabla-u-decomp}
\end{equation}
By Jensen's inequality and the triangle inequality,
$\|\nabla_x u\|_{\mathcal{H}_k}
\le \mathbb{E}_\eta\|\Delta_1(\eta)\|_{\mathcal{H}_k}
   + \mathbb{E}_\eta\|\Delta_2(\eta)\|_{\mathcal{H}_k}$.

\paragraph{Bound on $\Delta_1$.}
For any linear map $M: \mathbb{R}^{\dimy}\to\mathcal{H}_k$
and matrix $J\in\mathbb{R}^{\dimy\times \dimx}$,
the Hilbert--Schmidt operator norm satisfies
$\|M\,J\|_{\mathcal{H}_k} \le \|M\|_{\mathrm{op}}\,\|J\|_F$.
Applying this with $M = d\varphi_k(\fphi)$ (operator norm $\le L_k$ by
Assumption~\ref{ass:kernel}(ii)) and
$J = \partial_x \fphi - \partial_x f_{\mathrm{true}}$,
\[
\|\Delta_1(\eta)\|_{\mathcal{H}_k}
 \;\le\; L_k\,\|\partial_x \fphi(x,\eta) - \partial_x
 f_{\mathrm{true}}(x,\eta)\|_F.
\]
Taking $\mathbb{E}_\eta$ and applying Cauchy--Schwarz with
Assumption~\ref{ass:grad},
\begin{equation}
\mathbb{E}_\eta\|\Delta_1(\eta)\|_{\mathcal{H}_k}
 \;\le\; L_k\,\bigl(\mathbb{E}_\eta\|\partial_x \fphi - \partial_x
 f_{\mathrm{true}}\|_F^2\bigr)^{1/2}
 \;\le\; L_k\,\epsg.
\label{eq:T1-bound}
\end{equation}

\paragraph{Bound on $\Delta_2$.}
Apply the same Hilbert--Schmidt inequality with
$M = d\varphi_k(\fphi) - d\varphi_k(f_{\mathrm{true}})$ and
$J = \partial_x f_{\mathrm{true}}$:
\[
\|\Delta_2(\eta)\|_{\mathcal{H}_k}
 \;\le\; \|d\varphi_k\bigl(\fphi(x,\eta)\bigr) -
 d\varphi_k\bigl(f_{\mathrm{true}}(x,\eta)\bigr)\|_{\mathrm{op}}
   \cdot \|\partial_x f_{\mathrm{true}}(x,\eta)\|_F.
\]
By Assumption~\ref{ass:kernel}(iii), $d\varphi_k$ is $M_k$-Lipschitz, so
$\|d\varphi_k(\fphi) - d\varphi_k(f_{\mathrm{true}})\|_{\mathrm{op}}
 \le M_k\,\|\fphi(x,\eta) - f_{\mathrm{true}}(x,\eta)\|$, giving
\[
\|\Delta_2(\eta)\|_{\mathcal{H}_k}
 \;\le\; M_k\,\|\fphi-f_{\mathrm{true}}\|\cdot\|\partial_x
 f_{\mathrm{true}}\|_F.
\]
Take $\mathbb{E}_\eta$ and apply Cauchy--Schwarz in $L^2(\pi)$ to the
product, then use Assumption~\ref{ass:samples} (coupled output
fidelity) and Assumption~\ref{ass:repar}
($\mathbb{E}_\eta\|\partial_x f_{\mathrm{true}}\|_F^2 \le C_f^2$):
\begin{equation}
\mathbb{E}_\eta\|\Delta_2(\eta)\|_{\mathcal{H}_k}
 \;\le\; M_k\,\bigl(\mathbb{E}_\eta\|\fphi-f_{\mathrm{true}}\|^2\bigr)^{1/2}
   \cdot \bigl(\mathbb{E}_\eta\|\partial_x f_{\mathrm{true}}\|_F^2\bigr)^{1/2}
 \;\le\; M_k\,C_f\,\epss.
\label{eq:T2-bound}
\end{equation}
Combining \eqref{eq:T1-bound}--\eqref{eq:T2-bound},
\begin{equation}
\|\nabla_x u(x)\|_{\mathcal{H}_k}
 \;\le\; L_k\,\epsg \;+\; M_k\,C_f\,\epss.
\label{eq:nabla-u-final}
\end{equation}

\paragraph{(c) $\|v(x)\|_{\mathcal{H}_k} \le 2\sqrt{K_{\max}}$.}
For any probability measure $\mathcal{Q}$ on $\mathcal{Y}$,
$\|\mu_\mathcal{Q}\|_{\mathcal{H}_k}^2
 = \mathbb{E}_{Y,Y'\sim\mathcal{Q}}\bigl[k(Y,Y')\bigr] \le K_{\max}$
by Assumption~\ref{ass:kernel}(i), so
$\|\mu_\mathcal{Q}\|_{\mathcal{H}_k} \le \sqrt{K_{\max}}$.
Triangle inequality gives
$\|v\|_{\mathcal{H}_k} = \|A - C\|_{\mathcal{H}_k}
 \le \|A\|_{\mathcal{H}_k}+\|C\|_{\mathcal{H}_k} \le 2\sqrt{K_{\max}}$.

\paragraph{(d) $\|\nabla_x v(x)\|_{\mathcal{H}_k} \le L_k\,C_f$.}
Repeat the chain-rule expansion of (b) for $A$ alone, with only the
$f_{\mathrm{true}}$ term. By the same Hilbert--Schmidt inequality
($\|d\varphi_k(f_{\mathrm{true}})\,\partial_x
f_{\mathrm{true}}\|_{\mathcal{H}_k}
 \le L_k\,\|\partial_x f_{\mathrm{true}}\|_F$) and Cauchy--Schwarz with
Assumption~\ref{ass:repar}'s second-moment bound,
\[
\|\nabla_x v\|_{\mathcal{H}_k}
 = \|\nabla_x A\|_{\mathcal{H}_k}
 \le L_k\,(\mathbb{E}_\eta\|\partial_x f_{\mathrm{true}}\|_F^2)^{1/2}
 \le L_k\,C_f.
\]
($\nabla_x C = 0$ since $\target$ does not depend on $x$.)

\paragraph{Step 4: Assemble the bias bound.}
Apply Cauchy--Schwarz to each of the three pairings in
\eqref{eq:grad-decomp}. Step~3 supplied two valid upper bounds on
$\|u\|_{\mathcal{H}_k}$: the assumption-driven bound
$\|u\|\le L_k\epss$ from Step~3(a), and the universal bound
$\|u\|\le 2\sqrt{K_{\max}}$ (same triangle-inequality argument as
Step~3(c) applied to $u$).\footnote{It would be incorrect to combine
these into an upper bound on $\epss$ itself: two upper bounds
on $\|u\|$ do not in general bound $\epss$ from above.}
We use whichever of the two $\|u\|$ bounds gives the tightest
term-by-term Cauchy--Schwarz product.

For the first pairing $2\|u\|\|\nabla_x u\|$ both factors are
distillation quantities; using $\|u\|\le 2\sqrt{K_{\max}}$ and
\eqref{eq:nabla-u-final} keeps the bound first-order in
$\epsg, \epss$:
\[
2\,\|u\|\|\nabla_x u\|
 \;\le\; 2\,(2\sqrt{K_{\max}})\bigl(L_k\,\epsg + M_k C_f\,\epss\bigr)
 \;=\; 4 L_k\sqrt{K_{\max}}\,\epsg + 4 M_k C_f\sqrt{K_{\max}}\,\epss.
\]
For the second pairing $2\|u\|\|\nabla_x v\|$ we use the tighter
bound $\|u\|\le L_k\epss$ (Step~3(a)) combined with
$\|\nabla_x v\|\le L_k C_f$ (Step~3(d)):
\[
2\,\|u\|\|\nabla_x v\|
 \;\le\; 2\,(L_k\epss)(L_k C_f)
 \;=\; 2 L_k^2 C_f\,\epss.
\]
For the third pairing $2\|v\|\|\nabla_x u\|$ we use the universal
bound $\|v\|\le 2\sqrt{K_{\max}}$ (Step~3(c)) and \eqref{eq:nabla-u-final}:
\[
2\,\|v\|\|\nabla_x u\|
 \;\le\; 2\,(2\sqrt{K_{\max}})\bigl(L_k\,\epsg + M_k C_f\,\epss\bigr)
 \;=\; 4 L_k\sqrt{K_{\max}}\,\epsg + 4 M_k C_f\sqrt{K_{\max}}\,\epss.
\]
Summing the three pairings,
\[
\|\nabla_x \loss_\phi - \nabla_x \loss\|_2
 \;\le\; C_g\,\epsg + C_s\,\epss,
\]
with
\[
C_g \;=\; 8\,L_k\sqrt{K_{\max}},\qquad
C_s \;=\; 2\,L_k^2 C_f + 8\,M_k C_f\sqrt{K_{\max}}.
\]
No higher-order $\epss\epsg$ or $\epss^2$
remainder appears; the bound is first-order by construction.
Both depend only on $K_{\max}, L_k, M_k, C_f$.

\paragraph{Step 5: Variance.}
For an $\mathbb{R}^{\dimx}$-valued random vector $Z$, write
$\mathrm{Var}(Z) := \mathbb{E}\|Z - \mathbb{E}Z\|_2^2$ throughout this
step (the quantity that the variance bound in the proposition statement
controls).
Decompose $\lossHat_\phi(x) = S_{\mathrm{cc}}(x) - 2 S_{\mathrm{ct}}(x) +
S_{\mathrm{tt}}$,
where $S_{\mathrm{cc}}, S_{\mathrm{ct}}$ are the cond--cond and
cond--target U-statistics (only these depend on $x$) and $S_{\mathrm{tt}}$
is the target--target U-statistic (constant in $x$). Using
$\mathrm{Var}(A+B)\le 2\,\mathrm{Var}(A)+2\,\mathrm{Var}(B)$,
\begin{equation}
\mathrm{Var}\bigl(\nabla_x \lossHat_\phi\bigr)
 \;\le\;
 2\,\mathrm{Var}\!\bigl(\nabla_x S_{\mathrm{cc}}\bigr)
  + 8\,\mathrm{Var}\!\bigl(\nabla_x S_{\mathrm{ct}}\bigr).
\label{eq:var-split}
\end{equation}
By \eqref{eq:grad-dominating} the kernel of $\nabla_x S_{\mathrm{cc}}$
has uniformly bounded second moment $\le 8 L_k^2 K_{\max} C_f^2$
(use $(a+b)^2\le 2a^2+2b^2$ and $\mathbb{E}\|\partial_x \fphi\|_F^2 \le C_f^2$);
the same bound holds for the cond--target kernel of $\nabla_x S_{\mathrm{ct}}$.
Standard variance bounds for finite-second-moment $U$-statistics
(Hoeffding's projection / Hoeffding-decomposition bound for one-sample
order-2 U-stats~\citep{Hoeffding1948}; two-sample analog for cross
U-stats) imply
$\mathrm{Var}(\nabla_x S_{\mathrm{cc}}) \le c_1\,L_k^2 K_{\max} C_f^2 / \ncond$
and
$\mathrm{Var}(\nabla_x S_{\mathrm{ct}}) \le c_2\,L_k^2 K_{\max} C_f^2 (1/\ncond
+ 1/n_{\mathrm{target}})$
for absolute constants $c_1, c_2$. Substituting into \eqref{eq:var-split}
and combining,
\[
\mathrm{Var}(\nabla_x \lossHat_\phi)
 \;\le\;
 C_v\!\left(\frac{1}{\ncond} + \frac{1}{n_{\mathrm{target}}}\right),
\]
with $C_v = O\!\bigl(L_k^2 K_{\max} C_f^2\bigr)$. We do not pin down the
absolute constant since only the rate is used.
\end{proof}

\begin{remark}[Generic distillation-through-autograd
structure]\label{rem:generic-distillation}
The argument used to bound $\|\nabla_x u\|_{\mathcal{H}_k}$ in Step~3(b)
of the proof above generalizes to any pair of reparameterizable
samplers $f_\star, \fphi : \mathcal{X}\times\Omega\to\mathcal{Y}$
sharing noise space $(\Omega,\pi)$ and satisfying the same-$\eta$
output and Jacobian couplings
$\mathbb{E}_\eta\|\fphi-f_\star\|^2 \le \varepsilon_{\mathrm{dist}}^2$
and
$\mathbb{E}_\eta\|\partial_x \fphi-\partial_x f_\star\|_F^2 \le
\varepsilon_{g,\mathrm{dist}}^2$,
together with a bounded second-moment condition on the teacher
Jacobian. The student-vs-teacher gradient discrepancy of any
squared-MMD loss against a fixed target then satisfies
$\|\mathbb{E}[\nabla_x \lossHat_\phi] - \mathbb{E}[\nabla_x \lossHat_\star]\|
 \le C_g\,\varepsilon_{g,\mathrm{dist}} + C_s\,\varepsilon_{\mathrm{dist}}$,
with the same $C_g, C_s$ as Proposition~\ref{prop:grad};
Proposition~\ref{prop:memory}(part~3) below specializes this bound.

The output coupling alone is not sufficient: the one-dimensional
example $f_\star(x)=x$, $\fphi(x)=x+\varepsilon\,\sin(x/\varepsilon^2)$
has $\|f_\star-\fphi\|_\infty=\varepsilon\to 0$ while
$|f_\star'(x)-\fphi'(x)|=(1/\varepsilon)|\cos(x/\varepsilon^2)|$ has
amplitude $1/\varepsilon\to\infty$. The output condition controls
$L^2$ closeness; the Jacobian condition controls Sobolev $W^{1,2}$
closeness; the former does not imply the latter without additional
smoothness on higher derivatives. No widely-deployed
diffusion-distillation recipe (consistency models, progressive
distillation, LCM, DMD, ADD) penalizes Jacobian divergence during
training, so $\varepsilon_{g,\mathrm{dist}}$ is not directly bounded by
the training loss; Sobolev-style distillation
\citep{Czarnecki2017Sobolev} would supply it but is not part of any
current pipeline. We address this empirically in
Section~\ref{app:MLGD-ESTIMATOR:teacher-student-connection}.
\end{remark}

\subsection{Memory--accuracy trade-off}

For comparison, consider the same estimator using a $\Kstar$-step
\emph{teacher} sampler $f_\star$ obtained by unrolling the reverse
diffusion ($\Kstar \ge 2$; e.g.\ \sdxlbase{}, $\Kstar=20$), and let
$\fphi$ be a $\Ks$-step few-step \emph{student} ($\Ks\ll \Kstar$;
e.g.\ \sdxllightning{}, $\Ks=4$, or a single-step CM, $\Ks=1$).
Denote the teacher's induced conditional by $\PP_\star$, its
plug-in loss estimator by $\lossHat_\star$, and the
corresponding gradient estimator by $\nabla_x \lossHat_\star(x)$.

\begin{proposition}[Memory--accuracy trade-off]\label{prop:memory}
\textbf{Memory.} Assume both samplers are differentiable in $x$ with
finite second-moment Jacobians. Then:
\begin{enumerate}\itemsep2pt
\item \emph{Teacher (unrolled) gradient:} reverse-mode evaluation of
      $\nabla_x \lossHat_\star(x)$ has per-sample chain
      depth $\Theta(\Kstar)$ and total stored activations
      $\Theta(\Kstar \cdot \ncond)$ (absent checkpointing).
\item \emph{Few-step student:} reverse-mode evaluation of
      $\nabla_x \lossHat_\phi(x)$ has per-sample chain
      depth $\Theta(\Ks)$ and total stored activations
      $\Theta(\Ks \cdot \ncond)$. The memory ratio is
      therefore $\Theta(\Kstar/\Ks)$.
\end{enumerate}
\textbf{Gradient discrepancy.} Fix $x\in\mathcal{X}$ and assume
$\ncond, n_{\mathrm{target}}\ge 2$, so that the unbiased
\mmd{} U-statistics defining $\lossHat_\phi$ and
$\lossHat_\star$ are well-defined. In addition to
Assumption~\ref{ass:kernel}, suppose the following
\emph{student-vs-teacher} error conditions hold under a shared-noise
coupling (direct analogs of Assumptions~\ref{ass:samples}
and~\ref{ass:grad} with $f_\star$ in place of
$f_{\mathrm{true}}$):
\begin{enumerate}\itemsep2pt
\item[\rm(S1)] $\mathbb{E}_\eta\|\fphi(x,\eta) - f_\star(x,\eta)\|^2 \le
\varepsilon_{\mathrm{dist}}^2$ for every $x$;
\item[\rm(S2)] $\mathbb{E}_\eta\|\partial_x \fphi(x,\eta) - \partial_x
f_\star(x,\eta)\|_F^2 \le \varepsilon_{g,\mathrm{dist}}^2$ for every $x$;
\item[\rm(S3)] $\mathbb{E}_\eta\|\partial_x f_\star(x,\eta)\|_F^2 \le C_f^2$ for
every $x$ (teacher analog of Assumption~\ref{ass:repar}).
\end{enumerate}
Assume moreover that the gradient--expectation interchange used in
Proposition~\ref{prop:grad} is valid for both $\fphi$ and $f_\star$
(e.g., by local dominated differentiability in $x$). Then
\[
\bigl\|\mathbb{E}[\nabla_x \lossHat_\phi] - \mathbb{E}[\nabla_x
\lossHat_\star]\bigr\|
 \;\le\; C_g\,\varepsilon_{g,\mathrm{dist}} + C_s\,\varepsilon_{\mathrm{dist}},
\]
with constants $C_g, C_s$ as in Proposition~\ref{prop:grad}. The
teacher-to-truth errors $\varepsilon_d, \varepsilon_{g,d}$ do not
enter this bound \emph{because the comparison is between the student
and the teacher, not between either and the truth}: closeness to
truth would require those errors and is not claimed here.
\end{proposition}

\begin{proof}
\emph{Parts (1) and (2): reverse-mode memory.}
Reverse-mode differentiation of $\lossHat_\bullet(x)$ propagates through each of
the $\ncond$ evaluations $f_\bullet(x, \eta_i)$, each an unrolled chain of
$K_\bullet$ UNet calls with $\Theta(1)$ stored activations per call, and
combines them at a $\Theta(1)$-per-sample reduction; hence chain depth
$\Theta(K_\bullet)$ and stored activations $\Theta(K_\bullet \cdot \ncond)$ for
$K_\bullet \in \{\Kstar, \Ks\}$, absent checkpointing
(Appendix~\ref{app:ablation:memory}). The ratio is $\Theta(\Kstar/\Ks)$.

\emph{Gradient discrepancy:}
Replay the proof of Proposition~\ref{prop:grad} verbatim, but with the
population triple
$(\mu_{\target}, \mu_{\PP_\star(\cdot|x)}, \mu_{\PP_\phi(\cdot|x)})$
in place of
$(\mu_{\target}, \mu_{\PP(\cdot|x)}, \mu_{\PP_\phi(\cdot|x)})$
and with $f_\star$ playing the role $f_{\mathrm{true}}$ played there.
Concretely, set
\[
A'(x) := \mu_{\PP_\star(\cdot|x)},\quad
B'(x) := \mu_{\PP_\phi(\cdot|x)},\quad
u'(x) := B'(x) - A'(x),\quad
v'(x) := A'(x) - C,
\]
so that $\loss_\star(x) = \|v'(x)\|_{\mathcal{H}_k}^2$,
$\loss_\phi(x) = \|u'(x) + v'(x)\|_{\mathcal{H}_k}^2$ and the
same algebraic decomposition \eqref{eq:grad-decomp} holds with primes.
Steps~3(a)--(d) of Proposition~\ref{prop:grad}'s proof go through
unchanged with the substitutions $f_{\mathrm{true}} \!\to\! f_\star$,
$\epss \!\to\! \varepsilon_{\mathrm{dist}}$,
$\epsg \!\to\! \varepsilon_{g,\mathrm{dist}}$. The substitutions
require: Assumption~\ref{ass:kernel} (kernel regularity);
Assumptions~\ref{ass:diff} on $\fphi$ and (S3) above on $f_\star$
(bounded Jacobians); the same-$\eta$ output coupling (S1) and Jacobian
coupling (S2) of the proposition statement, which together replace
Assumptions~\ref{ass:samples} and~\ref{ass:grad} for this comparison.
Step~4 then yields
\[
 \bigl\|\mathbb{E}[\nabla_x \lossHat_\phi]
      - \mathbb{E}[\nabla_x \lossHat_\star]\bigr\|
   \;\le\;
   C_g\,\varepsilon_{g,\mathrm{dist}}
    + C_s\,\varepsilon_{\mathrm{dist}},
\]
with the same $C_g, C_s$ as in Proposition~\ref{prop:grad}.

The proof never invokes closeness of $f_\star$ to $f_{\mathrm{true}}$, so no
truth-related quantity enters, as stated.
\end{proof}

Proposition~\ref{prop:memory} formalizes the empirical observation of
Section~\ref{sec:ablation}: replacing the teacher with the distilled sampler
reduces
activation memory by the step-count ratio $\Kstar/\Ks$ (measured as
$\approx 15\times$ for \sdxlbase{} $\to$ \sdxlturbo{}, $\Kstar=30, \Ks=2$;
$\approx 5\times$ for \sdxlbase{} $\to$ \sdxllightning{}, $\Kstar=20,
\Ks=4$), while introducing the additional student-teacher
gradient discrepancy $O(\varepsilon_{g,\mathrm{dist}} +
\varepsilon_{\mathrm{dist}})$.

\subsection{Consequence for low-rank targets}

\begin{corollary}[Identifiability, including low-rank targets]\label{cor:lowrank}
Suppose $k$ is characteristic on $\mathcal{Y}$. Then
\[
\loss(x) = 0 \;\iff\; \Pcond{x} = \target.
\]
In particular, if $\target$ is supported on a (low-rank) submanifold
$\mathcal{M} \subseteq \mathcal{Y}$, any global minimizer $\xstar$ of
$\loss$ with $\loss(\xstar) = 0$ satisfies
\[
\operatorname{supp}\bigl(\Pcond{\xstar}\bigr)
 = \operatorname{supp}(\target) \subseteq \mathcal{M};
\]
the inclusion is an equality if additionally
$\operatorname{supp}(\target)=\mathcal{M}$.
If, in addition, $\lossHat_\phi \to \loss$
\emph{uniformly} on a compact $\mathcal{X}$ in the joint limit
$\ncond, n_{\mathrm{target}} \to \infty$ and $\epss, \epsg \to 0$,
then any sequence of empirical minimizers
$\widehat x_n \in \arg\min_x \lossHat_\phi(x)$ satisfies
$\loss(\widehat x_n) \to \min_x \loss(x)$.
Proposition~\ref{prop:loss} is pointwise; uniform convergence over a
compact $\mathcal{X}$ requires an additional covering / equicontinuity
argument under stronger assumptions (e.g.\ almost-sure bounded
Jacobians or sub-Gaussian derivative tails), which we do not prove here.
\end{corollary}

Corollary~\ref{cor:lowrank} complements
Propositions~\ref{prop:loss}--\ref{prop:grad}:
characteristic-kernel \mmd{} is identifying, so even when $\target$ has low-rank
(submanifold) support any $\xstar$ achieving zero loss recovers $\target$
exactly.
We state consistency for global minimizers only; non-convexity of $\loss$ in $x$
means stationary points of $\lossHat_\phi$ need not converge to minimizers
of $\loss$, which is a generic obstacle shared with all non-convex
gradient-based inverse design.

\subsection{$\Qb$ as ideal target; \mlgdf{} as an \lgd{}-style approximate
sampler}
\label{app:theory:lgd-approx}

\begin{remark}[$\Qb$ as ideal target; \mlgdf{} as an \lgd{}-style approximate
sampler]\label{rem:lgd-approx}
Problem~\ref{prob:sampling} defines the \emph{ideal} target
$\Qb(x_0)\propto\PP(x_0)\,e^{-\beta\loss(x_0)}$.
For a reverse diffusion to sample from $\Qb$ exactly, one
would replace the unconditional score $\nabla_{x_t}\log\PP_t(x_t)$
by the score of the noised tilted distribution
$\Qb^{(t)}(x_t)
=\int q_t(x_t\mid x_0)\,\Qb(x_0)\,dx_0
=\PP_t(x_t)\cdot\mathbb{E}\bigl[e^{-\beta\loss(x_0)}\mid x_t\bigr]$,
giving
\begin{equation}\label{eq:exact-tilt}
\nabla_{x_t}\log\Qb^{(t)}(x_t)
=\nabla_{x_t}\log\PP_t(x_t)
+\nabla_{x_t}\log\mathbb{E}_{x_0\sim \PP(\cdot\mid
x_t)}\!\bigl[e^{-\beta\loss(x_0)}\bigr].
\end{equation}
The second term is the exact loss-guidance score; it is intractable
because it requires an integral over the posterior $\PP(x_0\mid x_t)$ at
every $t$. Following \lgd{}~\citep{lossGuidedDiffusion}, \mlgdf{} substitutes
the Tweedie--Jensen surrogate
\[
\nabla_{x_t}\log\mathbb{E}\bigl[e^{-\beta\loss(x_0)}\mid x_t\bigr]
\;\approx\;
-\beta\,\nabla_{x_t}\loss\bigl(\xzero(x_t)\bigr),
\]
where $\xzero(x_t)$ is the Tweedie estimate. This substitution is
accurate when $\beta$ is small (to first order) or when $\PP(x_0\mid x_t)$ is
concentrated, and exact only in degenerate cases such as a point-mass
posterior. \mlgdf{} additionally Monte-Carlo
averages $\loss$ over $\nmc$ Gaussian perturbations of
$\xzero$ (Algorithm~\ref{alg:mlgd-d-full}), which tracks the posterior
expectation more faithfully than a single Tweedie point but still treats
$\PP(x_0\mid x_t)$ as Gaussian about the Tweedie mean. Consequently the
trajectory of Algorithm~\ref{alg:mlgd-outer} samples from a distribution
\emph{close to but not equal to} $\Qb$; the gap is affected
by $\beta$, the timestep discretization, $\nmc$, and the
inner-sampler errors $\epss,\epsg$.
Propositions~\ref{prop:loss}--\ref{prop:grad} analyze only the inner
loss/gradient estimator (one reverse-diffusion step), not the marginal
of the full guided trajectory. Asymptotically exact
alternatives (twisted sequential Monte Carlo
\citep{wu2023twisted}, MCMC-corrected guidance, and Feynman--Kac steering)
recover exact
$\Qb$ samples at the cost of additional particles or
accept/reject steps; we adopt the \lgd{} heuristic for its plug-and-play
compatibility with frozen pretrained samplers, and report empirical
agreement with the analytical $\Qb$ in the 1D MoG setting
(Figure~\ref{fig:sim_2d}(e), where twisted-SMC reweighting removes the
residual bias; Appendix~\ref{sec:app:cdms_beta}) as
evidence the heuristic is reasonable in practice. The \cdmo{} variant
(Problem~\ref{prob:opt}) makes a weaker claim that does not require exact
$\Qb$-sampling: only that the trajectory finds inputs of
low $\loss$.
\end{remark}

\subsection{Status of the error assumptions}
\label{app:theory:status}

\paragraph{Output gap
$\epss$.}\label{par:eps-s-status}\phantomsection\label{rem:eps-s-training}
The output gap $\epss$ of Assumption~\ref{ass:samples} is an
\emph{idealized} assumption.
\emph{Directly-trained samplers}: when $\fphi$ is trained directly on
data (e.g.\ improved Consistency Training on MNIST
\citep{ConsistencyModelts2}), score-based diffusion convergence
results \citep{chen2023sampling, benton2024nearly,
debortoli2022convergence} give related output-distribution guarantees
but do not directly imply the uniform conditional pathwise bound
required here, because they bound marginal sampling error, not the
same-$\eta$ coupling between two specific samplers.
\emph{Distilled samplers (teacher refinement)}: when $\fphi$ is
obtained by distilling a pretrained multi-step teacher $f_\star$
(e.g.\ \sdxlturbo{} or \sdxllightning{} distilled from \sdxlbase{}),
triangle inequality in the same-$\eta$ pathwise norm gives
$\epss \le \varepsilon_d + \varepsilon_{\mathrm{dist}}$, with
$\varepsilon_d := \sup_x \bigl(\mathbb{E}_\eta\|f_\star(x,\eta) -
f_{\mathrm{true}}(x,\eta)\|^2\bigr)^{1/2}$
the teacher-to-truth error and
$\varepsilon_{\mathrm{dist}} := \sup_x \bigl(\mathbb{E}_\eta\|\fphi(x,\eta) -
f_\star(x,\eta)\|^2\bigr)^{1/2}$
the student-to-teacher distillation error. The latter is the same-$\eta$
quantity closely related to the pathwise output agreement that
distillation encourages (adversarial or hybrid schemes such as
\sdxllightning{} do not minimize it directly); the former
remains an idealization, since the teacher does not exactly match
$f_{\mathrm{true}}$.

\paragraph{Jacobian gap
$\epsg$.}\label{par:eps-g-status}\phantomsection\label{rem:eps-g-training}
Unlike $\epss$, the Jacobian gap $\epsg$ of Assumption~\ref{ass:grad} is not
bounded by any current training result (Remark~\ref{rem:generic-distillation});
bounding it in terms of training quantities is an open problem, independent of
this paper.

We can still measure a \emph{component} of $\epsg$ against a
differentiable multi-step reference sampler $f_{\mathrm{ref}}$ trained
on the same data. Triangle inequality gives
\[
 \epsg \;\le\;
  \underbrace{\bigl(\mathbb{E}_\eta\|\partial_x f_{\mathrm{ref}}
                    - \partial_x
                    f_{\mathrm{true}}\|_F^2\bigr)^{1/2}}_{\varepsilon_{g,d}\ :\
                    \text{open}}
  \;+\;
  \underbrace{\bigl(\mathbb{E}_\eta\|\partial_x \fphi
                    - \partial_x
                    f_{\mathrm{ref}}\|_F^2\bigr)^{1/2}}_{\varepsilon_{g,\mathrm{dist}}\
                    :\ \text{measurable}},
\]
where $\varepsilon_{g,\mathrm{dist}}$ is directly measurable by autograd
under shared-noise coupling, while $\varepsilon_{g,d}$ remains
unmeasurable for the same reason the output-level $\varepsilon_d$
(above) is not empirically verifiable: the truth's Jacobian is unknown.

The empirical check is \emph{one-sided}. By triangle inequality,
$\widehat\varepsilon_{g,\mathrm{dist}}$ can \emph{falsify}
Assumption~\ref{ass:grad}: a large measurement forces either
$\epsg$ or $\varepsilon_{g,d}$ to be large. But a \emph{small}
$\widehat\varepsilon_{g,\mathrm{dist}}$ does not prove $\epsg$
is small, because $\fphi$ and $f_{\mathrm{ref}}$ may share systematic
biases and be jointly far from truth in the Jacobian without being far
from each other. We therefore report
$\widehat\varepsilon_{g,\mathrm{dist}}$ as a diagnostic, not as a proof
that $\epsg$ is small. Despite this caveat, the bounds in
Propositions~\ref{prop:loss}--\ref{prop:grad} degrade gracefully in
$\epsg$: if only $\epss$ is controlled and
$\epsg$ is unknown, Proposition~\ref{prop:loss} is unaffected
and Proposition~\ref{prop:grad} becomes a bias bound in terms of
$\epsg$ alone.

\paragraph{Linearity in $\epss$.} The sampler error enters both the loss and its
gradient linearly in $\epss$, through $\operatorname{MMD}(\PP_\phi,\PP) \le
L_k\,W_1(\PP_\phi,\PP)$ and the $L^2$ coupling of Assumption~\ref{ass:samples}.
This is sharper than a TV-based bound and specific to characteristic-kernel
distributional losses; point-target inverse design ($\|f(x) - y^\star\|$) has no
analogous distributional stability property.

\subsection{Empirical diagnostic of $\varepsilon_{g,\mathrm{dist}}$ on
\sdxllightning{} vs.\ \sdxlbase{}}
\label{app:MLGD-ESTIMATOR:teacher-student-connection}

We measure $\hat\varepsilon_{g,\mathrm{dist}}$ using
\sdxllightning{}~\citep{lin2024sdxllightningprogressiveadversarialdiffusion}
(4-step student $\fphi$) and \sdxlbase{} (20-step \ddim{} teacher $f_\star$)
across $N=100$ prompts spanning four visual style groups
(\texttt{photo}, \texttt{cinematic}, \texttt{dark~fantasy}, \texttt{cartoon}),
with noise $\eta_i$ shared between student and teacher.
Both models share the same UNet architecture; \sdxllightning{} is obtained
by distilling \sdxlbase{} with a hybrid of progressive (MSE) and adversarial
distillation in latent space, at native $1024\times1024$
resolution~\citep{lin2024sdxllightningprogressiveadversarialdiffusion}.
The student uses \texttt{EulerDiscreteScheduler} (trailing spacing, $w=0$,
CFG baked in), the teacher \texttt{DDIMScheduler} ($w=7.5$).

Because the latent observable $y\in\mathbb{R}^{65536}$ makes the exact
Jacobian intractable, we estimate the VJP-norm gap via $D=10$ random unit
vectors $v_d\sim\mathcal{N}(0,I)$, $\|v_d\|=1$, one reverse-mode pass each.
Writing $J_\bullet^{(i,d)}:=v_d^\top\partial_x f_\bullet(x_i,\eta_i)$
for the directional Jacobian (a row vector), where $x_i$ is the text-prompt
embedding of prompt $i$ and $f_\bullet$ maps it to the generated latent, so the
diagnostic measures Jacobian preservation in the text-conditioning channel, not
the scribble Jacobian of Section~\ref{sec:exp:sd}.
For a uniformly random unit $v \in \mathbb{R}^{\dimy}$ and $\Delta J =
\partial_x \fphi - \partial_x f_\star$,
$\mathbb{E}_v\|v^\top \Delta J\|^2 = \|\Delta J\|_F^2/\dimy$, so the directional
estimate is unbiased for the Frobenius error of Assumption~\ref{ass:grad} up to
the factor $1/\dimy$, which cancels in the relative errors below.
The group-level gaps are
\[
 \widehat\varepsilon_{\mathrm{dist}} \;=\;
   \sqrt{\tfrac{1}{N}\textstyle\sum_i
         \bigl\|\fphi(x_i,\eta_i) - f_\star(x_i,\eta_i)\bigr\|^2},
 \qquad
 \widehat\varepsilon_{g,\mathrm{dist}} \;=\;
   \sqrt{\tfrac{1}{ND}\textstyle\sum_{i,d}
         \bigl\|J_\phi^{(i,d)} - J_\star^{(i,d)}\bigr\|^2}.
\]
Since these quantities have different units, we normalize by the
student's own magnitude to obtain dimensionless relative errors
\[
 \widehat r_s \;:=\; \frac{\widehat\varepsilon_{\mathrm{dist}}}{\bar y_\phi},
 \qquad
 \widehat r_g \;:=\; \frac{\widehat\varepsilon_{g,\mathrm{dist}}}{\bar J_\phi},
\]
with
\[
\bar y_\phi \;:=\; \sqrt{\tfrac{1}{N}\textstyle\sum_i
  \|\fphi(x_i,\eta_i)\|^2},
\qquad
\bar J_\phi \;:=\; \sqrt{\tfrac{1}{N}\textstyle\sum_i
  \Bigl(\tfrac{1}{D}\textstyle\sum_d
  \|v_d^\top\partial_x \fphi(x_i,\eta_i)\|\Bigr)^2}.
\]

$\widehat r_g\approx\widehat r_s$ indicates that distillation preserved
Jacobian structure to the same relative precision as outputs;
$\widehat r_g\gg\widehat r_s$ signals Jacobian degradation beyond output
error, the failure mode flagged by Remark~\ref{rem:eps-g-training}.

\begin{table}[!htbp]
\centering
\setlength{\tabcolsep}{6pt}
\caption{Empirical Jacobian and output gaps: \sdxllightning{} ($\fphi$, 4-step,
  $1024\!\times\!1024$) vs.\ \sdxlbase{} ($f_\star$, 20-step \ddim{}),
  $N=100$ prompts, $D=10$ VJP directions.
  $\widehat r_s$ is the relative output error and $\widehat r_g$ the
  relative Jacobian error; both are dimensionless.
  Outliers flagged per group by Tukey IQR fence on
  $\hat{r}_g/\hat{r}_s$; count removed shown in parentheses.
  95\% bootstrap confidence intervals in brackets.
}
\label{tab:eps-g-combined}
\begin{tabular}{l ccc}
\toprule
\multicolumn{4}{c}{\textit{Panel A: with outliers}} \\
\midrule
Group
  & $\widehat r_s\ [95\%\,\mathrm{CI}]$
  & $\widehat r_g\ [95\%\,\mathrm{CI}]$
  & $\widehat r_g/\widehat r_s\ [95\%\,\mathrm{CI}]$ \\
\midrule

\texttt{photo}
  & $0.66\ [0.62,\,0.69]$
  & $1.75\ [1.24,\,2.38]$
  & $2.66\ [1.90,\,3.62]$ \\[2pt]

\texttt{cinematic}
  & $0.70\ [0.65,\,0.75]$
  & $31.34\ [4.92,\,45.25]$
  & $44.62\ [7.12,\,65.16]$ \\[2pt]

\texttt{dark fantasy}
  & $0.61\ [0.57,\,0.66]$
  & $6.42\ [1.68,\,12.61]$
  & $10.51\ [2.76,\,20.43]$ \\[2pt]

\texttt{cartoon}
  & $0.91\ [0.87,\,0.96]$
  & $28.35\ [2.92,\,54.59]$
  & $31.00\ [3.19,\,60.12]$ \\
\midrule

\textbf{All}
  & $0.75\ [0.72,\,0.79]$
  & $25.09\ [4.58,\,43.40]$
  & $33.31\ [6.16,\,58.38]$ \\

\midrule\midrule
\multicolumn{4}{c}{\textit{Panel B: without outliers (Tukey IQR fence on
$\hat{r}_g/\hat{r}_s$)}} \\
\midrule
Group
  & $\widehat r_s\ [95\%\,\mathrm{CI}]$
  & $\widehat r_g\ [95\%\,\mathrm{CI}]$
  & $\widehat r_g/\widehat r_s\ [95\%\,\mathrm{CI}]$ \\
\midrule

\texttt{photo} ($-3$)
  & $0.67\ [0.63,\,0.71]$
  & $1.60\ [1.15,\,2.19]$
  & $2.39\ [1.74,\,3.22]$ \\[2pt]

\texttt{cinematic} ($-5$)
  & $0.69\ [0.63,\,0.74]$
  & $2.45\ [1.59,\,2.94]$
  & $3.58\ [2.41,\,4.26]$ \\[2pt]

\texttt{dark fantasy} ($-4$)
  & $0.62\ [0.57,\,0.67]$
  & $1.81\ [1.24,\,2.93]$
  & $2.93\ [2.12,\,4.63]$ \\[2pt]

\texttt{cartoon} ($-4$)
  & $0.92\ [0.88,\,0.97]$
  & $3.75\ [2.24,\,5.06]$
  & $4.05\ [2.42,\,5.53]$ \\
\midrule

\textbf{All} ($-16$)
  & $0.75\ [0.71,\,0.79]$
  & $3.27\ [2.12,\,4.42]$
  & $4.34\ [2.87,\,5.84]$ \\
\bottomrule
\end{tabular}
\end{table}

\begin{figure}[t]
\centering
\begin{minipage}[t]{0.48\textwidth}
    \centering
    \includegraphics[width=\textwidth]{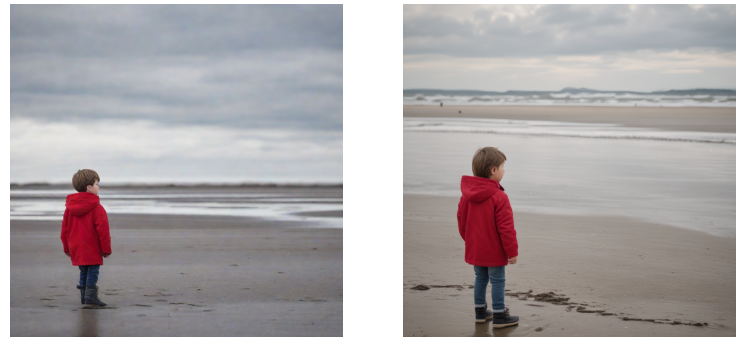}
    \textit{Well-aligned prompt} \\[2pt]
    $\hat{r}_s = 0.668$\quad $\hat{r}_g = 1.034$\quad $\hat{r}_g/\hat{r}_s =
    1.55$
\end{minipage}
\hfill
\begin{minipage}[t]{0.48\textwidth}
    \centering
    \includegraphics[width=\textwidth]{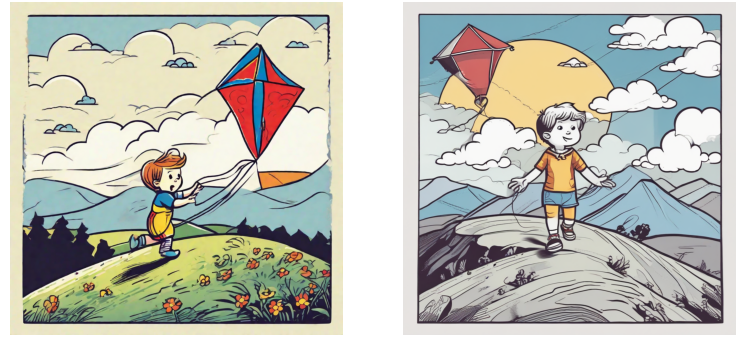}
    \textit{Outlier prompt} \\[2pt]
    $\hat{r}_s = 0.920$\quad $\hat{r}_g = 182.5$\quad $\hat{r}_g/\hat{r}_s =
    198.4$
\end{minipage}
\vspace{0.8em}
\includegraphics[width=\textwidth]{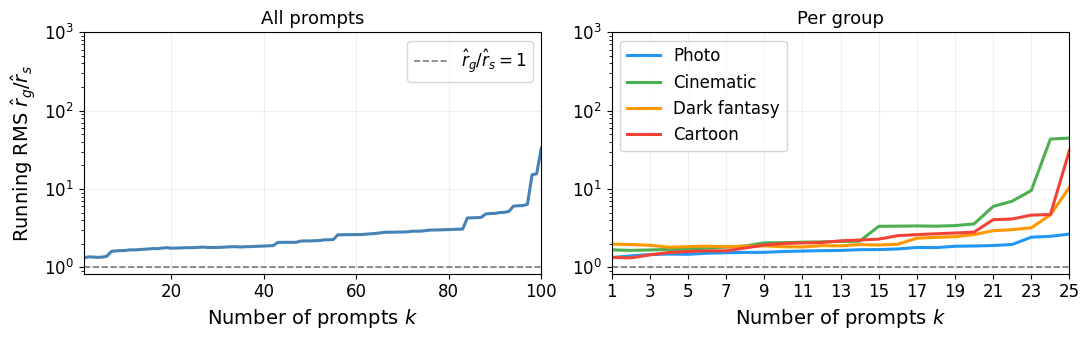}
\caption{\textbf{Empirical Jacobian error of \sdxllightning{} relative to
    \sdxlbase{}.}
    \textbf{Top row:} a well-aligned \texttt{photo} prompt (left,
    $\hat{r}_g/\hat{r}_s = 1.55$; \textit{``a young boy in a red jacket
    on an empty beach\ldots''}) and an outlier \texttt{cartoon} prompt
    (right, $\hat{r}_g/\hat{r}_s = 198.4$; \textit{``a cartoon
    child flying a kite on a hill, bold outlines...''}),
    each showing \sdxlbase{} (teacher-left) and \sdxllightning{} (student-right)
    outputs
    from the same noise $\eta$. The normalized output gap $\hat{r}_s$ is
    similar for both ($0.67$ vs.\ $0.92$), while the normalized Jacobian
    gap $\hat{r}_g$ differs by two orders of magnitude ($1.03$ vs.\
    $182.5$), showing that Jacobian misalignment is not predicted by
    output misalignment alone.
    \textbf{Bottom left:} running RMS $\hat{r}_g/\hat{r}_s$ over all
    100 prompts (sorted ascending); the aggregate stays moderate for
    most prompts and spikes only at the tail, confirming that a small
    number of outliers dominate.
    \textbf{Bottom right:} same plot by visual style group ($N=25$ each).
    \texttt{photo} and \texttt{cinematic} remain well-aligned throughout;
    \texttt{cartoon} degrades more consistently, as stylized rendering
    targets lie out-of-distribution for the 4-step Lightning student.
}
\label{fig:jacobian_main}
\end{figure}

\paragraph{Outlier identification.}
Within each group we flag prompt $i$ as an outlier if its
per-prompt ratio $\hat{r}_{g,i}/\hat{r}_{s,i}$ falls outside the
Tukey fence $[Q_1 - 1.5\,\mathrm{IQR},\; Q_3 + 1.5\,\mathrm{IQR}]$.
Across all 100 prompts, 16 are flagged (3 \texttt{photo},
5 \texttt{cinematic}, 4 \texttt{dark~fantasy}, 4 \texttt{cartoon}).
These alone are responsible for the large $\widehat r_g$ and wide
confidence intervals in Panel~A of Tables~\ref{tab:eps-g-combined}.
Removing them reduces $\widehat r_g$ from $25.09$ to $3.27$
while leaving $\widehat r_s$ essentially unchanged ($0.753$ in both cases),
confirming that the outliers reflect Jacobian instability rather than
global output degradation.

\paragraph{Running-RMS diagnostic.}
Figure~\ref{fig:jacobian_main} (bottom-left) shows the \emph{running RMS
aggregate} $\hat{r}_g / \hat{r}_s$ computed on the first $k$ prompts,
sorted by $\hat{r}_g$ ascending, so each curve is non-decreasing by
construction and its final value equals the full-group RMS aggregate.
The left panel pools all 100 prompts together; the right panel splits
by group.
Across all prompts the aggregate remains close to
$\hat{r}_g/\hat{r}_s \approx 1$--$5$ for the first 85 prompts,
providing encouraging empirical support for Remark~\ref{rem:eps-g-training}.
The aggregate spikes only at the tail, confirming that a small number
of outlier prompts drive the overall RMS rather than a systematic failure.
When splitting by visual style group (bottom-right), \texttt{photo}
and \texttt{cinematic} remain well-aligned for the bulk of their
prompts; the tail jump in \texttt{cinematic} is driven by 5
mode-mismatch outliers.
\texttt{cartoon} degrades more consistently.

\paragraph{The bound is qualitative.} Empirically, the
Jacobian gap $\hat r_g$ exceeds the output gap $\hat r_s$ across all
style groups (Table~\ref{tab:eps-g-combined}). Without outliers (Panel~B),
$\hat{r}_g/\hat{r}_s$ ranges from $2.39\times$ (\texttt{photo}) to
$4.05\times$ (\texttt{cartoon}), with an overall ratio of $4.34\times$;
outlier prompts can push this to $33.31\times$ (Panel~A).
Proposition~\ref{prop:grad}'s bound is therefore qualitatively
informative: it identifies the right two error sources and shows the
gradient bias decomposes cleanly into them, but \emph{not}
quantitatively tight at current distillation quality. We view the
propositions as a framework for measurement and regime identification:
they tell us \emph{which} quantities to estimate and how they propagate.

More broadly, distillation theory and training objectives control the output
error, which suffices when the student is a function evaluator; \mlgdf{}, like
any downstream optimization through a distilled model, consumes its Jacobian as
well, and needs the $\epsg$-type guarantee of
Remark~\ref{rem:generic-distillation} that current training does not provide.

\paragraph{Open questions.}
Several questions remain open.
First, the prompt-level properties that predict Jacobian instability are
unknown:
16 outliers across 100 prompts show no obvious syntactic or semantic
pattern, though one plausible explanation is mode mismatch: prompts
that elicit visually different outputs from teacher and student may
induce larger Jacobian discrepancies, since the two models effectively
compute gradients around different local modes of the image distribution.
Second, whether the strong alignment observed for
\texttt{photo}/\texttt{cinematic} reflects an inductive bias of
\sdxllightning{} towards photorealistic outputs, which would imply that
the choice of distillation target distribution matters for downstream
gradient-based use of the student.
We leave these to future work.

\paragraph{Compute.}
Experiments were conducted on a single NVIDIA RTX PRO 6000 Blackwell GPU (96\,GB
VRAM).
Each run evaluated $N=100$ prompts using $D=10$ VJP directions,
with two reverse-mode passes per direction (one for the teacher, one for the
student),
amounting to roughly 3 hours of wall-clock time.
The reproduction notebook is
\href{\repoeps}{\texttt{eps\_g\_experiment.ipynb}} (Appendix~\ref{app:code}).

\section{Licenses}
\label{app:licenses}

\begin{table}[H]
\centering
\caption{Licenses for all assets used in this paper.}
\label{tab:licenses}
\begin{threeparttable}
\setlength{\tabcolsep}{4pt} 

\begin{tabular}{@{}l p{5.5cm} l@{}} 
\toprule
\textbf{Asset} & \textbf{License} & \textbf{Version} \\
\midrule
\sdxlbase{}\tnote{1}                & CreativeML Open RAIL++-M & 1.0 \\
\sdxllightning{}\tnote{2}           & CreativeML Open RAIL++-M & --  \\
\sdxlturbo{}\tnote{3}               & SAI-NC Community License & -- \\
\controlnetscribble{}\tnote{4}      & Openrail                 & --  \\
HuggingFace Annotators\tnote{5}   & other (unspecified)      & --  \\
\clip{} ViT-L/14\tnote{6}            & MIT$^*$                  & --  \\
\sdedit{}\tnote{7}                   & MIT                      & --  \\
HuggingFace Diffusers\tnote{8}    & Apache 2.0               & --  \\
PyTorch\tnote{9}                  & BSD-3                    & --  \\
WandB\tnote{10}                   & MIT                      & --  \\
MNIST\tnote{11}                   & CC BY-SA 3.0             & --  \\
GB-FESO\tnote{12}              & MIT                      & -- \\
Our models\tnote{13}              & MIT                      & --  \\
Our code\tnote{14}                & MIT                      & --  \\
\bottomrule
\end{tabular}

\begin{tablenotes}
\item[1] \url{https://huggingface.co/stabilityai/stable-diffusion-xl-base-1.0}
\item[2] \url{https://huggingface.co/ByteDance/SDXL-Lightning}
\item[3] \url{https://huggingface.co/stabilityai/sdxl-turbo}
\item[4] \url{https://huggingface.co/lllyasviel/sd-controlnet-scribble}
\item[5] \url{https://huggingface.co/lllyasviel/Annotators}. 
Auxiliary annotator weights repository by the same author 
as \controlnet{}\tnote{4}; license not explicitly stated.
\item[6] \url{https://huggingface.co/openai/clip-vit-large-patch14}
\item[7] \url{https://github.com/ermongroup/SDEdit}
\item[8] \url{https://github.com/huggingface/diffusers}
\item[9] \url{https://github.com/pytorch/pytorch}
\item[10] \url{https://github.com/wandb/wandb}
\item[11] \url{https://huggingface.co/datasets/ylecun/mnist}
 \item[12] \url{https://github.com/MendelsResearchGroup/Gradient-Based-Free-Energy-Surface-Optimization}
\item[13] \url{\hfmnist}
\item[14] \url{\repomain}
\end{tablenotes}
\end{threeparttable}
\end{table}

\end{document}